\documentclass{article}
\usepackage{iclr2027_conference,times}
\usepackage[T1]{fontenc}

\usepackage{amsmath,amsfonts,bm}

\def\figref#1{Figure~\ref{#1}}
\def\secref#1{section~\ref{#1}}
\def\Secref#1{Section~\ref{#1}}

\def\1{\bm{1}}

\DeclareMathAlphabet{\mathsfit}{\encodingdefault}{\sfdefault}{m}{sl}
\SetMathAlphabet{\mathsfit}{bold}{\encodingdefault}{\sfdefault}{bx}{n}

\usepackage{hyperref}
\hypersetup{colorlinks=true,linkcolor=black,citecolor=black,urlcolor=black,pdftitle={Aperture: Merge-Consistent Rotary States for Compressed Tokens},pdfauthor={Yuhao DU and Shunian CHEN}}
\usepackage{url}
\usepackage{graphicx}
\usepackage{booktabs}
\usepackage{capt-of}
\usepackage{longtable}
\usepackage{multirow}
\usepackage{amsmath,amssymb,amsthm}
\usepackage{xcolor}
\usepackage{enumitem}
\usepackage{placeins}

\newcommand{\tabref}[1]{Table~\ref{#1}}
\newcommand{\propref}[1]{Proposition~\ref{#1}}
\newcommand{\thmref}[1]{Theorem~\ref{#1}}

\newtheorem{definition}{Definition}
\newtheorem{proposition}{Proposition}
\newtheorem{lemma}{Lemma}
\newtheorem{theorem}{Theorem}

\newcommand{\RedCondExp}{2.02}

\newcommand{\SpanN}{64}
\newcommand{\SpanTilingMax}{0.860}
\newcommand{\SpanTilingGain}{0.204}

\newcommand{\SpanTilingMiss}{0.140}

\newcommand{\SpanEvictGain}{0.000}

\newcommand{\SpanTilingMissIdent}{0.016}

\newcommand{\SpanTilingLOne}{0.656}

\newcommand{\SpanNmin}{4}
\newcommand{\SpanTilingMinOne}{0.518}

\newcommand{\GaugeUniform}{2.6\times 10^{-13}}
\newcommand{\GaugePlain}{539}
\newcommand{\GaugeQueryRel}{0.54}
\newcommand{\GaugeHetRel}{0.49}
\newcommand{\GaugeScale}{501}

\newcommand{\CovApertureRms}{0.34}

\newcommand{\CovNyquist}{0.10}
\newcommand{\ClampCovAbs}{0.06}
\newcommand{\ClampCovMain}{0.03}

\newcommand{\CovUneqAperture}{0.220}
\newcommand{\CovUneqApertureVar}{0.224}
\newcommand{\CovUneqApertureUnion}{0.30}

\newcommand{\CovUneqEqCtrl}{6.5\times 10^{-17}}

\newcommand{\AtomIntervalErr}{7.8\times 10^{-16}}
\newcommand{\AtomPointErr}{1.0}
\newcommand{\AtomDirichletId}{1.2\times 10^{-15}}

\newcommand{\AtomRatioPi}{1.5708}

\newcommand{\AtomNGrid}{32}
\newcommand{\AtomQwenPastNullHeld}{7/16}
\newcommand{\AtomQwenMinGainHeld}{0.0053}

\newcommand{\AtomQwenPastNullTrained}{4/16}

\newcommand{\RoteSmallHarmOne}{0.011}

\newcommand{\CostGainOverRope}{25\%}
\newcommand{\CostGainOverAttn}{19\%}

\newcommand{\CostPhiBytes}{512}
\newcommand{\CostPhiOverKv}{6.2\%}
\newcommand{\CostRopeUs}{1386}
\newcommand{\CostRopeGainUs}{1727}

\newcommand{\CostShape}{$B{=}8$, $H{=}16$, $L{=}2048$, $D{=}128$}
\newcommand{\CostFoldOverRope}{2.8\%}

\newcommand{\CostFoldHoistOverRope}{2.4\%}

\newcommand{\CostRounds}{12}
\newcommand{\CostFoldErr}{2e-16}
\newcommand{\CostFoldAmortised}{+0.9\%}
\newcommand{\CostFoldShared}{+6.2\%}

\newcommand{\TwoDMinN}{5}
\newcommand{\TwoDCleanUpto}{4}

\newcommand{\TwoDCodim}{2}

\newcommand{\TwoDCertified}{38}
\newcommand{\TwoDCertFailed}{0}

\newcommand{\TwoDRelAxes}{1664}
\newcommand{\TwoDRelTrivial}{1122}

\newcommand{\CovFarRate}{128}

\newcommand{\CovApertureFarLo}{0.403}
\newcommand{\CovApertureFarHi}{1.000}
\newcommand{\CovLearnFarLo}{0.184}
\newcommand{\CovLearnFarHi}{1.000}

\newcommand{\CovNSeeds}{15}

\newcommand{\CovTrainRates}{$2$--$32$}
\newcommand{\CovLearnCoveredP}{=0.93}

\newcommand{\CovLearnSingleP}{=0.42}

\newcommand{\CovWorstP}{=0.42}

\newcommand{\ScaleNBins}{12}
\newcommand{\ScaleChance}{0.0833}
\newcommand{\ScaleTrainable}{68}
\newcommand{\ScaleWfeatExtra}{4096}
\newcommand{\ScaleWfeatExtraPct}{0.006\%}

\newcommand{\ScaleTrained}{1, 2}
\newcommand{\ScrambleCost}{+0.658\text{ to }+0.682}

\newcommand{\ScaleNPaired}{24}

\newcommand{\ScaleSeedList}{0\text{--}9\text{, }11\text{--}24}

\newcommand{\ScaleFarBudget}{4}
\newcommand{\ScaleVsRote}{+0.0597}

\newcommand{\ScaleVsBlind}{+0.0259}

\newcommand{\ScaleVsBlindP}{=0.051}

\newcommand{\ScaleDidWfeatSel}{+0.0197}

\newcommand{\ScaleSyncN}{24}
\newcommand{\ScaleSyncFar}{+0.0000}

\newcommand{\ScaleSyncInDist}{-0.0055}

\newcommand{\ScaleBinTokRatioMFour}{1.001}

\newcommand{\ScaleTokRote}{+0.0610}

\newcommand{\ScaleSatBudget}{8}

\newcommand{\LadNSeeds}{24}

\newcommand{\LadTokMEight}{+0.0376}

\newcommand{\LadSlopeTok}{+0.0115}

\newcommand{\LadWideBudget}{8}
\newcommand{\LadWideRatio}{50}

\newcommand{\LadNarrowRatio}{6}

\newcommand{\LadCtlRms}{+0.0064}

\newcommand{\LadCtlRmsP}{=0.593}

\newcommand{\LadCtlNyq}{+0.0255}

\newcommand{\LadCtlNyqP}{=0.020}

\newcommand{\ScaleInflN}{24}
\newcommand{\ScaleInflVsRote}{+0.0214}

\newcommand{\ScaleInflVsRoteP}{=0.223}

\newcommand{\ClampAbsVsAp}{-0.0084}

\newcommand{\ClampAbsN}{48}
\newcommand{\ClampAbsCiLo}{-0.0274}
\newcommand{\ClampAbsCiHi}{+0.0105}

\newcommand{\ClampAbsMargin}{0.0199}

\newcommand{\ClampMainVsAp}{+0.0015}

\newcommand{\ClampMainN}{24}
\newcommand{\ClampMainCiLo}{-0.0274}
\newcommand{\ClampMainCiHi}{+0.0303}

\newcommand{\SevScaleTrainable}{106}

\newcommand{\SevScaleVsRote}{+0.0549}

\newcommand{\SevScaleVsBlind}{+0.0018}

\newcommand{\SevScaleVsBlindP}{=0.876}

\newcommand{\SevScaleShapeWfeatFar}{+0.0256}
\newcommand{\SevScaleShapeWfeatIn}{+0.0142}

\newcommand{\SevLadTokMEight}{+0.0197}

\newcommand{\SevLadSlopeTok}{+0.0144}

\newcommand{\SevScaleInflN}{24}
\newcommand{\SevScaleInflVsRote}{+0.0073}

\newcommand{\SevScaleInflVsRoteP}{=0.581}

\newcommand{\SpanGapCoarsenRange}{+0.186\text{ to }+0.202}

\newcommand{\SpanGapEvictRange}{+0.085\text{ to }+0.108}

\newcommand{\VseTrainRate}{8}
\newcommand{\VseFarRate}{32}

\newcommand{\VseNbins}{8}
\newcommand{\VseBinMs}{100}

\newcommand{\VseDuration}{4}
\newcommand{\VseFpsA}{50}
\newcommand{\VseLagMax}{0.4}
\newcommand{\VseSigma}{0.2}

\newcommand{\VseKeep}{0.5}
\newcommand{\VseKeepLo}{0.25}
\newcommand{\VseNlayer}{4}
\newcommand{\VseDmodel}{256}
\newcommand{\VseSteps}{15k}
\newcommand{\VseNval}{2048}
\newcommand{\VseNparams}{3.18M}
\newcommand{\VseEvalFpsLo}{2}
\newcommand{\VseEvalFpsHi}{32}
\newcommand{\VseNtypes}{12}
\newcommand{\VseDsig}{32}

\newcommand{\VseNhead}{8}
\newcommand{\VsePmin}{0.02}
\newcommand{\VsePmax}{8}
\newcommand{\VseBatch}{128}
\newcommand{\VseWd}{0.01}
\newcommand{\VseWarmup}{200}
\newcommand{\VseEvalEvery}{1500}
\newcommand{\VseLr}{3\times10^{-4}}
\newcommand{\VseNruns}{180}
\newcommand{\VseFlatness}{1.27}
\newcommand{\VseSatLo}{0.999}
\newcommand{\VseSatHi}{1.000}

\newcommand{\VseEvTiledLo}{0.935}
\newcommand{\VseEvTiledHi}{1.000}

\newcommand{\VseTiApertureRmsFar}{0.985}

\newcommand{\VseEvApertureRmsFar}{0.997}

\newcommand{\VseTiWallclockWfeatFar}{0.999}

\newcommand{\VseEvWallclockWfeatFar}{0.845}

\newcommand{\VseEvIndexFar}{0.238}

\newcommand{\VseTiNyquistFar}{0.960}

\newcommand{\VseMechTi}{-0.0147}

\newcommand{\VseIndexEvHalf}{0.755}

\newcommand{\CeilNcells}{15}
\newcommand{\CeilNval}{2048}
\newcommand{\CeilRateLo}{1.0000}

\newcommand{\CeilRateMarginRatio}{4.5}
\newcommand{\CeilEvictLo}{0.9976}

\newcommand{\CeilSelfCheck}{6\!\times\!10^{-12}}

\newcommand{\GeomAccCoarsen}{0.5238}
\newcommand{\GeomBaseCoarsen}{0.4643}

\newcommand{\GeomBaseDilate}{0.4508}

\newcommand{\GeomBaseMask}{0.4627}

\newcommand{\GeomBaseEvict}{0.4081}

\newcommand{\GeomEffInversion}{+0.0867}

\newcommand{\GeomEffInversionBootLo}{+0.0663}
\newcommand{\GeomEffInversionBootHi}{+0.1070}
\newcommand{\GeomEffInversionG}{+2.82}

\newcommand{\GeomNBoot}{20{,}000}

\newcommand{\GeomAncContent}{-0.0651}

\newcommand{\GeomPlainTiling}{-0.0077}

\newcommand{\GeomPlainContent}{-0.0677}

\newcommand{\GeomPlainInversion}{+0.0600}

\newcommand{\GeomPlainInversionG}{+2.00}
\newcommand{\GeomPlainInversionBootLo}{+0.0391}
\newcommand{\GeomPlainInversionBootHi}{+0.0798}

\newcommand{\GeomConvFrac}{0.5}

\newcommand{\GeomNrunsDepth}{48}

\newcommand{\GeomSteps}{40k}
\newcommand{\GeomChance}{0.0417}
\newcommand{\GeomMajority}{0.0498}
\newcommand{\GeomNbins}{24}
\newcommand{\GeomBinS}{0.5}
\newcommand{\GeomWindow}{16}
\newcommand{\GeomDvid}{768}
\newcommand{\GeomDaud}{1280}
\newcommand{\GeomDquery}{1536}
\newcommand{\GeomJv}{64}
\newcommand{\GeomJa}{32}
\newcommand{\GeomVidHz}{4}
\newcommand{\GeomNparams}{4.1M}
\newcommand{\GeomNparamsWfeat}{512}
\newcommand{\GeomNparamsWfeatPct}{0.013}
\newcommand{\GeomNTrainItems}{6084}
\newcommand{\GeomNTrainVids}{44}
\newcommand{\GeomNValItems}{2093}
\newcommand{\GeomNValVids}{18}
\newcommand{\GeomNTestItems}{1404}
\newcommand{\GeomNTestVids}{16}
\newcommand{\GeomDmodel}{256}
\newcommand{\GeomNhead}{8}
\newcommand{\GeomNlayer}{4}
\newcommand{\GeomBatch}{64}
\newcommand{\GeomWd}{0.01}
\newcommand{\GeomWarmup}{300}
\newcommand{\GeomNval}{1024}
\newcommand{\GeomAudHz}{2}
\newcommand{\GeomPmin}{0.5}
\newcommand{\GeomPmax}{32}
\newcommand{\GeomMmax}{8}
\newcommand{\GeomTmin}{2}
\newcommand{\GeomTmax}{14}
\newcommand{\GeomOnsetP}{8}
\newcommand{\GeomOnsetThr}{1.4}
\newcommand{\GeomOnsetIsolate}{12}
\newcommand{\GeomLr}{3\times10^{-4}}
\newcommand{\GeomAuditQonly}{0.0723}

\newcommand{\GeomAuditSonly}{0.1074}
\newcommand{\GeomAuditSonlyConv}{0.0977}

\newcommand{\GeomAuditTshuf}{0.0664}

\newcommand{\GeomAuditSonlyCell}{mask}
\newcommand{\GeomAuditSonlyAcc}{0.3701}

\newcommand{\GeomHeadInversion}{-0.0118}

\newcommand{\GeomHeadInversionP}{=0.29}

\newcommand{\GeomCvLo}{0.522}
\newcommand{\GeomCvHi}{0.531}

\newcommand{\GeomCtlStreams}{2048}

\newcommand{\GeomMeanWCoarsen}{1.08}

\newcommand{\GeomLostMask}{7.1}

\newcommand{\GeomDepthCoarsenOne}{0.2552}
\newcommand{\GeomDepthBlindCoarsenOne}{0.3913}
\newcommand{\GeomDepthCoarsenTwo}{0.2256}
\newcommand{\GeomDepthBlindCoarsenTwo}{0.4753}
\newcommand{\GeomDepthCoarsenFour}{0.1787}
\newcommand{\GeomDepthBlindCoarsenFour}{0.4557}
\newcommand{\GeomDepthCoarsenEight}{0.1755}
\newcommand{\GeomDepthBlindCoarsenEight}{0.4551}
\newcommand{\GeomDepthEvictOne}{0.0661}
\newcommand{\GeomDepthBlindEvictOne}{0.4167}
\newcommand{\GeomDepthEvictTwo}{0.0498}
\newcommand{\GeomDepthBlindEvictTwo}{0.4176}
\newcommand{\GeomDepthEvictFour}{0.1084}
\newcommand{\GeomDepthBlindEvictFour}{0.4268}
\newcommand{\GeomDepthEvictEight}{0.1328}
\newcommand{\GeomDepthBlindEvictEight}{0.4528}
\newcommand{\GeomDepthNs}{1, 2, 4, 8}
\newcommand{\GeomDepthNseeds}{3}
\newcommand{\GeomDepthBlindGainCoarsen}{+0.0638}
\newcommand{\GeomDepthGapGainCoarsen}{-0.0798}
\newcommand{\GeomDepthBlindGainEvict}{+0.0361}
\newcommand{\GeomDepthGapGainEvict}{+0.0667}

\newcommand{\GeomDoseFirst}{+0.0732}

\newcommand{\GeomDoseFirstP}{=0.011}

\newcommand{\GeomDoseFirstKeep}{0.875}

\newcommand{\GeomDoseRhoMasked}{+0.095}

\newcommand{\GeomDosePlateauLo}{+0.0990}
\newcommand{\GeomDosePlateauHi}{+0.1316}

\newcommand{\GeomTransNruns}{100}
\newcommand{\GeomTransNitems}{1024}
\newcommand{\GeomTransNcells}{5}

\newcommand{\GeomTransBelowUnanWfeat}{3 of 6}

\newcommand{\GeomTransBelowUnanAperture}{0 of 6}

\newcommand{\GeomTransBelowUnanApertureRms}{2 of 6}

\newcommand{\GeomTransRankRho}{-0.80}

\newcommand{\GeomTransRankRhoMatched}{-0.80}

\newcommand{\GeomTransPoolVsWfeat}{-0.0344}

\newcommand{\GeomTransPoolVsWfeatDil}{+0.0236}

\newcommand{\GeomTransPoolVsWfeatInt}{-0.0510}

\newcommand{\GeomTransPoolVsAperture}{-0.0020}

\newcommand{\GeomTransPoolVsApertureDil}{+0.0135}

\newcommand{\GeomTransPoolPlAperture}{+0.0325}

\newcommand{\GeomTransNseedcells}{105}
\newcommand{\GeomTransNseedcellsPos}{105}

\newcommand{\GeomWshufNruns}{70}
\newcommand{\GeomWshufNblind}{30}
\newcommand{\GeomWshufNmodes}{8}

\newcommand{\GeomWshufFullGap}{+0.1632}

\newcommand{\GeomWshufSwapGap}{+0.0261}

\newcommand{\GeomWshufFullEffInversion}{+0.0934}

\newcommand{\GeomWshufValidEffInversion}{+0.1084}

\newcommand{\GeomWshufPadFracCoarsen}{25.1}

\newcommand{\GeomWshufPadFracEvict}{61.3}

\newcommand{\GeomWshufPairedQuant}{+0.0401}
\newcommand{\GeomWshufPairedQuantSd}{0.0187}

\newcommand{\GeomShamNruns}{20}

\newcommand{\GeomShamFloor}{-0.0018}
\newcommand{\GeomShamFloorSd}{0.0075}
\newcommand{\GeomShamFloorMax}{0.0156}

\newcommand{\GeomShamDelta}{+0.0624}
\newcommand{\GeomShamDeltaSd}{0.0454}

\newcommand{\GeomShamDeltaBlind}{+0.0425}
\newcommand{\GeomShamDeltaBlindSd}{0.0496}

\newcommand{\GeomShamBelowBlind}{-0.0199}

\newcommand{\GeomShamEffTiling}{+0.0000}

\newcommand{\GeomShamEffContent}{-0.0436}

\newcommand{\TmrTemporal}{16}

\newcommand{\TmrLayers}{28}
\newcommand{\TmrPairs}{64}
\newcommand{\TmrPeriodFast}{0.251}
\newcommand{\TmrPeriodSlow}{6.40}
\newcommand{\TmrAudioStepMs}{40}
\newcommand{\TmrAliasedAudio}{0}
\newcommand{\TmrAliasedFpsTwo}{10}
\newcommand{\TmrAliasedFpsOne}{13}
\newcommand{\TmrEnergyQ}{0.920}

\newcommand{\TmrMatchedHQ}{0.807}
\newcommand{\TmrMatchedWQ}{1.004}

\newcommand{\TmrEnergyK}{0.793}

\newcommand{\TmrMatchedHK}{0.714}
\newcommand{\TmrMatchedWK}{1.108}

\newcommand{\TmrSensN}{240}
\newcommand{\TmrSensBase}{0.6667}

\newcommand{\TmrSensShufD}{+0.0000}

\newcommand{\TmrSensRevD}{-0.0042}

\newcommand{\TmrSensConstD}{-0.0500}
\newcommand{\TmrSensSinc}{0.6625}

\newcommand{\TmrSensStretchD}{+0.0000}

\newcommand{\TmrRevN}{246}

\newcommand{\TmrRevBase}{0.5711}
\newcommand{\TmrRevBlind}{0.663}
\newcommand{\TmrRevSinc}{0.5650}

\newcommand{\TmrRevNyq}{0.5630}

\newcommand{\GateNTotal}{240}
\newcommand{\GateNSens}{59}

\newcommand{\GateLooConstN}{50}
\newcommand{\GateLooConst}{-0.1400}
\newcommand{\GateLooConstLo}{-0.2396}
\newcommand{\GateLooConstHi}{-0.0404}
\newcommand{\GateLooConstP}{=0.007}

\newcommand{\GateSensSinc}{-0.0169}
\newcommand{\GateSensSincLo}{-0.1304}
\newcommand{\GateSensSincHi}{+0.0965}

\newcommand{\PubNruns}{36}
\newcommand{\PubSteps}{40k}
\newcommand{\PubNcells}{5}
\newcommand{\PubNitems}{1024}

\newcommand{\PubAccCoarsen}{0.5379}

\newcommand{\PubGapTome}{+0.3113}

\newcommand{\PubAccTome}{0.4900}

\newcommand{\PubGapEvict}{+0.1084}

\newcommand{\PubTomeCoarsen}{+0.1018}

\newcommand{\ManuscriptPlainPairedLo}{+0.0393}
\newcommand{\ManuscriptPlainPairedHi}{+0.0808}
\newcommand{\ManuscriptRmsPairedLo}{+0.0603}
\newcommand{\ManuscriptRmsPairedHi}{+0.1130}

\newcommand{\StatAlpha}{0.05}

\newcommand{\StatNConfirm}{41}

\newcommand{\StatHolmInversion}{\le 0.004}

\newcommand{\GeomDoseFirstAdj}{0.264}

\newcommand{\StatHolmDidAdj}{0.591}

\newcommand{\LadTokMEightAdj}{0.012}
\newcommand{\SevLadTokMEightAdj}{0.784}

\newcommand{\RebuildTransferLo}{0.0155}
\newcommand{\RebuildTransferHi}{0.0494}
\newcommand{\RebuildRawMomentMax}{5.33\times10^{-15}}
\newcommand{\RebuildRmsTreeLoss}{0.2684}
\newcommand{\RebuildQueryBound}{2.46\times10^{-4}}
\newcommand{\RebuildQueryBoundLo}{2.29\times10^{-4}}
\newcommand{\RebuildQueryBoundHi}{2.63\times10^{-4}}
\newcommand{\RebuildEndpointError}{5.34\times10^{-16}}
\newcommand{\RebuildAnchorError}{1.03\times10^{-14}}

\newcommand{\GpuBootReps}{10{,}000}
\newcommand{\GpuBootSeed}{20260910}
\newcommand{\GpuBootNParticipants}{3}

\newcommand{\GpuBootGFourVideos}{3{,}817}
\newcommand{\GpuBootGFourItems}{4{,}000}
\newcommand{\GpuBootGTwoPolicyLo}{-0.0145}
\newcommand{\GpuBootGTwoPolicyHi}{+0.0153}
\newcommand{\GpuBootGTwoRateLo}{-0.0138}
\newcommand{\GpuBootGTwoRateHi}{+0.0172}
\newcommand{\GpuBootGTwoPlacementLo}{-0.0150}
\newcommand{\GpuBootGTwoPlacementHi}{+0.0168}
\newcommand{\GpuBootGThreePolicyLo}{-0.0125}
\newcommand{\GpuBootGThreePolicyHi}{+0.0153}
\newcommand{\GpuBootGFourSpatialLo}{-0.0228}
\newcommand{\GpuBootGFourSpatialHi}{-0.0067}
\newcommand{\GpuBootVidTopGTwoPolicyLo}{-0.0135}
\newcommand{\GpuBootVidTopGTwoPolicyHi}{+0.0133}
\newcommand{\GpuBootVidTopGTwoRateLo}{-0.0133}
\newcommand{\GpuBootVidTopGTwoRateHi}{+0.0153}
\newcommand{\GpuBootVidTopGTwoPlacementLo}{-0.0144}
\newcommand{\GpuBootVidTopGTwoPlacementHi}{+0.0135}
\newcommand{\GpuBootVidTopGThreePolicyLo}{-0.0113}
\newcommand{\GpuBootVidTopGThreePolicyHi}{+0.0131}

\newcommand{\EpicTestParticipants}{3}
\newcommand{\EpicTestParticipantList}{P01, P06, P26}
\newcommand{\EpicTestVideos}{16}

\newcommand{\GpuInvGOneItems}{2{,}048}
\newcommand{\GpuInvGTwoGThreeItems}{1{,}024}

\newcommand{\DenseRefAcc}{0.6515}
\newcommand{\DenseRefMediaErrors}{200}

\newcommand{\DenseRefFullMediaAcc}{0.6878}
\newcommand{\PpeRefAcc}{0.6695}
\newcommand{\PpeRefMediaErrors}{47}
\newcommand{\PpeRefAvailItems}{3{,}953}
\newcommand{\PpeRefAvailAcc}{0.6775}

\newcommand{\GpuGridArmsGTwo}{13}
\newcommand{\GpuGridArmsGThree}{8}

\newcommand{\GpuGridCellSeedAccs}{16{,}800}

\newcommand{\GpuGridAccMin}{0.0635}
\newcommand{\GpuGridAccMax}{0.1201}
\newcommand{\GpuGridCellsAboveTen}{407}

\newcommand{\GpuGridAnomalyLo}{0.1123}
\newcommand{\GpuGridAnomalyHi}{0.1201}
\newcommand{\GpuGridAnomalyCells}{80}
\newcommand{\GpuGridAnomalyPTwoSided}{1.2\times10^{-3}}
\newcommand{\GpuGridConcTrmsMean}{0.1055}
\newcommand{\GpuGridConcRest}{302}
\newcommand{\GpuGridConcTail}{25}
\newcommand{\GpuGridConcTailCkpts}{8}

\newcommand{\ExplBudgetArms}{5}
\newcommand{\ExplBudgetSeeds}{10}
\newcommand{\ExplBudgetCells}{80}
\newcommand{\ExplBudgetLr}{10^{-3}}
\newcommand{\ExplBudgetLnTwelve}{2.4849}
\newcommand{\ExplBudgetMaxLossDev}{0.0107}
\newcommand{\ExplBudgetGridMeanLo}{0.0848}
\newcommand{\ExplBudgetGridMeanHi}{0.0866}
\newcommand{\ExplBudgetChanceFactorLo}{1.0172}
\newcommand{\ExplBudgetChanceFactorHi}{1.0395}
\newcommand{\ExplBudgetContrasts}{8}
\newcommand{\ExplBudgetContrastMeanLo}{-0.0019}
\newcommand{\ExplBudgetContrastMeanHi}{-0.0006}
\newcommand{\ExplBudgetContrastCILo}{-0.0073}
\newcommand{\ExplBudgetContrastCIHi}{0.0044}
\newcommand{\ExplBudgetTestInvSha}{92043cda886e0ab1}
\newcommand{\ExplBudgetTrainLossCentre}{2.4950}
\newcommand{\ExplBudgetGridMeanCentre}{0.0866}
\newcommand{\ExplBudgetFactorCentre}{1.0395}
\newcommand{\ExplBudgetTrainLossStoredSupport}{2.4898}
\newcommand{\ExplBudgetGridMeanStoredSupport}{0.0862}
\newcommand{\ExplBudgetFactorStoredSupport}{1.0342}
\newcommand{\ExplBudgetTrainLossExactInterval}{2.4903}
\newcommand{\ExplBudgetGridMeanExactInterval}{0.0848}
\newcommand{\ExplBudgetFactorExactInterval}{1.0172}
\newcommand{\ExplBudgetTrainLossLearnedInterval}{2.4922}
\newcommand{\ExplBudgetGridMeanLearnedInterval}{0.0850}
\newcommand{\ExplBudgetFactorLearnedInterval}{1.0202}
\newcommand{\ExplBudgetTrainLossPpeRanked}{2.4956}
\newcommand{\ExplBudgetGridMeanPpeRanked}{0.0860}
\newcommand{\ExplBudgetFactorPpeRanked}{1.0324}

\title{Aperture: Merge-Consistent Rotary States\\ for Compressed Tokens}

\author{Yuhao DU$^{1,2}$ \quad Shunian CHEN$^{1}$
\\ {\normalfont $^1$The Chinese University of Hong Kong, Shenzhen}
\\ {\normalfont $^2$Shenzhen Loop Area Institute}
\\ {\normalfont\texttt{yuhaodu1@link.cuhk.edu.cn}}}

\iclrfinalcopy
\renewcommand{\headrulewidth}{0pt}
\begin{document}

\maketitle
\fancyhead{}
\suppressfloats[t]

\begin{abstract}
Token compression combines content from several positions, yet rotary position embeddings
usually assign the merged token one coordinate. We ask what positional information must
survive later merges. Aperture stores Fourier moments of the token's weighted support at
the model's rotary frequencies. We prove that these moments have minimal real dimension among continuous states
sufficient for the selected expected rotary interactions. Represented mass makes updates additive;
attention normalisation remains a separate readout choice. Uniform intervals give a centre
rotation times a sinc gain. We characterise when centres determine interval widths and
construct matched examples where they do not. Numerical checks verify the weighted-support
implementation. In trained temporal readers, compression transfer varies with gain
calibration and feature placement. In a prespecified native video question-answering comparison, stored
support reaches $65.63\%$ accuracy versus $67.12\%$ for the deployed merging rule. These
results separate exact positional preservation under compression from downstream benefit.

\end{abstract}

\section{Introduction}
\label{sec:intro}

A compressor changes the object to which a position belongs. A token that once represented
one frame or patch may, after merging, represent unequal exposures or disconnected regions
\citep{tome,prumerge,apollo}. Attaching a single coordinate preserves its centre but discards
how content is distributed around it.

The natural unit of positional state is therefore the \emph{weighted support} of the token:
the source positions together with the weights used to aggregate them.
Aperture averages their complex rotary phases at each model frequency.
\textbf{The positional state follows the same aggregation as the content.}
\figref{fig:support_principle} shows the result: different merge trees yield the same state
(panel a), which generally differs from the rotary code of the centre (panel b).
Uniform intervals give a rotary position embedding (RoPE) \citep{rope} at the interval centre multiplied by a sinc gain; arbitrary merges
retain the complex moments directly.

\begin{figure}[t]
\centering
\includegraphics[width=\linewidth]{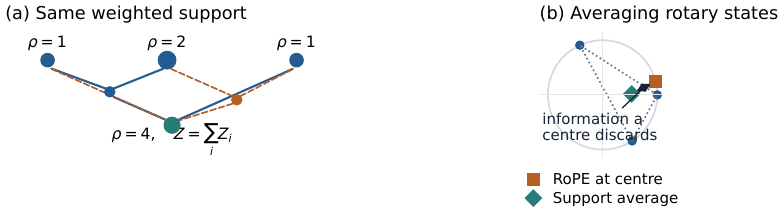}
\caption{\textbf{Position follows the represented support.}
A token merger combines content and positional moments with the same represented masses.
\textbf{(a)} Different merge trees produce the same state for the same weighted support.
\textbf{(b)} A centre summary is also independent of the merge tree, but generally differs from
averaging the original rotary states; the marked chord is the information a centre discards.}
\label{fig:support_principle}
\end{figure}

Integrated positional encoding averages Fourier features over regions
\citep{mipnerf,ipesr,exactnerf}, and Time Interval Encoding (TIE) already derives interval-integrated rotary features
\citep{tie}. Our question concerns the unnormalised state retained through subsequent
weighted merges: which affine summaries suffice for a fixed rotary bank?

\begin{samepage}
\begin{itemize}[leftmargin=1.3em,itemsep=2pt,topsep=3pt]
\item \textbf{A minimal positional state for aggregation.} We characterise the affine state, prove a dimension lower bound for continuous states,
and give additive updates carrying represented mass.
\item \textbf{A constructive account of missing extent.} We identify the centre-preserving width
family, characterise which anchors resolve it, and evaluate matched support queries.
\item \textbf{Distinct roles for storage and readout.} Gain calibration changes temporal contrasts;
native comparisons quantify predictive performance and measured request cost alongside
the representation's algebraic guarantee.
\end{itemize}
\end{samepage}

Storage and readout serve different purposes: additive moments preserve the state through
later merges, while normalising only at attention controls logit magnitude.

Centres can sometimes determine support. In a one-dimensional interval chain, one observed
boundary suffices to recover all widths. Without it, an alternating width perturbation
preserves even the absolute centres. We use this ambiguity to construct matched support
queries, then examine predictive benefit separately on trained readers.

On EPIC-Kitchens onset localisation, the exact interval readout (exact Aperture) exceeds a
scalar-width feature by $\GeomTransPoolPlAperture$ accuracy under held-out content-deleting
policies, a descriptive five-seed contrast; its difference from a centre-only clock is
$\GeomTransPoolVsAperture$, with an interval including zero. Native adaptation comparisons
remain unresolved, while a deployed spatial-merging rule leads stored support by $1.49$
percentage points on video question answering. The practical question is when a trained
reader uses the information that exact aggregation preserves.

\section{Rotary state under token aggregation}
\label{sec:principle}

\subsection{From a position to a support measure}

Let $\mu_i$ be a probability measure describing the source positions and nonnegative scalar
weights represented by token $x_i$. A point observation has measure $\delta_t$; a uniformly
pooled exposure has a uniform interval measure. Under a convex pooling rule,
$x=\sum_j\alpha_jx_j$ carries $\mu=\sum_j\alpha_j\mu_j$, where $\alpha_j\ge0$ and
$\sum_j\alpha_j=1$. The measure specifies aggregation provenance; a nonlinear frontend's
payload need not preserve every source signal within its support.
\figref{fig:aperture}a draws the object the state describes: a weighted support, of which a
centre-only code retains a single point.

\begin{center}
\small
\begin{tabular}{@{}ll@{}}
\toprule
Symbol & Meaning \\
\midrule
$\mu_i,\rho_i$ & Normalised support measure and represented mass of token $i$ \\
$\theta_m,M$ & Rotary frequency and number of complex rotary pairs \\
$z_i,Z_i$ & Normalised Fourier state and raw moments, $Z_i=\rho_i z_i$ \\
$c_i,w_i$ & Centre and width for a uniform interval support \\
\bottomrule
\end{tabular}
\end{center}

\begin{definition}[Aperture state]
\label{def:aperture}
For selected frequencies $\Theta=\{\theta_1,\ldots,\theta_M\}$, the rotary support state is
\begin{equation}
 z(\mu)=\bigl(\phi_{\theta_m}(\mu)\bigr)_{m=1}^M,\qquad
 \phi_\theta(\mu)=\int e^{i\theta t}\,d\mu(t).
 \label{eq:aperture}
\end{equation}
\end{definition}
It obeys
\begin{equation}
 z\!\left(\sum_j\alpha_j\mu_j\right)=\sum_j\alpha_jz(\mu_j).
 \label{eq:merge}
\end{equation}

This is the affine requirement: mixing support measures and mixing their states give the
same result. It lets a compressor update positional information from the summaries of its
children, using the same weights as for content, without revisiting their source positions.

\begin{samepage}
\begin{theorem}[Characterisation and dimension of rotary support state]
\label{thm:rep}
Let $\theta_1,\ldots,\theta_M$ be distinct positive frequencies and $D$ an interval with
nonempty interior; $\mathcal P(D)$ denotes its probability measures.
\textbf{(i)} Any weakly continuous affine state $S:\mathcal P(D)\to\mathbb C^M$ with
$S(\delta_t)=(e^{i\theta_m t})_{m=1}^M$ coincides with $z$; on finite mixtures,
affinity and point calibration suffice without continuity.
\textbf{(ii)} Any affine or weakly continuous state $T:\mathcal P(D)\to\mathbb R^d$
that determines the selected-frequency expected rotary interactions with arbitrary fixed
queries, keys and point probes has $d\ge2M$. The state $z$ attains this bound.
For affine states, sufficiency is exactly a factorisation: the state determines these
interactions if and only if $z$ is an affine function of it on its attained domain.
Any minimal affine sufficient state is equivalent to $z$ by an invertible affine
change of coordinates.
\end{theorem}
\end{samepage}

For (i), point calibration and affinity fix finite mixtures; weak continuity extends the
identity. For (ii), point probes recover each moment's real and imaginary parts.
Independence of $1,\cos(\theta_m t),\sin(\theta_m t)$ gives $2M+1$ points whose
feature vectors span a $2M$-dimensional simplex. A sufficient state is injective on its
mixtures. Affine dimension, or invariance of domain for a continuous state, then gives
the lower bound. Among affine states, sufficiency carries affine relations from state
values to moments, giving the factorisation. On a fixed atomic set, the bound is its
calibrated features' affine rank (proof and domain variants, \secref{sec:app_rep}).

\begin{figure}[t]
\centering
\includegraphics[width=\linewidth]{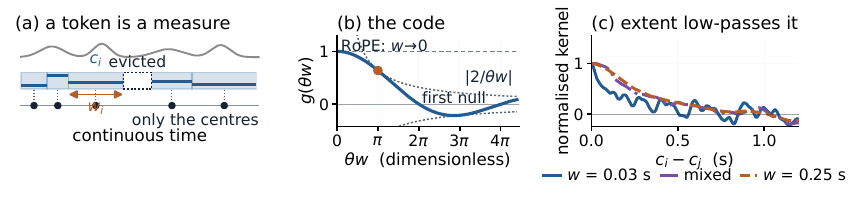}
\caption{\textbf{From token support to rotary gain.}
\textbf{(a)} Retention leaves intervals of different widths and gaps after eviction; centres
alone omit their extent (schematic).
\textbf{(b)} A uniform interval contributes the signed gain $\operatorname{sinc}(\theta w/2)$
with envelope $\min(1,2/(|\theta|w))$. The dot marks the hard-cutoff control's threshold
$\theta w=\pi$, where the exact gain is still $2/\pi$.
\textbf{(c)} Positional kernels $K(\Delta)=\sum_m g_m(w_i)g_m(w_j)\cos(\theta_m\Delta)$,
normalised by $K(0)$, for narrow, wide and mixed supports. Panels (b)--(c) use the model's bank.}
\label{fig:aperture}
\end{figure}

\subsection{Additive storage and attention readout}
\label{sec:props}

When a token represents mass $\rho_i>0$, write $\nu_i=\rho_i\mu_i$ and store
\begin{equation}
 (\rho_i,Z_i)=\left(\nu_i(\mathbb R),\int (e^{i\theta_m t})_{m=1}^M\,d\nu_i(t)\right),
 \qquad (\rho,Z)=\sum_i(\rho_i,Z_i),\qquad z=Z/\rho.
 \label{eq:massstate}
\end{equation}
This update is associative, costs $O(M)$ per pairwise merge, and uses no learned parameters.
Mass denotes represented weight, which need not equal support length. If the compressor supplies
normalised mixture weights instead, \eqref{eq:merge} uses those weights directly.
Mass makes summaries reusable: equal normalised states can represent different weights
and yield different states after merging with the same third token. Carrying mass preserves
this distinction (\secref{sec:app_rep}).

Identify each query/key pair with a complex number. Apply the state by
$\widetilde q_{i,m}=z_{i,m}q_{i,m}$ and $\widetilde k_{j,m}=z_{j,m}k_{j,m}$.
Their real inner product uses $\overline z_{i,m}z_{j,m}$ and equals the expected RoPE logit
for independent support draws, conditional on fixed aggregated $q_i,k_j$
(expected-kernel identity, \secref{sec:app_expected}). This is an identity for the positional
contribution to pre-softmax logits at fixed pooled content.

For a uniform interval with centre $c$ and width $w$,
\begin{equation}
 z_m=e^{i\theta_m c}\operatorname{sinc}(\theta_m w/2),\qquad
 \operatorname{sinc}(u)=\sin(u)/u,\quad \operatorname{sinc}(0)=1.
 \label{eq:sinc}
\end{equation}
In real coordinates this yields
\begin{equation}
 \widetilde q_i^\top\widetilde k_j
 =\sum_m g_m(w_i)g_m(w_j)q_i^{(m)\top}R_m(c_j-c_i)k_j^{(m)},
 \qquad g_m(w)=\operatorname{sinc}(\theta_mw/2).
 \label{eq:logit}
\end{equation}
Here $R_m(t)$ denotes the planar rotation through angle $\theta_m t$.
The signed gain has envelope $\min(1,2/(|\theta_m|w))$ and the point limit is RoPE.
\figref{fig:aperture}b-c draw this gain and its consequence for the induced kernel: the
equal-width kernels retain a maximum at zero displacement, while unequal widths can have
opposite-signed gains. \Secref{sec:app_limits_properties} proves the point limit,
translation covariance and envelope.
A self-pair follows the same independent-draw convention by construction: the diagonal
term carries $g_m(w_i)^2=|z_{i,m}|^2$ rather than the identical-draw value $R_m(0)=I_2$.
The interval formula also follows from repeated equal-width merging
(\secref{sec:app_equalmerge}). Point-sampled tokens instead use the corresponding discrete
Fourier average (\secref{sec:app_atoms}); unequal-weight mixtures require the general state
(\secref{sec:app_measure}). In several dimensions,
replace $\theta t$ by $\omega^\top x$; a uniform box contributes one sinc factor per axis.

\paragraph{Exact storage and terminal normalisation.}
Let $r(z)=(M^{-1}\sum_m|z_m|^2)^{1/2}$. The terminal map
$h_\eta(z)=z/\max(\eta,r(z))$, $\eta>0$, controls attention magnitude while retaining
exact $(\rho,Z)$ for subsequent merges; storing $h_\eta(z)$ at every intermediate merge
generally loses that property. The \emph{Aperture RMS} readout applies this map to the support
code, or to its gains in interval-based studies, without overwriting the raw moments.
Protocol-specific floors appear in \secref{sec:app_cpu_state}.
For common nonzero gains and unrestricted query/key projections, gains can be absorbed into
those projections (gauge invariance at uniform width, \secref{sec:app_gauge}); heterogeneous
supports and constrained adaptation are therefore central to the empirical evaluation.

\subsection{Auditing aggregation consistency}
\label{sec:cpu_state}

Do different merge trees preserve the same weighted-support state?
The implementation retains mass and unnormalised complex moments. We compare direct summation
with balanced, left-branching, right-branching and random merge trees on weighted point,
interval and rectangle supports. Independent Gauss--Legendre quadrature checks the exposure
formulas. Across 64 mixtures in one and two dimensions, the largest raw-moment discrepancy
is $\RebuildRawMomentMax$ in float64 (\secref{sec:app_cpu_state}).

\begin{table}[htbp]
\centering
\small
\caption{\textbf{Tree equality and affinity are different.} Float64 checks on four adjacent intervals of widths $(0.2,0.7,0.8,1.3)$, with mass proportional to width. Tree discrepancy recomputes each readout from the final associative summary. Affine defect compares the parent code with the weighted child codes; it is a numerical witness, not an accuracy score.}
\label{tab:support_algebra}
\begin{tabular}{lcc}
\toprule
Readout & Tree discrepancy & Affine defect \\
\midrule
Centre RoPE & $<10^{-12}$ & $1.347$ \\
\textbf{exact Aperture} & $<10^{-12}$ & $<10^{-12}$ \\
Aperture RMS & $<10^{-12}$ & $0.432$ \\
MN-sinc & $<10^{-12}$ & $1.831$ \\
Hard cutoff & $<10^{-12}$ & $0.577$ \\
\bottomrule
\end{tabular}
\end{table}

\tabref{tab:support_algebra} contrasts exact Aperture with centre RoPE, terminal
root-mean-square (RMS) normalisation, mean-normalised sinc (MN-sinc), and a hard cutoff.
Every readout is tree-independent when computed from the same final associative summary;
only exact Aperture here also equals the weighted average of child codes.
Normalising at each internal merge and discarding the norm instead produces discrepancy
$\RebuildRmsTreeLoss$. Thus tree equality alone does not certify affine aggregation.

The trained studies distinguish the interval readouts from \emph{stored support}, which
uses general support moments, and Positional Preservation Embedding (PPE), which uses
original member coordinates. The 3B/7B and native spatial paths form moments from retained
membership; incremental storage is checked here. The small EPIC reader uses known intervals.
Published operators and their experimental adaptations are
specified in \secref{sec:app_related} and \secref{sec:app_gpu}.

\section{Extent identifiability under centre observations}
\label{sec:redundancy}

\subsection{A centre-preserving width family}

For $N\ge2$ contiguous intervals with widths $w_j$ and centres $c_j$, adjacency gives
\begin{equation}
 (Dw)_j=\tfrac12(w_j+w_{j+1})=c_{j+1}-c_j.
 \label{eq:offset}
\end{equation}
The nullspace is spanned by $u=(1,-1,1,-1,\ldots)$. Hence $w'=w+\delta u$ has the same
centre offsets whenever
$-\min_{u_j=1}w_j<\delta<\min_{u_j=-1}w_j$.
If the left endpoint moves from $a$ to $a-\delta/2$, every absolute centre also remains fixed.
The resulting supports are positive, ordered and contiguous; their boundaries have changed.

\begin{proposition}[Boundary and parity determine identifiability]
\label{prop:red}
Given absolute centres and the left boundary $a$, the widths are recovered by
\begin{equation}
 w_1=2(c_1-a),\qquad w_{j+1}=2(c_{j+1}-c_j)-w_j.
 \label{eq:redundancy}
\end{equation}
Absolute centres alone leave the alternating mode. An observed total span removes that
mode for odd $N$ and leaves it for even $N$.
\end{proposition}
The proof is in \secref{sec:app_anchored}; \tabref{tab:nullity} tabulates the nullity under four observation models.
A beginning-of-sequence token or causal ordering supplies a boundary only when its numeric position is tied
to the observed chain; the theorem applies to an explicit geometric interface,
not to every transformer architecture by default.

\begin{minipage}{\linewidth}
\centering
\small
\captionof{table}{\textbf{Nullity of the width equations}
under four observation models (anchor audit of \secref{sec:app_cpu_queries}).}
\label{tab:nullity}
\begin{tabular}{@{}lcc@{}}
\toprule
Observation & $N{=}8$ & $N{=}9$ \\
\midrule
Relative offsets only & 1 & 1 \\
Absolute centres & 1 & 1 \\
Absolute centres, left boundary & 0 & 0 \\
Absolute centres, total span & 1 & 0 \\
\bottomrule
\end{tabular}
\end{minipage}

\label{sec:breakage}
\thmref{thm:reach}
(\secref{sec:app_relative}) gives the
identified width interval and the nullity after eviction; the reconstruction studies in
\tabref{tab:ladder} (\secref{sec:app_ladder}) distinguish optimal ambiguity from fitted-reader error; finite-span recovery is examined
in \secref{sec:app_span}.
In two dimensions, identical rectangle-centre sets can coexist with different extents
(planar centre ambiguity, \secref{sec:app_twod}), though the support-state construction itself applies in either dimension.

\subsection{Matched support queries}
\label{sec:ladder}
\label{sec:spanlaw}

Which observations resolve the ambiguity left by centres?
The alternating family provides paired examples with identical centres, payload, token order
and sequence length: \figref{fig:support_ambiguity} draws one such pair, whose two chains share
every interval centre, differ in every boundary, and place the marked query in the first
interval only in the lower chain. We query a Fourier integral over a selected token's support,
which changes
with its width. A reader observing only the matched inputs must return the same answer for both
members. Its squared-error lower bound is one quarter of their squared target difference.
Supplying width, endpoints, or the relevant support moment resolves this constructed ambiguity.
The query probes support metadata; the label of a real onset at a fixed event time is unchanged.

\begin{figure}[t]
\centering
\includegraphics[width=\linewidth]{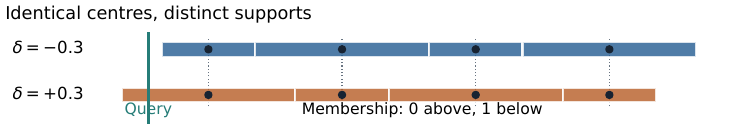}
\caption{\textbf{A width change invisible to centre observations.}
Two valid tilings have the same centres and different boundaries ($\delta=\pm0.3$ on the
alternating mode).
The marked query point belongs to the first interval only in the lower chain; content can be
identical in this constructed pair. Which observations break the ambiguity is algebraic:
\propref{prop:red} states the boundary and parity dependence, and
\tabref{tab:nullity} tabulates the nullity of the width equations for $N=8$ and $N=9$.}
\label{fig:support_ambiguity}
\end{figure}

Across 1,024 held-out base streams, the paired centre-only lower bound is
$\RebuildQueryBound$ mean squared error, with a base-stream bootstrap interval
$[\RebuildQueryBoundLo,\RebuildQueryBoundHi]$.
Width and endpoint decoders recover the queried integral to numerical precision; the relevant
Fourier coordinate supplies it directly, and an RMS state admits an exact decoder on the
chosen bounded width range (\secref{sec:app_cpu_queries}). The experiment identifies missing
information at the centre-only interface, without ranking these information-equivalent encodings.

\section{Predictive effects of support representations}
\label{sec:real}

\subsection{In-distribution accuracy and compression shift}
\label{sec:instrument}
\label{sec:transfer}

Does width information transfer differently when supplied through content or rotary interactions?
Our EPIC-Kitchens-100 task asks a causally masked reader when a queried appearance change occurs
\citep{epickitchens}. Each $\GeomWindow$\,s window contains frozen SigLIP\,2 video and
Whisper audio features \citep{siglip2,whisper}; the query supplies content identity and the
reader predicts one of $\GeomNbins$ onset bins. Training, validation and test participants
are disjoint; the fixed test population spans \EpicTestParticipants{} participants
(\EpicTestParticipantList) over \EpicTestVideos{} videos. Construction and checkpoint
selection are in
\secref{sec:app_onset}; statistical estimators are in \secref{sec:app_stats}.
Additional controls vary retention rules (\secref{sec:app_pub}) and reader depth and
attention window (\secref{sec:app_depth}).

\begin{figure}[t]
\centering
\includegraphics[width=\linewidth]{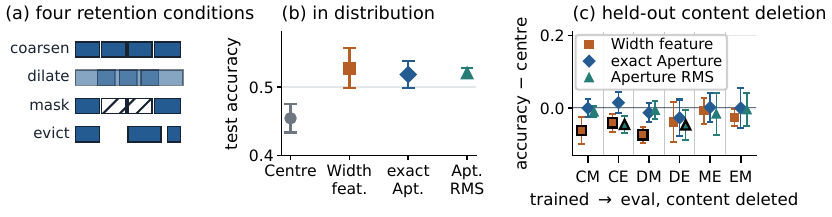}
\caption{\textbf{Retention conditions and transfer under content deletion.}
\textbf{(a)} The four conditions: coarsen (C) merges adjacent supports into a plain tiling; dilate (D)
widens them into overlap; mask (M) zeroes payloads in place (hatched); evict (E) deletes slots (gaps).
\textbf{(b)} Coarsen-trained, coarsen-tested accuracy per arm on the five retained checkpoints;
bars are 95\% $t$ intervals, and the $y$ axis is truncated to $[0.40,0.58]$.
\textbf{(c)} Accuracy minus the centre-only arm in the six train$\to$eval cells whose evaluation
deletes content (cell codes are the panel-a conditions, train$\to$eval),
paired within seed; bars are one seed standard deviation, and a black rim marks
cells below centre on every seed (width feature \GeomTransBelowUnanWfeat; Aperture RMS
\GeomTransBelowUnanApertureRms; exact Aperture \GeomTransBelowUnanAperture).
Full cells and sample sizes are in \secref{sec:app_trans}.}
\label{fig:epic}
\end{figure}

Width information can enter the content pathway as a scalar feature, or modulate query/key
interactions through the support gain. In distribution, both recover much of the improvement
over a centre-only clock (\figref{fig:epic}b). Across held-out content-deleting policies, the
feature arm differs from the clock by $\GeomTransPoolVsWfeat$, and exact Aperture by
$\GeomTransPoolVsAperture$. The direct Aperture-minus-feature contrast is
$\GeomTransPoolPlAperture$, with a five-seed $t$ interval
$[\RebuildTransferLo,\RebuildTransferHi]$ and an exact two-sided sign-flip $p=0.0625$.
We report this five-seed contrast descriptively.

\subsection{Interventions distinguish dependence from benefit}
\label{sec:predictions}

Does a reader depend on widths, or benefit from them? For a fixed model, width sensitivity is
\begin{equation}
 \Delta_w=\operatorname{acc}-\operatorname{acc}(\text{permuted widths}).
 \label{eq:gap}
\end{equation}
It measures dependence on metadata; accuracy differences between separately trained arms
measure predictive benefit. The centre-only arm has exactly zero sensitivity.
\figref{fig:epic}a draws the four conditions: coarsen keeps a plain tiling, dilate breaks it by
overlap, mask deletes content in place, and evict deletes slots.
On the four-policy study, dilation retains more width sensitivity than masking, with
$D-M=\GeomEffInversion$ for Aperture RMS and $\GeomPlainInversion$ for exact Aperture.
These are contrasts between interventions. Overlap, masking and eviction do not form a
fully crossed geometry/content design (\secref{sec:factorial}). Live-slot and sham-width
controls establish the scope of the permutation result (\secref{sec:app_wshuf}, \secref{sec:app_sham}).

Complete overlap-by-mask and gap-by-mask crossings on a continuous onset signal recompute
each token's integral under the physical-support protocol and hold it fixed under the
metadata protocol, keeping the event time and mask draw matched across
cells. These CPU controls separate support mismatch from content removal and expose the role
of a boundary-aware centre decoder (\secref{sec:app_crossed}); the trained
real-video crossing is reported in \secref{sec:gpu}.

\subsection{Sampling rate and gain calibration}
\label{sec:synthetic}
\label{sec:ratetask}
\label{sec:mech}
\label{sec:closedform}
\label{sec:synthlimits}

Can the same support readout transfer across sampling rates?
A continuous-time audiovisual task holds the event and lag label fixed while changing the
sampling rate. Each token is an analytic integral over its support, and training includes
random contiguous merges. At $\VseFarRate$\,fps, compared with $\VseTrainRate$\,fps in training,
Aperture RMS and the width-feature arm reach $\VseEvApertureRmsFar$ and
$\VseEvWallclockWfeatFar$, respectively. The mean gap is driven by variable feature-arm
performance; its paired test does not survive the full multiplicity correction.
A hard cutoff also transfers well, and energy normalisation benefits this rate-shift setting.
Together with the real policy comparison, these results make placement and calibration
empirical choices. The generator, full seed distributions and rate curves are in
\secref{sec:app_rate} and \figref{fig:rate}; gain-argument coverage is examined in
\secref{sec:app_coverage}.

\subsection{Controlled large-model comparisons}
\label{sec:scale}

Do differences between interval codes persist when their gain magnitude is matched?
Qwen2.5-Omni 3B/7B \citep{qwen25omni} experiments use frozen audiovisual towers
(\secref{sec:app_tmrope} diagnoses the 7B tower's temporal rotary configuration without
adaptation), low-rank adaptation (LoRA) and a shared
content-similarity readout aid. Training uses adjacent merge budgets
$b\in\{\ScaleTrained\}$; the primary held-out budget is $b=\ScaleFarBudget$.
\tabref{tab:scale_summary} reports $\ScaleNPaired$ matched seeds at each scale; training details and gain controls appear in
\secref{sec:app_scale} and \secref{sec:app_ladder_scale}.

\begin{table}[t]
\centering
\setlength{\tabcolsep}{4pt}
\caption{\textbf{Large-model contrasts at the primary held-out budget.} Mean paired difference and 95\% $t$ interval on 24 matched seeds per comparison. Intervals are unadjusted and in accuracy units; family-wise conclusions are stated in the text. Matched-RMS rescales exact Aperture to the temporal gain RMS of MN-sinc.}
\label{tab:scale_summary}
\small
\begin{tabular}{@{}lcc@{}}
\toprule
Paired accuracy contrast & 3B & 7B \\
\midrule
exact Aperture $-$ centre only & $+0.0259\;[-0.0001,+0.0518]$ & $+0.0018\;[-0.0214,+0.0249]$ \\
exact Aperture $-$ width feature & $+0.0114\;[-0.0134,+0.0362]$ & $+0.0256\;[+0.0004,+0.0508]$ \\
exact Aperture $-$ MN-sinc & $+0.0597\;[+0.0294,+0.0900]$ & $+0.0549\;[+0.0298,+0.0800]$ \\
Matched-RMS $-$ MN-sinc & $+0.0214\;[-0.0140,+0.0568]$ & $+0.0073\;[-0.0198,+0.0345]$ \\
\bottomrule
\end{tabular}
\end{table}

Matching gain RMS substantially reduces the mean Aperture--MN-sinc difference at both
scales (\tabref{tab:scale_summary}), making gain magnitude important to this comparison.
The centre-only intervals include zero at the primary budget. At exploratory
$b=\LadWideBudget$, exact Aperture's token-accuracy difference from centre-only is
$\LadTokMEight$ at 3B and $\SevLadTokMEight$ at 7B
(Holm-adjusted $p=\LadTokMEightAdj$ and $p=\SevLadTokMEightAdj$, respectively;
$\LadNSeeds$ paired seeds). This metric counts selection of the onset-bearing token
(\secref{sec:app_ladder_scale}). \Secref{sec:gpu} tests compression-aware adaptation and
native model behaviour under a frozen protocol.

\subsection{Native interface comparisons}
\label{sec:gpu}

Does preserving support improve a model's native predictions, and at what cost?
The protocol fixes the evaluation cells, seeds, checkpoint selection and five primary
contrasts before execution (\secref{sec:app_gpu}). Paired-seed intervals measure variation
across trained models; a crossed bootstrap also resamples participant--video--item clusters.
The five contrasts form a separate Holm family. \tabref{tab:gpu_pending} reports both
uncertainty analyses, and \figref{fig:gpu_eval}a shows the paired seed differences.
The EPIC bootstrap resamples \GpuBootNParticipants{} top-level participants; the systems
measurements use one A800-SXM4-80GB device class.

\begin{table}[t]
\centering
\small
\caption{\textbf{Native interface contrasts, in accuracy units.} Stored support minus the
named reference (centre-only unless specified). Onset uses ten paired seeds and 1,024
items; spatial question answering uses five seeds and 4,000 items. Seed $t$ intervals
condition on the test set; crossed bootstrap intervals resample seeds and population
clusters. Both 95\% intervals are unadjusted; Holm correction covers the five seed tests.}
\label{tab:gpu_pending}
\begin{tabular}{@{}lccc@{}}
\toprule
Comparison & Difference & 95\% seed interval & 95\% crossed bootstrap \\
\midrule
Qwen2.5-Omni-3B: unseen policy & $+0.0001$ & $[-0.0068,+0.0070]$ & $[\GpuBootGTwoPolicyLo,\GpuBootGTwoPolicyHi]$ \\
Qwen2.5-Omni-3B: unseen rate & $+0.0007$ & $[-0.0078,+0.0092]$ & $[\GpuBootGTwoRateLo,\GpuBootGTwoRateHi]$ \\
Qwen2.5-Omni-3B: width feature & $-0.0007$ & $[-0.0062,+0.0049]$ & $[\GpuBootGTwoPlacementLo,\GpuBootGTwoPlacementHi]$ \\
Qwen2.5-Omni-7B: unseen policy & $+0.0007$ & $[-0.0028,+0.0042]$ & $[\GpuBootGThreePolicyLo,\GpuBootGThreePolicyHi]$ \\
Qwen2.5-VL-7B: spatial / PPE & $-0.0149$ & $[-0.0175,-0.0123]$ & $[\GpuBootGFourSpatialLo,\GpuBootGFourSpatialHi]$ \\
\bottomrule
\end{tabular}
\end{table}

\begin{figure}[t]
\centering
\includegraphics[width=\linewidth]{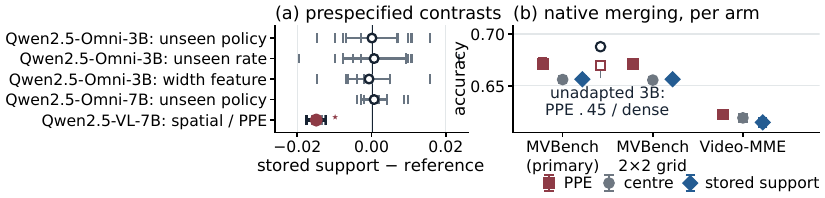}
\caption{\textbf{Native interface comparisons.}
\textbf{(a)} The five contrasts of the frozen manifest (stored-support accuracy minus the named
reference): effect, 95\% paired-seed $t$ interval, and per-seed differences as ticks. The filled
marker ($\star$) is the only contrast significant after Holm correction (adjusted $p=4.7\times10^{-4}$).
\textbf{(b)} Native merging per arm -- MVBench at the deployed selector (primary cell), MVBench
$2{\times}2$ grid, Video-MME -- as seed means with 95\% $t$ intervals over five shared seeds;
$y$ axis truncated to $[0.605,0.708]$. Hollow markers: unadapted 3B references (PPE at retention
$0.45$; dense, complete media), not adapted-7B comparators.}
\label{fig:gpu_eval}
\end{figure}

\paragraph{Real-video crossings.}
On EPIC-Kitchens with frozen features, each method and seed uses a separate model for each
crossing and payload protocol, evaluated in its four geometry$\times$content cells
(\tabref{tab:gpu_crossed_pending}). Masking half the payloads while preserving onset-bearing tokens raises accuracy by $0.116$
to $0.130$ in all 28 groups. This target-preserving intervention measures the effect of
removing surrounding content while retaining token timing. Supplied
support geometry moves accuracy by at most $0.006$ for the
stored-support reader; for the centre-only reader it vanishes under the metadata
protocol by construction, and under the physical protocol it moves only through the
recomputed payload. The full grid's largest geometry effect is $0.0069$, its
largest interaction $0.0081$.

\paragraph{Compression-aware adaptation at 3B and 7B.}
The 3B grid trains LoRA adapters at 8\,fps and evaluates across unseen policies, fractions
and frame rates (\figref{fig:gpu_null}). At the frozen
budget (3,000
steps, learning rate $10^{-4}$) all three prespecified intervals include zero
(\tabref{tab:gpu_pending}); the 7B confirmation gives the same unresolved ranking.
This grid uses a learned readout without the content-similarity aid of
\tabref{tab:scale_summary}.
Accuracies span $\GpuGridAccMin$--$\GpuGridAccMax$ against chance $1/12$.
An exploratory 3B comparison at ten times the learning rate also remains near chance.
The near-chance grid leaves the interface ranking unresolved; \secref{sec:app_gpu}
reports the full distributions and training-batch probe.

\paragraph{Native spatial merging.}
The deployed-interface comparison adapts Qwen2.5-VL-7B \citep{qwen25vl} on a declared
video-QA mixture, evaluating the complete MVBench \citep{mvbench} test split through the
original question/answer interface at one-quarter retention. Stored support
scores $0.6563$ against $0.6712$ for PPE under the released clustering interface, using
original member coordinates: $-0.0149$ (raw $p=9.4\times10^{-5}$, Holm-adjusted
$4.7\times10^{-4}$), all five seed differences
negative. The direction repeats under $2{\times}2$ grid merging ($0.6714$ vs $0.6563$) and
on Video-MME \citep{videomme} ($0.6224$ vs $0.6146$), while centre-only and stored support stay
within $0.005$ throughout ($0.6560$ vs $0.6563$ on MVBench; \figref{fig:gpu_eval}b).
The native rule's advantage persists across the tested selectors and benchmarks.

\paragraph{Request timing and memory.}
At one-quarter retention, classification requests including media decoding have centre-only
and stored-support medians of $5.19$\,s and $5.20$\,s, with identical $22.5$\,GiB peaks.
The other classification points also have similar medians and identical peaks
(\tabref{tab:classification_cost}). Generative timing starts after media preprocessing.
On the 256-item MVBench timing subset, PPE, centre-only and stored
support take $984$, $931$ and $1{,}334$\,ms, respectively; timing-subset accuracies are
$0.777$, $0.676$ and $0.692$. Charged positional buffers occupy $0.527$\,MiB for stored
support versus $4.237$\,MiB for PPE, beside a $320$\,MiB key--value cache. The timer covers
model feature extraction, compression, prefill and autoregressive decoding; variable output
lengths contribute to latency. Complex moments are computed from membership and not
separately metered (\secref{sec:app_gpu}).

\section{Related work}
\label{sec:related}

\paragraph{Fourier features and support.}
RoPE rotates queries and keys by relative displacement \citep{rope}.
Integrated positional encodings average Fourier features over spatial regions
\citep{mipnerf,ipesr,exactnerf}; finite Fourier features approximate kernels
\citep{randomfeatures}, and mean embeddings represent distributions by feature expectations
\citep{kernelmean,hilbertembed}. TIE integrates rotary features over intervals \citep{tie}.
We characterise the affine state preserved through weighted merges and establish
its dimension minimality among continuous sufficient states.

\paragraph{Position in compressed multimodal models.}
PPE assigns original member position IDs to rotary channel groups \citep{ppe}; Sync-RoPE
reallocates temporal and spatial channels for audiovisual alignment \citep{echoingpixels}.
Position interpolation, frequency scaling and multimodal coordinates make complementary
clock and axis choices \citep{pi,yarn,qwen2vl,videorope}; we distinguish them from the
MN-sinc and width-capped phase controls (\secref{sec:app_related}).

\paragraph{Compression and evaluation.}
Merging changes individual token supports; retention changes the set available to attention
\citep{tome,prumerge,apollo,h2o,snapkv,streamingllm}. Aperture supplies positional state to
such operations, not token selection. Native multimodal models and
streaming benchmarks test broader behaviour \citep{qwen25omni,streamingbench,ovobench,streamingeval}. Our comparisons isolate positional interfaces within fixed
compression policies and readouts.

\section{Conclusion and limitations}
\label{sec:discussion}

Aperture carries the moments sufficient for expected rotary interactions through weighted
token aggregation. Normalisation is a readout choice. The centre-preserving width
family identifies the ambiguity in point coordinates and which boundary observations resolve it.

The empirical findings separate the state's algebraic guarantee from its use by a trained
reader. Temporal transfer varies with placement and gain calibration. Folding interval gains costs
$\CostFoldOverRope$ rotary-step overhead (\secref{sec:app_cost}), while native request
cost includes membership processing and response length (\secref{sec:gpu}). Storage,
readout and deployment cost therefore require separate evaluation.

\paragraph{Limitations.}
The guarantee assumes a known support measure and shared nonnegative aggregation weights.
A finite frequency bank preserves selected interactions, not arbitrary source content;
learning reliable support metadata is a further problem. The temporal tests cover
\EpicTestParticipants{} EPIC test participants, and native question answering uses one
model interface. The near-chance onset adaptation leaves the ranking of successfully
trained interfaces unresolved. Extending the
evaluation to well-trained onset readers and other native interfaces would test where the
retained information improves prediction.

\FloatBarrier
\label{manuscript:bodyend}
\section*{Reproducibility statement}
The support assumptions and proofs appear in Appendix~\ref{sec:app_proofs}.
Appendix~\ref{sec:app_new} specifies the executable CPU audits and
Appendix~\ref{sec:app_gpu} the GPU protocol, selection rules and recorded outcomes. Appendix~\ref{sec:app_experiments} specifies the experimental protocols, estimators and
complete result grids. The typesetting source includes all numerical tables, figures and bibliography entries.

\section*{AI use statement}
Generative AI assisted code development, manuscript writing, mathematical and reference checks, the authors reviewed and verified all AI-assisted work.

\bibliography{references}
\bibliographystyle{iclr2027_conference}

\clearpage
\appendix
\section{Mathematical analysis}
\label{sec:app_proofs}

For a uniform interval, write $c$ for its centre and $w$ for its width.
Let $R_m(t)$ denote the planar rotation through angle $\theta_m t$.
With the convention $\tilde q_i^{(m)}=R_m(c_i)q_i^{(m)}$ and
$\tilde k_j^{(m)}=R_m(c_j)k_j^{(m)}$,
the unmodified channel logit is $q_i^{(m)\top}R_m(c_j-c_i)k_j^{(m)}$.

\subsection{Expected rotary interactions}
\label{sec:app_expected}

\begin{proposition}[Expected-kernel identity]
\label{prop:kernel}
Let $t_i\sim\mu_i$ and $t_j\sim\mu_j$ be independent. The expectation of
$R_m(t_j-t_i)$ is the real matrix representing multiplication by
$\overline{\phi_{\theta_m}(\mu_i)}\phi_{\theta_m}(\mu_j)$.
In particular, if both measures are uniform intervals with centres $c_i,c_j$
and widths $w_i,w_j$, then
\begin{equation}
  \mathbb{E}[R_m(t_j-t_i)]
  =g_m(w_i)g_m(w_j)R_m(c_j-c_i).
  \label{eq:expected}
\end{equation}
For fixed query and key vectors, the expected RoPE logit is therefore
\eqref{eq:logit}.
\end{proposition}

\begin{proof}
Identify each channel pair with $\mathbb{C}$. Independence gives
\begin{equation}
  \mathbb{E}[e^{i\theta(t_j-t_i)}]
  =\overline{\mathbb{E}[e^{i\theta t_i}]}\,
    \mathbb{E}[e^{i\theta t_j}]
  =\overline{\phi_\theta(\mu_i)}\phi_\theta(\mu_j).
  \label{eq:prod}
\end{equation}
For a uniform interval, direct integration gives
$\phi_\theta(\mu)=e^{i\theta c}\operatorname{sinc}(\theta w/2)$.
Substitute this expression into \eqref{eq:prod} and take its real matrix
representation. Multiplication on the left by $q_i^{(m)\top}$ and on the right
by $k_j^{(m)}$, followed by summation over channels, proves the logit identity.
\end{proof}

For a general support symmetric about $c$, the centred transform is real:
$\phi_\theta(\mu)=e^{i\theta c}\int\cos(\theta(t-c))\,d\mu(t)$.
Symmetry alone does not imply the uniform-interval $\operatorname{sinc}$ gain.
An asymmetric support generally has a frequency-dependent phase in addition
to its magnitude. The product identity \eqref{eq:prod} remains valid in both cases.

The interval gains are signed. A channel interaction reverses relative to its
centre-only value when the two gains have opposite signs; two negative gains
have a positive product. Thus replacing the gain by its absolute value changes
some expected interactions. This algebraic difference does not imply an
accuracy difference: the sign and first-lobe ablations in
\secref{sec:app_ladder_scale} do not statistically distinguish the trained models.

The identity also has two probabilistic boundaries. It conditions on fixed,
already aggregated query and key vectors, and uses independent support draws.
It does not equate the expectation of softmax with softmax of the expected logits.
Even for a self-pair, independent draws from the same support are different
from assigning the same sampled position to both sides.

\subsection{Point limit, translation and attenuation}
\label{sec:app_limits_properties}

\begin{lemma}[Point-support limit]
\label{lem:pointlimit}
If $\mu$ converges weakly to $\delta_c$, then $\phi_\theta(\mu)\to e^{i\theta c}$.
For uniform intervals, the zero-width gain is exactly $\operatorname{sinc}(0)=1$.
\end{lemma}

\begin{proof}
Weak convergence $\mu\to\delta_c$ implies
$\int e^{i\theta t}\,d\mu(t)\to e^{i\theta c}$ because the test function is bounded
and continuous. For uniform intervals, $\operatorname{sinc}(0)=1$ gives the
exact zero-width limit.
\end{proof}

\begin{lemma}[Translation covariance]
\label{lem:transcov}
Let $(\tau_a\mu)(B):=\mu(B-a)$. Then
$\phi_\theta(\tau_a\mu)=e^{i\theta a}\phi_\theta(\mu)$, and the interaction of
jointly translated supports depends only on their relative displacement and shapes.
\end{lemma}

\begin{proof}
For the translated measure $\tau_a\mu$,
$\phi_\theta(\tau_a\mu)=e^{i\theta a}\phi_\theta(\mu)$.
The two factors cancel in the conjugate product of jointly translated supports.
The interaction consequently depends on relative displacement and support
shape, rather than on the shared time origin.
\end{proof}

\begin{lemma}[Spectral envelope]
\label{lem:envelope}
The uniform-interval gain satisfies
$|\operatorname{sinc}(u)|\le\min(1,1/|u|)$ for $u=\theta w/2$, with zeros at
$\theta w=2k\pi$ for nonzero integers $k$.
\end{lemma}

\begin{proof}
The inequalities $|\sin u|\le|u|$ and $|\sin u|\le1$ give
$|\operatorname{sinc}(u)|\le\min(1,1/|u|)$.
Substitute $u=\theta w/2$. Zeros occur at $\theta w=2k\pi$ for nonzero
integers $k$. The bound is an envelope; it does not make the gain monotone
or eliminate all frequencies above a cutoff.
\end{proof}

The hard-cutoff control sets gains to zero when $\theta w>\pi$.
At this threshold the exact gain is $2/\pi$, so the control is a distinct
spectral rule rather than a numerical approximation to the first sinc zero.

\subsection{A representation compatible with aggregation}
\label{sec:app_rep}

Let $\mathcal{A}$ be a collection of atomic support measures and let
$\mathcal{M}=\operatorname{conv}_{\mathrm{fin}}(\mathcal{A})$ contain their finite
convex mixtures. Atomic supports specify the temporal measurement model:
they may be points, exposure windows or other known pooling kernels.

\begin{proof}[Proof of \thmref{thm:rep}]
Identify $\mathbb C^M$ with $\mathbb R^{2M}$, put $n=2M$, and write
\[
 f(t)=(\cos\theta_1t,\sin\theta_1t,\ldots,
       \cos\theta_Mt,\sin\theta_Mt),
 \qquad F(\mu)=\int f(t)\,d\mu(t).
\]
For part (i), affinity and point calibration give
$S(\sum_j\alpha_j\delta_{t_j})=\sum_j\alpha_jf(t_j)=F(\mu)$
on finite point mixtures, without a continuity assumption.
Finite point measures are weakly dense in $\mathcal P(D)$, and each coordinate of
$f$ is bounded and continuous. Weak continuity of $S$ therefore extends the
identity to all of $\mathcal P(D)$.

For part (ii), fix any point probe at $a\in D$ and isolate one rotary channel.
Arbitrary real query and key directions determine the expected two-dimensional
rotation matrix against that probe. Its entries give the real and imaginary
parts of $\phi_{\theta_m}(\mu)$ up to a known invertible rotation and conjugation.
Consequently, sufficiency implies
\[
 T(\mu)=T(\eta)\quad\Longrightarrow\quad F(\mu)=F(\eta).
\]
Conversely, if $F$ is a function of $T$ on its attained domain,
\propref{prop:kernel} expresses each selected expected interaction of fixed
queries and keys as a function of the two $F$-values, and hence of the two
$T$-values. No continuity of this recovery function is needed.

The functions $1,\cos\theta_m t,\sin\theta_m t$ are linearly independent on
$D$: a finite exponential polynomial vanishing on an interval with nonempty
interior has every coefficient zero. It follows that there are
$t_0,\ldots,t_n\in D$ for which $f(t_0),\ldots,f(t_n)$ are affinely independent.
Let
\[
 U=\{p\in\mathbb R^n:p_j>0\ (1\le j\le n),\ \sum_{j=1}^n p_j<1\},
 \qquad
 \mu_p=\left(1-\sum_{j=1}^n p_j\right)\delta_{t_0}
             +\sum_{j=1}^n p_j\delta_{t_j}.
\]
The set $U$ is nonempty and open in $\mathbb R^n$.
The map $p\mapsto\mu_p$ is weakly continuous, since integration against any
bounded continuous test function is affine in $p$.
Moreover,
\[
 F(\mu_p)=f(t_0)+\sum_{j=1}^n p_j\bigl(f(t_j)-f(t_0)\bigr)
\]
is injective by affine independence. The sufficiency implication above makes
$p\mapsto T(\mu_p)$ injective as well.

If $T$ is affine, this is an injective affine map from the open set $U$ into
$\mathbb R^d$, so its linear part has rank $n$ and $d\ge n$.
If $T$ is weakly continuous, its restriction is instead a continuous injective
map $U\to\mathbb R^d$. Suppose $d<n$ and compose it with the coordinate
inclusion $\iota:\mathbb R^d\hookrightarrow\mathbb R^n$.
Invariance of domain states that a continuous injective map from an open subset
of $\mathbb R^n$ to $\mathbb R^n$ has open image \citep[Theorem~2B.3]{hatcher}.
Thus $\iota(T(\{\mu_p:p\in U\}))$ would be open in $\mathbb R^n$.
It lies in the proper linear subspace $\iota(\mathbb R^d)$, which has empty
interior, a contradiction. Hence $d\ge n=2M$ in either class.
The state $F$, or equivalently $z$, is affine, weakly continuous and sufficient,
so it attains the bound in both classes.

It remains to prove the affine factorisation. Assume $T$ is affine and
sufficient. Consider any finite affine relation
$\sum_i\lambda_iT(\mu_i)=0$ with $\sum_i\lambda_i=0$.
If its coefficients are not all zero, their positive and negative parts have
the same positive sum $s$. The probability measures
\[
 \mu_+=s^{-1}\sum_{\lambda_i>0}\lambda_i\mu_i,
 \qquad
 \mu_-=s^{-1}\sum_{\lambda_i<0}(-\lambda_i)\mu_i
\]
have equal $T$-values by affinity. Sufficiency gives equal $F$-values, so
$\sum_i\lambda_iF(\mu_i)=0$. The same conclusion holds for the zero relation.
Thus the rule $T(\mu)\mapsto F(\mu)$ respects every affine relation and extends
to an affine map on the affine hull of the attained $T$-values. This proves
necessity of the stated factorisation; its sufficiency follows from the
expected-kernel identity as above. At the minimum dimension $d=n$, the image
of $F$ has affine span $\mathbb R^n$ by the simplex construction, so the
factor map has rank $n$. Its domain therefore also has affine dimension $n$,
and the factor map is an invertible affine change of coordinates.
\end{proof}

This characterisation is the finite-feature affine extension underlying Fourier features
and kernel mean embeddings \citep{kernelmean,hilbertembed,randomfeatures}.
Its sufficiency is for the selected interactions, not for reconstructing arbitrary measures.
For convex mixtures of a restricted atomic family, with the same point probes, the lower
bound for affine or weakly continuous states is the affine rank of its calibrated feature vectors. A domain with $K$ fixed atoms has rank at most $K-1$. Fixed-width exposure atoms
can have zero response at particular frequencies, reducing the rank further. Duplicate
frequencies, opposite-frequency pairs and the constant zero-frequency probability moment
also do not supply independent coordinates.

More generally, let $\mathcal A$ contain calibrated atomic support measures, including
exposures. For $\mu=\sum_j\alpha_j\nu_j$ with $\nu_j\in\mathcal A$, affinity gives
$\phi_\theta(\mu)=\sum_j\alpha_j\phi_\theta(\nu_j)$, independent of the
mixture decomposition. Mixtures of one fixed nonzero-width exposure family are not weakly
dense in all probability measures; a zero-width limit varies the atomic model.

\paragraph{The role of represented mass.}
For a finite positive measure, storing $(\rho,Z)$ as in \eqref{eq:massstate} makes aggregation
additive. Normalised moments alone cannot implement every later additive merge without
external mass: two measures with equal normalised state $z$ and distinct masses $\rho_1,\rho_2$
produce $(\rho_i z+ap)/(\rho_i+a)$ after adding a point mass of mass $a>0$ and state $p\ne z$.
The two outputs differ. If mass and raw moments must be recoverable from a real-linear
state, independence of $1,\cos,\sin$ gives dimension $2M+1$ on the unrestricted atomic domain.
An external mass already supplied by the compressor counts toward this storage.

\paragraph{Tree equality and the affine identity.}
For identical final weighted support, any associative summary followed by a deterministic
readout is independent of the merge tree for a fixed leaf order; a summary whose tree
also permutes leaves would additionally need to be commutative. This includes a
mass-weighted centre followed by
RoPE and exact moments followed by terminal RMS normalisation. Affinity is the stronger,
different requirement that coding the parent agree with averaging the child codes.
The centre map fails this identity in general because
$e^{i\theta\sum_j\alpha_jc_j}\ne\sum_j\alpha_je^{i\theta c_j}$.
Repeatedly projecting intermediate complex states onto unit RMS can lose the discarded norm
and depend on the tree. Retaining $(\rho,Z)$ and normalising once avoids that loss.
When state, content, ordering, masks and all other inputs coincide, a fixed deterministic
reader has equal outputs. Different final partitions or intervening attention operations
need not meet those premises.

Table~\ref{tab:rep} checks affinity on equal-width merges and common-translation
covariance for the implemented alternatives. These are necessary-condition
tests on the chosen supports. A small residual in a finite test is not a
proof of the corresponding identity on arbitrary measures; a nonzero
counterexample suffices to show that a code violates it.

\begin{table}[htbp]
\centering
\setlength{\tabcolsep}{4pt}
\caption{\textbf{Numerical checks of support affinity and translation covariance.} Largest residual in float64 over sampled supports. A1 tests merging equal-width adjacent intervals; A2 tests a common translation. Values below $10^{-9}$ are numerical zero. A1 is a necessary special case of affinity, not a test over all measures. The point-support limit holds for all tested codes to $1.1\times10^{-16}$. MN-sinc and the width-capped phase control are defined in Appendix~\ref{sec:app_related}.}
\label{tab:rep}
\small
\begin{tabular}{lrrl}
\toprule
Code & A1: affinity & A2: translation & Violated \\
\midrule
Aperture & $1.1\times 10^{-14}$ & $4.9\times 10^{-13}$ & --- \\
Centre only & $2.0$ & $5.7\times 10^{-13}$ & A1 \\
Hard cutoff & $6.5\times 10^{-1}$ & $2.8\times 10^{-14}$ & A1 \\
Aperture RMS & $1.6$ & $5.3\times 10^{-13}$ & A1 \\
MN-sinc & $1.5\times 10^{2}$ & $5.3\times 10^{-13}$ & A1 \\
Width-capped phase & $1.5$ & $1.7$ & A1+A2 \\
\bottomrule
\end{tabular}
\end{table}

\subsection{General mixtures and positional state}
\label{sec:app_measure}

The merge identity follows directly from linearity of the integral:
\[
  \phi_\theta(\mu)=\sum_j\alpha_j\phi_\theta(\mu_j)
  \quad\text{when}\quad \mu=\sum_j\alpha_j\mu_j.
\]
Store $z_m=\phi_{\theta_m}(\mu)$ at each rotary frequency and update
$z_m\leftarrow\sum_j\alpha_j z_{j,m}$ at a merge.
This takes $O(M)$ operations for a fixed number of members and retains complex
phase as well as gain. Hierarchical averaging must use the represented masses;
an unweighted mean of unequal-mass groups is a different measure.

Equal-weight mixtures of adjacent intervals with unequal widths have a
nonuniform density and generally cannot be represented by a single uniform
interval. A $(c,w)$ summary therefore need not reproduce their Fourier code.
The measured residual of this summary is $\CovUneqAperture$, compared with
$\CovUneqEqCtrl$ for the equal-width control. Using the union centre
($\CovUneqApertureUnion$) or a variance-matched width
($\CovUneqApertureVar$) does not restore exactness.
Equation~\eqref{eq:merge} handles these mixtures without a uniformity assumption.
If the mixture weights are proportional to the widths of adjacent disjoint
uniform intervals, the mixture is uniform on their union and the compact
interval formula is exact.

\subsection{Point samples and finite-width atoms}
\label{sec:app_atoms}

For $n$ equally weighted point samples at pitch $\Delta$, centred at $c$,
the Fourier code is
\begin{equation}
  \phi_\theta(\mu)
  =e^{i\theta c}\,
    \frac{\sin(n\theta\Delta/2)}{n\sin(\theta\Delta/2)},
  \label{eq:dirichlet}
\end{equation}
with removable singularities evaluated from the finite sum.
For $n$ adjacent uniform atoms of width $\Delta$, the mixture instead fills
an interval of width $w=n\Delta$, giving
$e^{i\theta c}\operatorname{sinc}(\theta w/2)$.
The distinction is a sampling assumption: exposure duration need not equal
frame pitch in a physical camera, and the interval model represents the
pooling support used in the experiment.

At fixed $w$ with $\Delta=w/n\to0$, the point-sample expression converges to
the interval expression. Away from their singularities, their ratio is
$(\theta\Delta/2)/\sin(\theta\Delta/2)$, reaching $\AtomRatioPi$ at
$\theta\Delta=\pi$.
The numerical audit over $n\le\AtomNGrid$ and $\theta\Delta\le2\pi$ gives
interval-mixture residual $\AtomIntervalErr$, point-versus-interval discrepancy
$\AtomPointErr$, and finite-sum agreement with \eqref{eq:dirichlet} of
$\AtomDirichletId$.
These checks distinguish an exact interval code from a discrete-sampling approximation.

On the Qwen temporal frequency bank, $\AtomQwenPastNullTrained$ channels lie
beyond their first zero at the finest trained budget, increasing to
$\AtomQwenPastNullHeld$ at the held-out budget. The smallest held-out gain
magnitude is $\AtomQwenMinGainHeld$. These are properties of the chosen
frequency/width pairs, and explain why the signed-gain and magnitude controls
are relevant. They do not identify which channels the trained model uses.

\subsection{Uniqueness under repeated equal-width merging}
\label{sec:app_equalmerge}

\begin{proposition}[Equal-merge characterisation]
\label{prop:unique}
Fix $\theta>0$ and a positive width set $\mathcal{W}$ closed under halving.
Let
$\varphi_\theta(c,w)=e^{i\theta c}H(w)$,
where $H$ is continuous at zero and $H(0)=1$.
If averaging the codes of two adjacent equal-width intervals gives the code
of their union at every width in $\mathcal{W}$, then
$\varphi_\theta(c,w)=e^{i\theta c}\operatorname{sinc}(\theta w/2)$
for every $w\in\mathcal{W}$.
\end{proposition}

\begin{proof}
The two children of an interval of width $w$ have width $w/2$ and centres
$c-w/4$ and $c+w/4$. Merge consistency gives
\begin{equation}
  H(w)=\cos(\theta w/4)H(w/2).
  \label{eq:funeq}
\end{equation}
Iterating,
\[
  H(w)=\left[\prod_{j=1}^{n}\cos(\theta w/2^{j+1})\right]H(w/2^n).
\]
The last factor converges to one. Repeated application of the sine
double-angle identity gives
$\prod_{j=1}^{\infty}\cos(x/2^j)=\sin x/x$.
Taking $x=\theta w/2$ proves the result.
\end{proof}

The characterisation is stable on a finite ladder.
Write the merge defect at scale $x$ as $r(x)=H(x)-\cos(\theta x/4)H(x/2)$, and
suppose $|r(x)|\le\varepsilon$ at $x=w,w/2,\ldots,w/2^{n-1}$.
Iterating the recurrence as in the proof and collecting the defects, each
multiplied by a cosine product of magnitude at most one, gives
\begin{equation}
  \bigl|H(w)-\operatorname{sinc}(\theta w/2)\bigr|
  \le n\varepsilon+\bigl|H(w/2^n)-\operatorname{sinc}(\theta w/2^{n+1})\bigr|.
  \label{eq:funstable}
\end{equation}
A depth-$n$ ladder with rung defect at most $\varepsilon$ therefore certifies the
interval code up to $n\varepsilon$ plus the base calibration error. With zero merge defect, continuity at zero makes the
remaining term vanish as $n\to\infty$. With nonzero defects, this is a finite-depth bound.

The proposition identifies the combined complex multiplier $H$.
If a code is written as $e^{i\theta(c+s(w))}h(w)$, then
$H(w)=e^{i\theta s(w)}h(w)$; the separate functions $s$ and $h$ are not
uniquely identified. Choosing a signed real sinc gain with centroid phase is
a canonical implementation of the uniquely determined code.

A finite deployed width ladder has an alternative calibration.
Given the atomic value $H(w_0)$, equal-merge consistency determines
$H(2^kw_0)$ recursively without continuity at zero.
If $H(w_0)=\operatorname{sinc}(\theta w_0/2)$, every rung has the interval code.
Without either atomic calibration or a zero-width limit, the base value remains
free; merge consistency on a finite ladder alone cannot choose it.

For comparison, a hard cutoff violates \eqref{eq:funeq}; at
$\theta w/2=1$, its two sides are $1$ and $\cos(1/2)$.
The measured residual is $\CovNyquist$.
RMS normalisation also changes the recurrence, with residual
$\CovApertureRms$. These violations concern representation consistency,
not a ranking of trained-model accuracy.

For the mean-normalised control in \eqref{eq:mncontrol}, away from denominator clamping,
merge consistency would require
\begin{equation}
 \frac{C_{2w}}{C_w}
 =\frac{\cos(\theta_m w)}{\cos(\theta_m w/2)}
 \label{eq:rotecond}
\end{equation}
where the cancelled factors are nonzero.
The right side varies across channels, while the left side is a shared scalar.
The tested bank therefore violates the identity.
A half-width, mean-normalised gain, matching RoTE's interval convention, also generally
violates it: after cancelling the sinc recurrence, consistency requires an unchanged
normaliser between parent and child widths. The measured half-width-normalised residual
is $\RoteSmallHarmOne$ at $w=1$\,s in the small bank.
These counterexamples concern affine aggregation, not predictive performance.

\subsection{Uniform-width reparameterisation}
\label{sec:app_gauge}

\begin{proposition}[Gauge invariance at uniform width]
\label{prop:gauge}
Suppose every token has the same width $w$ and let
$\Gamma=\bigoplus_m g_m(w)I_2$.
The Aperture attention map with projections $(W_q,W_k)$ equals the centre-only
map with projections $(\Gamma W_q,\Gamma W_k)$.
If all gains are nonzero and the projection matrices are unrestricted,
the two hypothesis classes coincide.
\end{proposition}

\begin{proof}
Define $R(c)=\bigoplus_m R_m(c)$. Scalar gain blocks commute with their planar
rotations, so $\Gamma R(c)=R(c)\Gamma$.
Substituting this identity on the query and key sides reproduces every logit.
When $\Gamma$ is invertible, $W\mapsto\Gamma W$ is a bijection on unrestricted
projection matrices, proving equality of the representable function classes.
\end{proof}

The numerical folding residual is $\GaugeUniform$ at logit scale
$\GaugeScale$; using identical, unfolded projections instead gives difference
$\GaugePlain$.
A point-support readout query among finite-width content tokens violates
the common-gain premise (relative residual $\GaugeQueryRel$).
Heterogeneous content widths likewise cannot generally be absorbed into one
shared projection (residual $\GaugeHetRel$ in the tested configurations).
Zero gains make the map noninvertible.
Finally, frozen or constrained projections, initialisation and regularisation
need not be preserved by the parameter transformation.
Uniform-width reparameterisation is a function-class result under its stated conditions,
not a prediction that a fixed-rate benchmark must give equal accuracies.

\subsection{Anchored extent recovery}
\label{sec:app_anchored}

\begin{proof}[Proof of \propref{prop:red}]
Write the intervals as $[a_j,a_{j+1})$, with known $a_1=a$.
Since $c_j=(a_j+a_{j+1})/2$, the recursion
$a_{j+1}=2c_j-a_j$ recovers all boundaries and therefore all widths.
With only $a_1$ fixed, $(a_2,\ldots,a_{N+1})$ has $N$ free coordinates
subject to strict ordering. Its affine map to the centres has a triangular
Jacobian with diagonal $1/2$, hence is invertible.
A separately fixed right boundary removes one degree of freedom.
For unanchored centres, $w'=w+\delta u$ can preserve their absolute values by moving each
boundary $a_j$ to $a_j+(-1)^j\delta/2$, indexing $a_1$ as the left endpoint.
The total span changes by $\delta\sum_j u_j$, equal to $\delta$ for odd $N$ and zero
for even $N$. Hence a known span resolves the alternating degree of freedom exactly in
the odd case. Absolute timestamps alone supply no endpoint equation.
\end{proof}

\subsection{Identifiability from relative offsets}
\label{sec:app_relative}

\begin{theorem}[Width ambiguity under relative observations]
\label{thm:reach}
Let a reader observe only relative centre offsets of a one-dimensional
interval chain, without a boundary or total-span constraint.
\begin{enumerate}[leftmargin=1.6em,itemsep=2pt,topsep=3pt,label=(\roman*)]
  \item For $N\ge2$ adjacent intervals, $\operatorname{rank}D=N-1$ and
  $\ker D=\operatorname{span}\{(1,-1,\ldots)\}$.
  The compatible positive widths form $w+\delta u$ with
  $-m_+<\delta<m_-$, where $m_+=\min_{u_j=1}w_j$ and $m_-=\min_{u_j=-1}w_j$.
  \item In a contiguous observation window containing at least two tokens,
  define $m_+$ and $m_-$ as the minimum widths over its two parity classes.
  The identified interval for any target width has length $A=m_++m_-$.
  Every estimator has supremum absolute error at least $A/2$, and the
  midpoint of the identified interval attains this bound.
  \item If $J\ge2$ survivors are retained from a $2N$-interval chain,
  the relative-offset matrix $E$ has rank $J-1$ and nullity $2N-J+1$.
  No surviving width is uniquely determined by those offsets alone.
\end{enumerate}
\end{theorem}

\begin{proof}[Part (i)]
The first $N-1$ columns of $D$ form an upper triangular matrix with diagonal
$1/2$, so the row rank is $N-1$.
The alternating vector $u$ satisfies $Du=0$, proving the nullspace statement.
Strict positivity gives $-m_+<\delta<m_-$.
Adjacent indices have opposite parity, hence
\begin{equation}
  A=m_++m_-\le2\min_k(c_{k+1}-c_k).
  \label{eq:ambbound}
\end{equation}
For example, $(1,5,1,5)$ and $(5,1,5,1)$ have identical relative offsets;
the perturbation magnitude need not be bounded by the smallest original width.
\end{proof}

\begin{proof}[Part (ii)]
Apply part (i) to the observation window. External widths impose no additional
constraints under the specified observation model.
Every admissible value in the interval gives the same observations.
For two values separated by $A-\varepsilon$, a common estimate has error at
least $(A-\varepsilon)/2$ for one of them.
Letting $\varepsilon\downarrow0$ proves the supremum bound.
The midpoint of the identified interval deviates from every admissible value
by less than $A/2$, with supremum exactly $A/2$ over the open interval, so the
bound is tight.
A single-token window has no offset observation and instead has unbounded
positive width ambiguity.
\end{proof}

\begin{proof}[Part (iii)]
Let survivor indices be $j_1<\cdots<j_J$. Consecutive survivor offsets are
\begin{equation}
  c_{j_{r+1}}-c_{j_r}
  =\tfrac12w_{j_r}+\sum_{j_r<k<j_{r+1}}w_k+\tfrac12w_{j_{r+1}}.
  \label{eq:evictrow}
\end{equation}
These form the rows of $E$.
In a vanishing linear combination of rows, column $j_1$ forces the first
coefficient to zero. Successive survivor columns then force every coefficient
to zero, so the rows are independent and the stated rank and nullity follow.
Moreover, $E$ sums rows of the full-chain operator $D$ and therefore annihilates
its alternating vector $u$. Every survivor coordinate of $u$ is nonzero.
A sufficiently small positive or negative perturbation along $u$ preserves
all widths' positivity while changing every surviving width and leaving
the observed offsets fixed.
\end{proof}

Eviction increases the nullspace dimension, but nonidentifiability does not
mean every surviving width has unbounded uncertainty. With at least two
survivors, each surviving endpoint in \eqref{eq:evictrow} is bounded above by
twice an adjacent observed offset.
An absolute boundary, known total span, or additional content-derived geometry
changes the observation model and may reduce ambiguity.
The theorem places no lower bound on the benefit of an extent code for a
downstream label.

\paragraph{Finite-bank resolution at the rank level.}
A stored finite bank supplies such geometry, and the rank it adds is computable
channel by channel, complementing the discrete affine-rank count of
\secref{sec:app_rep}. For a uniform interval of centre $c$ and width $w$,
channel $m$ contributes $g_m(w)(\cos\theta_m c,\sin\theta_m c)$ to the
calibrated feature vector, with centre and width partial derivatives
$\theta_m g_m(w)(-\sin\theta_m c,\cos\theta_m c)$ and
$g_m'(w)(\cos\theta_m c,\sin\theta_m c)$. The two are orthogonal, with squared
norms $\theta_m^2 g_m(w)^2$ and $g_m'(w)^2$, so the bank's Jacobian in $(c,w)$
has rank two unless every channel gain vanishes or every channel gain is
stationary. An envelope zero at $\theta_mw=2k\pi$
(Lemma~\ref{lem:envelope}) removes channel $m$'s centre coordinate; a stationary
gain, where $\tan(\theta_mw/2)=\theta_mw/2$, including the zero-width limit,
removes its width coordinate. Whenever some channel gain is non-stationary at
some interval width, the alternating null vector of \thmref{thm:reach} changes
the stored codes at first order: the mode invisible to relative offsets is
visible to the bank. If every gain is stationary at every interval width, the
finite bank shares the offsets' first-order blindness. This is a linearised
rank computation, not a resolution theory: which width differences a finite
bank identifies at a given signal-to-noise ratio remains open.

\subsection{The scope of the one-dimensional result}
\label{sec:app_twod}

\begin{proposition}[Planar centre ambiguity]
\label{prop:twod}
Two axis-aligned rectangular tilings of $[0,3]^2$ can have identical centroid
multisets and different ordered dimension multisets, even when both are guillotine tilings.
\end{proposition}

\begin{proof}
Let $A$ contain the two columns $[0,1]\times[0,3]$ and
$[2,3]\times[0,3]$, together with the three unit boxes
$[1,2]\times[k,k+1]$, $k=0,1,2$.
Let $B$ be its transpose.
Both tile the square and have centroid set
\[
  \{(1/2,3/2),(3/2,1/2),(3/2,3/2),(3/2,5/2),(5/2,3/2)\}.
\]
Their ordered dimension multisets, as $(x,y)$ extent pairs, are
$\{(1,1)^3,(1,3)^2\}$ and $\{(1,1)^3,(3,1)^2\}$, respectively.
Full-length cuts at $x=1,2$ for $A$, and at $y=1,2$ for $B$,
give guillotine decompositions.
\end{proof}

The witness extends to a two-parameter family on $[0,1]^2$.
For $a,c\in(0,1/2)$, divide $A$ into columns of widths $a,1-2a,a$
and split its middle column at heights $c,1-c$.
Divide $B$ into rows of heights $c,1-2c,c$ and split its middle row
at $x=a,1-a$.
Both centroid sets form the same five-point cross.
Anisotropic scaling transfers the construction to any rectangular frame.

The exact combinatorial search reports no collision for
$N\le\TwoDCleanUpto$ and a dimension-changing family at $N=\TwoDMinN$.
An independent coordinate-based certifier checks total area, disjoint interiors,
centroid equality and differing ordered dimensions, with $\TwoDCertified$ accepted
witnesses and $\TwoDCertFailed$ failures.
The family has codimension $\TwoDCodim$ within its fixed-type parameter space;
local full-rank checks within one type do not test cross-type ambiguity.

With the rectangular dissection type and centroid-to-box correspondence known,
a different recovery statement holds. Order the distinct cut coordinates
$0=x_0<\cdots<x_p$. Each interior coordinate is a right edge of some box $i$,
so $x_k=2c_i^x-x_{l_i}$ recovers it from a smaller indexed coordinate.
The same recursion applies in $y$. Thus labelled centres determine a fixed
dissection, without implying uniqueness across unlabelled types.

The relative ambiguity also depends on dimension.
In the enumerated planar types, $\TwoDRelTrivial$ of $\TwoDRelAxes$
coordinate-axis systems have a trivial unanchored kernel.
The universal alternating mode is therefore specific to a one-dimensional
chain. The support transform itself extends directly to higher dimensions as
$\phi_\omega(\mu)=\int e^{i\omega^\top x}\,d\mu(x)$; the chain-recovery theorem
does not make that extension.

\section{Support-state implementation and controlled interventions}
\label{sec:app_new}

\subsection{General-support audit}
\label{sec:app_cpu_state}

The state implementation stores one positive mass and one unnormalised complex moment per
frequency. Point, interval and rectangle constructors use their analytic Fourier transforms.
An additive merge sums masses and raw moments; a separate probability-mixture interface
requires nonnegative coefficients summing to one and combines normalised child states.
These interfaces distinguish represented-mass addition from externally specified convex pooling.
Optional atomic provenance is retained for auditing and is not part of the compact deployed state.

The float64 audit uses 32 random mixtures in one dimension and 32 in two dimensions,
with seven atoms and seven frequency vectors per mixture. Centres, positive masses and
box dimensions vary independently. Mixtures can overlap, be disconnected and have asymmetric
weights. The same atoms are combined by balanced, left-branching, right-branching and random
trees. Independent tensor-product Gauss--Legendre quadrature evaluates the atomic exposures;
points use direct complex exponentiation. Translation covariance and narrowing boxes to
points are also checked. Generator seeds, trial counts, frequencies, errors and source hashes are saved
with the machine-readable audit.

\begin{table}[htbp]
\centering
\small
\caption{\textbf{General-support implementation audit.} Maximum absolute raw-moment tree/quadrature errors over 32 mixtures per dimension, seven weighted atoms and seven frequencies per mixture. Interval-summary error is measured on normalised moments using one uniform box over the enclosing bounds.}
\label{tab:support_general}
\begin{tabular}{lcc}
\toprule
Support domain & Tree / quadrature error & Interval-summary error \\
\midrule
1D mixtures & $<10^{-12}$ & $0.764$ \\
2D mixtures & $<10^{-12}$ & $0.699$ \\
\bottomrule
\end{tabular}
\end{table}

The uniform-box comparison in \tabref{tab:support_general} uses the enclosing bounds of the entire mixture. It differs from
the true support in both density and connectivity; the displayed discrepancy is a witness
of summary misspecification. A separate asymmetric-support check finds a nonzero imaginary
part after subtracting the centre phase, confirming that a real gain alone cannot express
every support. Two distinct positive point mixtures are also constructed with identical
selected Fourier moments by using a null vector of the $[1;\cos;\sin]$ feature matrix.
This numerically illustrates the finite bank's noninjectivity.

\paragraph{Aggregation and normalisation.}
The adjacent-interval audit uses edges $(0,0.2,0.9,1.7,3)$ and angular frequencies
$(0.3,0.7,1.2,2.4,4.1)$. Mass equals interval width in this audit, making the parent measure
uniform on $[0,3]$. Final centre and endpoints are computed by associative summaries for every
readout. All codes therefore pass the tree-equality control; only the exact Fourier transform
in this example passes the separate affine check in \tabref{tab:support_algebra}.
The repeated-RMS experiment instead normalises each internal weighted mean and discards its
previous norm. It is a distinct lossy update rule and gives the nonzero tree discrepancy.

A three-layer deterministic attention reader, with hidden dimension 10 and no stochastic
operations, receives equal pooled content, masks and final token partitions across trees.
Exact moments, Aperture RMS and centre-summary controls all give equal outputs to roundoff.
Changing the final partition yields a nonzero discrepancy even with exact moments, because
the content tokens and softmax normalisation differ. These are implementation checks, with
no fitting, accuracy metric or significance test.

\paragraph{Readout floors.}
The RMS map $h_\eta$ averages squared magnitudes over the encoded frequency pairs.
The general support API uses $\eta=10^{-12}$ over its supplied bank; the native 3B/7B
onset code uses $10^{-6}$ over temporal pairs; the synthetic, small EPIC and assisted-scale
interval readers use $10^{-4}$ over their gain bank (temporal pairs only at scale).
These are denominator floors: the latter two implementations clamp the mean square before
its square root. They change the terminal readout, not the underlying measure.

\paragraph{Support misspecification.}
A separate 32-mixture audit changes observed atomic widths by a common factor $1+\epsilon$,
with $\epsilon\in\{-0.2,-0.1,0,0.1,0.2\}$. True support, atomic centres and positive masses
remain fixed. The mean maximum-channel errors are approximately
$0.0600,0.0300,0,0.0295,0.0580$, respectively. The stored state remains tree-consistent for
the supplied measure, including an incorrectly specified one. Exact aggregation therefore
preserves declared metadata; it does not correct errors in the atomic measurement model.
The per-atom errors also give a convex upper bound on normalised mixture error.

\subsection{Paired support queries and anchors}
\label{sec:app_cpu_queries}

The generator samples positive base widths, chooses an admissible alternating perturbation,
and forms both branches of each base stream. Their absolute centres and payloads coincide;
the left endpoint moves by the amount required to preserve adjacency. A query selects a token
and asks for its uniform-support expectation of $\cos t$, measuring time relative to the
first token centre. This target is a known support functional, not an event-time label.
Both branches stay in the same split. Train, validation and test sets contain
1,024, 256 and 1,024 independently generated base streams, respectively, with seeds
18401, 28603 and 39817.

For a pair of targets $y_-,y_+$ that share the observed input, squared loss is minimised by
the conditional mean $(y_-+y_+)/2$. Its loss is $(y_+-y_-)^2/4$ on either branch.
Averaging that value over test base streams gives the lower bound in the main text.
The 95\% interval resamples 1,024 whole base streams 2,000 times; it keeps both branches
and the queried token together. The analytic width and start/end decoders have maximum
error at most $\RebuildEndpointError$.

The exact Fourier input contains the queried frequency and therefore its answer directly.
For the RMS bank $(0.25,0.5,1,2)$, the positive width range is below $2\pi$.
After removing the known centre phase, the ratio of gains at frequencies 2 and 1 is
$\cos(w/2)$; normalisation cancels. Hence $w=2\arccos(g_2/g_1)$ recovers width in this
domain and supplies another exact decoder. All of these are information-access controls.
Small learned feature readouts are retained in the supplementary numerical archive, but
are not used as an encoding-performance comparison: one input includes the target coordinate
and the task does not compare attention-kernel placement.

The anchor audit solves the full linear system for chains of lengths 8 and 9.
Absolute centres without endpoints leave rank $N-1$. A numeric left boundary in their
coordinate system gives rank $N$; total span does so only for odd $N$.
The maximum reconstruction error in the full-rank cases is $\RebuildAnchorError$.
Order, token count and causal masks are held fixed and supply no extra numeric boundary.

\subsection{Orthogonal support and content interventions}
\label{sec:app_crossed}

The CPU onset instrument contains a single continuous step on $[0,8]$ seconds with known
pre- and post-step feature vectors, supplied to every decoder as the query. Onset time is
sampled independently of the token geometry. Each base stream defines a tiling and one
Bernoulli feature mask shared by all geometry variants. The mask leaves token slots and
metadata in place and does not protect the onset-bearing token.

Two separate $2\times2$ crossings compare tiling with overlapping supports, and tiling with
gapped supports, each under retained and masked content. Geometry keeps the centres and
slot count fixed. In the physical-support protocol, payloads are recomputed as exact integrals
over the new intervals. In the metadata-only protocol, original payloads are retained and
only supplied support changes. Both protocols hold the true event time fixed. All four cells
in a crossing share base-stream identity and mask draw.

For per-stream absolute error $e_{g,m}$, the geometry effect is
$[(e_{1,0}-e_{0,0})+(e_{1,1}-e_{0,1})]/2$, the masking effect is
$[(e_{0,1}-e_{0,0})+(e_{1,1}-e_{1,0})]/2$, and the interaction is
$e_{1,1}-e_{1,0}-e_{0,1}+e_{0,0}$. Uncertainty resamples base streams, preserving the full
crossing. In the physical protocol, changing support also changes the content measurement;
these effects refer to independently assigned operators, not identical payload after pooling.

A support-aware analytic decoder uses the supplied endpoints and known step-feature vectors.
The centre-only decoder uses their pre/post ordering and the known global boundaries to
bracket the event. A further anchored-tiling decoder reconstructs widths from all centres
before applying the same support-aware rule; on a true tiling it agrees with the interval
decoder. Applying that reconstruction to overlap or gaps tests a misspecified tiling assumption.
These diagnostics validate the crossed design and its observation model, not a trained
Aperture advantage; the trained real-video comparisons are reported in \secref{sec:gpu}.

\section{Native model evaluation protocol}
\label{sec:app_gpu}

The five primary contrasts in \tabref{tab:gpu_pending} were fixed in a hashed manifest before
execution, together with all cells, seeds, checkpoint-selection rules and result contracts; no
seed or hyperparameter was changed afterwards. The protocol separates five tasks: real EPIC
geometry/content crossings; 3B compression and sampling-rate transfer; 7B confirmation; native
temporal question answering and spatial token merging; and whole-model cost at the same
evaluated budgets. The records contain input and source hashes, checkpoint-selection criteria,
source membership, and per-item labels and predictions.
\tabref{tab:gpu_crossed_pending} reports the real-video crossings.

\paragraph{Real-video masking and geometry.}
Each method/seed/crossing/payload-protocol combination trains one balanced-mixture model
and evaluates its selected checkpoint in the four cells of that crossing. The mask retains
half the source slots and force-keeps those containing the true onset on the base tiling;
the same mask is shared across geometry cells. In contrast, the continuous-signal controls
of \secref{sec:app_crossed} do not protect the onset. The real-video result therefore
concerns removal of surrounding content with target-bearing evidence retained.
Source features are piecewise constant on $0.25$\,s frame bins. The gap operator shrinks
each interval to $0.15$\,s inside its original bin, so normalised physical re-pooling
returns the same payload. Its centre-only zero effect follows from this construction,
rather than demonstrating learned invariance to a new signal.

The reference cohorts have the following media coverage: \DenseRefMediaErrors{}
NTU RGB+D \citep{nturgbd} items of MVBench \citep{mvbench} were
unavailable to the dense 3B reference evaluation and were scored incorrect under the frozen
missing-media rule, depressing it to $\DenseRefAcc$; the full-media reference scores
$\DenseRefFullMediaAcc$, and
every adapted Qwen2.5-VL-7B comparison uses complete media with paired structure intact.
The PPE 3B reference at retention $0.45$ likewise contains \PpeRefMediaErrors{} media-error
items in a single task block, scored incorrect under the same rule: its reported \PpeRefAcc{}
corresponds to \PpeRefAvailAcc{} on the \PpeRefAvailItems{} available-media items. Every
systems record and the complete spatial panel were measured on one A800-SXM4-80GB device
class; all reported timing cohorts and the spatial accuracy cohort use this device class.
EchoingPixels' re-annotation mixture is not public, so its temporal-QA numbers use the
PPE-path common-data adaptation under a separately named protocol.

\begin{table}[htbp]
\centering
\small
\caption{\textbf{Real-video crossings.} Four-cell accuracy and paired geometry (G), mask (M) and interaction (I) effects with 95\% seed intervals over ten prespecified seeds and 2,048 fixed test items. Masks retain half the slots and protect onset-bearing evidence. TR/TM denote tiling with retained/masked content; GR/GM denote changed geometry. The compact display uses centre and stored-support readers; the full grid is exported separately.}
\label{tab:gpu_crossed_pending}
\begin{tabular}{@{}p{.23\linewidth}p{.38\linewidth}p{.32\linewidth}@{}}
\toprule
Crossing / payload & Cell accuracies & Paired effects \\
\midrule
\shortstack[l]{Overlap / metadata\\Centre} & \shortstack[l]{TR: 0.526 [0.502, 0.550]\\TM: 0.646 [0.628, 0.665]\\GR: 0.526 [0.502, 0.550]\\GM: 0.646 [0.628, 0.665]} & \shortstack[l]{G: 0.000 [0.000, 0.000]\\M: 0.121 [0.113, 0.128]\\I: 0.000 [0.000, 0.000]} \\
\shortstack[l]{Overlap / metadata\\Stored support} & \shortstack[l]{TR: 0.521 [0.507, 0.536]\\TM: 0.636 [0.622, 0.650]\\GR: 0.523 [0.508, 0.537]\\GM: 0.640 [0.623, 0.656]} & \shortstack[l]{G: 0.002 [0.000, 0.005]\\M: 0.116 [0.110, 0.122]\\I: 0.002 [-0.002, 0.005]} \\
\shortstack[l]{Gap / metadata\\Centre} & \shortstack[l]{TR: 0.527 [0.506, 0.548]\\TM: 0.651 [0.631, 0.670]\\GR: 0.527 [0.506, 0.548]\\GM: 0.651 [0.631, 0.670]} & \shortstack[l]{G: 0.000 [0.000, 0.000]\\M: 0.124 [0.117, 0.131]\\I: 0.000 [0.000, 0.000]} \\
\shortstack[l]{Gap / metadata\\Stored support} & \shortstack[l]{TR: 0.524 [0.507, 0.542]\\TM: 0.641 [0.626, 0.657]\\GR: 0.523 [0.505, 0.541]\\GM: 0.642 [0.625, 0.659]} & \shortstack[l]{G: 0.000 [-0.002, 0.001]\\M: 0.118 [0.112, 0.125]\\I: 0.002 [-0.001, 0.005]} \\
\shortstack[l]{Overlap / physical\\Centre} & \shortstack[l]{TR: 0.521 [0.506, 0.537]\\TM: 0.651 [0.635, 0.666]\\GR: 0.524 [0.510, 0.539]\\GM: 0.656 [0.642, 0.670]} & \shortstack[l]{G: 0.004 [0.001, 0.006]\\M: 0.130 [0.124, 0.136]\\I: 0.002 [-0.002, 0.006]} \\
\shortstack[l]{Overlap / physical\\Stored support} & \shortstack[l]{TR: 0.524 [0.514, 0.534]\\TM: 0.646 [0.630, 0.662]\\GR: 0.529 [0.519, 0.540]\\GM: 0.652 [0.636, 0.668]} & \shortstack[l]{G: 0.006 [0.003, 0.009]\\M: 0.122 [0.113, 0.131]\\I: 0.001 [-0.006, 0.007]} \\
\shortstack[l]{Gap / physical\\Centre} & \shortstack[l]{TR: 0.534 [0.516, 0.552]\\TM: 0.652 [0.636, 0.668]\\GR: 0.534 [0.516, 0.552]\\GM: 0.652 [0.636, 0.668]} & \shortstack[l]{G: 0.000 [0.000, 0.000]\\M: 0.118 [0.112, 0.123]\\I: 0.000 [0.000, 0.000]} \\
\shortstack[l]{Gap / physical\\Stored support} & \shortstack[l]{TR: 0.524 [0.507, 0.542]\\TM: 0.641 [0.626, 0.657]\\GR: 0.523 [0.505, 0.541]\\GM: 0.642 [0.625, 0.659]} & \shortstack[l]{G: 0.000 [-0.002, 0.001]\\M: 0.118 [0.112, 0.125]\\I: 0.002 [-0.001, 0.005]} \\
\bottomrule
\end{tabular}
\end{table}

\paragraph{Transfer cells and selection.}
Controlled adaptation trains at 8\,fps with random contiguous merging at retained fractions
1 and $1/2$. Evaluation covers 2, 4, 8 and 16\,fps and fractions
1, $1/2$, $1/4$, $1/8$, where source videos support those rates. The unseen policies are
adjacent similarity merging, noncontiguous similarity merging, age-dependent merging and
separate deletion stress tests. Rate transfer uses native frame extraction, not a renamed
subsample of cached embeddings. The primary policy cell uses noncontiguous merging at
$1/4$ retention and 8\,fps; the primary rate cell uses contiguous merging at $1/2$ retention
and 16\,fps. Every completed cell is reported; \tabref{tab:gpu_fullgrid} lists all
\GpuGridArmsGTwo{} 3B and \GpuGridArmsGThree{} 7B arms over the full ladder. The 7B arm
set was fixed before the 3B results were inspected, so the 7B grid is a prespecified
confirmation.

\begin{longtable}{@{}lcccc@{}}
\caption{\textbf{Complete G2/G3 adaptation grid.} Per-arm accuracy over the full prespecified
ladder (five policies $\times$ four retention fractions $\times$ four frame rates $= 80$ cells;
1{,}024 fixed test items per cell; ten seeds). \emph{Overall} averages all 800 cell-seed
accuracies; \emph{80-cell range} spans the per-cell seed means. The policy and rate columns are
the prespecified primary cells (\texttt{noncontiguous} keep $1/4$ @ 8\,fps;
\texttt{contiguous} keep $1/2$ @ 16\,fps) with 95\% seed $t$ intervals; for G3 only the policy
cell is a prespecified contrast. Twelve-way chance is 0.0833. The dagger marks the exact-Aperture seed-9 checkpoint
analysed with the full distributions in \figref{fig:gpu_null} and the accompanying text.}
\label{tab:gpu_fullgrid}\\
\toprule
Arm & Overall & 80-cell range & Policy cell & Rate cell \\
\midrule
\endfirsthead
\multicolumn{5}{l}{\emph{Table \ref{tab:gpu_fullgrid} (continued)}}\\
\toprule
Arm & Overall & 80-cell range & Policy cell & Rate cell \\
\midrule
\endhead
\endfoot
\bottomrule
\endlastfoot
\multicolumn{5}{@{}l}{\emph{G2: Qwen2.5-Omni-3B (13 arms)}}\\[2pt]
\texttt{centre} & 0.084 & 0.083--0.084 & 0.084 [0.081, 0.087] & 0.083 [0.080, 0.087] \\
\texttt{stored\_support} & 0.083 & 0.083--0.084 & 0.084 [0.079, 0.090] & 0.084 [0.078, 0.090] \\
\texttt{exact\_interval}$^{\dagger}$ & 0.088 & 0.087--0.088 & 0.088 [0.080, 0.096] & 0.088 [0.080, 0.096] \\
\texttt{learned\_interval} & 0.084 & 0.084--0.085 & 0.084 [0.082, 0.087] & 0.085 [0.083, 0.087] \\
\texttt{start\_end} & 0.083 & 0.083--0.084 & 0.083 [0.079, 0.088] & 0.083 [0.078, 0.088] \\
\texttt{terminal\_rms} & 0.086 & 0.085--0.087 & 0.085 [0.080, 0.091] & 0.087 [0.080, 0.093] \\
\texttt{width\_feature} & 0.085 & 0.085--0.085 & 0.085 [0.082, 0.088] & 0.085 [0.082, 0.088] \\
\texttt{quadrature} & 0.084 & 0.084--0.085 & 0.084 [0.080, 0.089] & 0.085 [0.081, 0.089] \\
\texttt{rote\_matched} & 0.084 & 0.081--0.089 & 0.087 [0.083, 0.091] & 0.085 [0.082, 0.088] \\
\texttt{ppe\_native} & 0.085 & 0.084--0.085 & 0.085 [0.079, 0.090] & 0.085 [0.079, 0.090] \\
\texttt{ppe\_ranked} & 0.085 & 0.085--0.085 & 0.085 [0.082, 0.088] & 0.085 [0.082, 0.088] \\
\texttt{sync\_rope\_native} & 0.086 & 0.085--0.086 & 0.085 [0.080, 0.089] & 0.086 [0.082, 0.090] \\
\texttt{fourier\_feature} & 0.087 & 0.086--0.088 & 0.087 [0.082, 0.091] & 0.088 [0.082, 0.094] \\
\midrule
\multicolumn{5}{@{}l}{\emph{G3: Qwen2.5-Omni-7B (8 arms)}}\\[2pt]
\texttt{centre} & 0.084 & 0.083--0.085 & 0.085 [0.081, 0.089] & 0.083 [0.079, 0.088] \\
\texttt{stored\_support} & 0.085 & 0.084--0.086 & 0.086 [0.082, 0.090] & 0.085 [0.082, 0.088] \\
\texttt{start\_end} & 0.086 & 0.085--0.087 & 0.086 [0.080, 0.091] & 0.086 [0.081, 0.091] \\
\texttt{terminal\_rms} & 0.087 & 0.086--0.088 & 0.088 [0.081, 0.094] & 0.087 [0.081, 0.094] \\
\texttt{width\_feature} & 0.084 & 0.083--0.084 & 0.084 [0.079, 0.089] & 0.084 [0.079, 0.089] \\
\texttt{ppe\_native} & 0.083 & 0.083--0.084 & 0.083 [0.079, 0.088] & 0.083 [0.079, 0.088] \\
\texttt{ppe\_ranked} & 0.087 & 0.086--0.088 & 0.086 [0.082, 0.090] & 0.087 [0.082, 0.092] \\
\texttt{sync\_rope\_native} & 0.084 & 0.083--0.085 & 0.085 [0.082, 0.088] & 0.084 [0.081, 0.087] \\
\end{longtable}

The primary 3B/7B controlled comparisons use ten fixed matched training seeds.
Training uses 3,000 steps, batch 8, AdamW with learning rate $10^{-4}$ and weight decay
$0.01$, 100 warmup steps, cosine decay, gradient clipping 1 and LoRA rank 32.
Validation at training conditions alone selects checkpoints every 500 steps.
A learned point-support readout token follows the query, audio and video sequence at time zero.
Its final hidden state is classified into twelve fixed one-second bins on $[2,14]$ seconds
within a 16-second window, preserving the label space across budgets. Unlike the scale study
of \secref{sec:app_scale}, this readout omits the content-similarity aid and uses ten seeds
over 1,024 items rather than 24 over 2,048, so the frozen-budget null and the scale study are
different measurements of the 3B/7B models. The existing token-index-to-centre-bin head is a separately
identified comparison. The native spatial comparison uses five adaptation seeds on the
complete MVBench \citep{mvbench} test split
and Qwen2.5-VL-7B's \citep{qwen25vl} original question/answer interface, with a $2{\times}2$
grid-merging
variant and Video-MME \citep{videomme} with subtitles as supplementary tasks and unadapted
3B dense/PPE
references for orientation.

\begin{figure}[t]
\centering
\includegraphics[width=\linewidth]{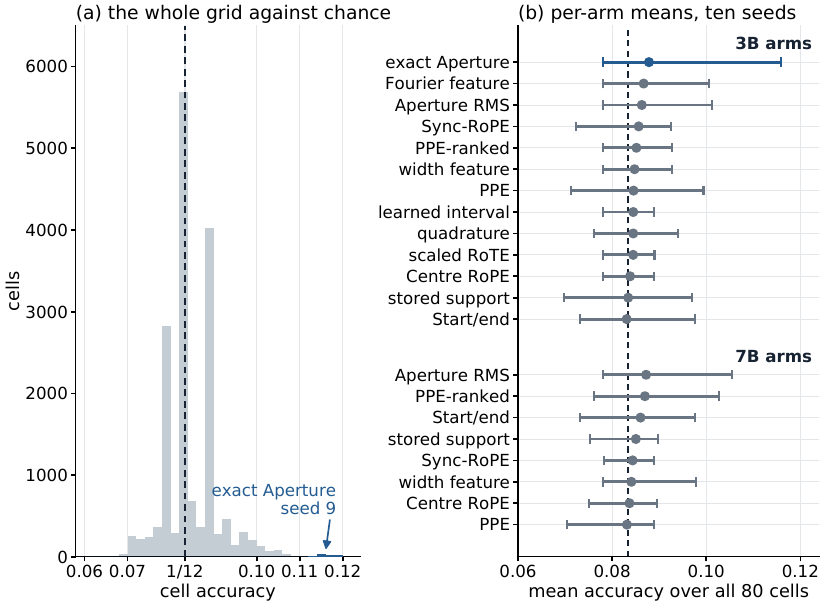}
\caption{\textbf{Onset adaptation remains near chance.}
\textbf{(a)} Accuracy of all 16,800 evaluation cells (13 3B arms and 8 7B arms, 10 seeds, 80
cells per checkpoint: five policies $\times$ four retained fractions $\times$ four frame rates)
against twelve-way chance (dashed line, $1/12$). The blue subset marks the exact-Aperture seed-9 checkpoint; the text reports its
above-chance accuracy and the unresolved prespecified contrasts.
\textbf{(b)} Per-arm means over all cells, with the range of the ten per-seed means as whiskers.
A training-batch probe at learning rate $10^{-3}$ memorises 16 items within 325 steps;
this verifies fitting that batch, not learning the full task.}
\label{fig:gpu_null}
\end{figure}

Across the full adaptation grid,
cell-seed accuracies span $\GpuGridAccMin$--$\GpuGridAccMax$ against twelve-way
chance, with $\GpuGridCellsAboveTen$ of $\GpuGridCellSeedAccs$ above $0.10$, concentrated
in a few checkpoints: one G3 \texttt{terminal\_rms} (Aperture RMS) checkpoint (seed 5) exceeds
$0.10$ in all $\GpuGridAnomalyCells$ of its cells (mean $\GpuGridConcTrmsMean$), and six
further checkpoints carry $\GpuGridConcRest$ more (the remaining $\GpuGridConcTail$
scattered across $\GpuGridConcTailCkpts$ others), most extremely one
\texttt{exact\_interval} (exact Aperture) checkpoint (seed 9) at $\GpuGridAnomalyLo$--$\GpuGridAnomalyHi$
in every cell (worst-cell binomial $p=\GpuGridAnomalyPTwoSided$, two-sided); the
prespecified contrasts remain null
(\tabref{tab:gpu_fullgrid}). 

\paragraph{Support construction and comparison names.}
The 3B/7B onset and native spatial implementations retain source membership and construct
Fourier moments from it. They evaluate the support readout; the compact incremental $(\rho,Z)$ update is
validated separately by the general-support implementation in \secref{sec:app_cpu_state}.
The onset paths average equal-mass uniform exposure atoms of width $1/f_{\mathrm{tok}}$,
where the visual tower produces $f_{\mathrm{tok}}=\mathrm{fps}/2$ tokens per second
(8\,fps gives 4 tokens/s and 0.25\,s atoms). The spatial path averages point phases at retained
member coordinates. Their measured buffers and latency describe these membership-based paths.
The small EPIC crossing study instead constructs a uniform-box support directly from each
token's centre and width.

The onset arm \texttt{exact\_interval} (exact Aperture) instead uses the uniform interval
between the extreme member endpoints; for noncontiguous membership this is a different
measure. \texttt{terminal\_rms} applies the Aperture RMS readout to the full support moments.
Centre-only rounds temporal position IDs, whereas these support readouts use continuous
phases in the temporal channels. Their comparisons therefore test complete positional
interfaces, including coordinate quantisation. The assisted-scale study uses the common
rounded centres specified in \secref{sec:app_scale}.

Feature arms supply scalar width, endpoints, a learned interval embedding or the same
Fourier moments to the content pathway. Sample-average RoPE and exposure quadrature provide
numerical references. The common-trunk RoTE adaptation (saved as \texttt{rote\_matched})
uses the model's positions-per-second value $P$ twice: the supplied bank is
$\omega_m=P\,\mathrm{inv\_freq}_m$, and the interval operator also scales time by $P$.
For interval centre $c$ and width $w$, the executed interval code is
\begin{equation}
 s_m=\operatorname{sinc}(P\omega_m w/2),\qquad
 z_m=e^{iP\omega_m c}\frac{s_m}{M_t^{-1}\sum_{\ell=1}^{M_t}s_\ell},
 \label{eq:gpu_scaled_rote}
\end{equation}
where $M_t$ counts temporal pairs. Its effective interval bank is $P\omega_m$, approximately
$P^2\mathrm{inv\_freq}_m$ up to dtype rounding. We call this arm \emph{scaled RoTE}; it does
not share the support arm's effective physical bank or reproduce TIE's default time scale.
The self-attention patch applies this interval code to both video queries and keys;
learned query/readout tokens retain point codes. CPU reference-operator checks validate
the primitive separately from this GPU parameterisation. PPE retains member coordinates and channel-group assignment;
its released selector and score-ranked variant are \texttt{ppe\_native} and
\texttt{ppe\_ranked}. \texttt{sync\_rope\_native} implements Sync-RoPE's temporal/spatial
allocation. Operator formulas and the distinct MN-sinc and width-capped controls are in
\secref{sec:app_related}.

\paragraph{Predictive and systems outcomes.}
All arms retain item IDs and the same physical labels. Paired-seed $t$ intervals describe model
variation on the fixed test population. A separate crossed bootstrap resamples model seeds and
participant--video--item clusters with shared draws across arms (\GpuBootReps{} resamples,
rng seed \GpuBootSeed); video is the top cluster when participant metadata is unavailable
(\GpuBootGFourVideos{} videos and \GpuBootGFourItems{} items for the spatial contrast).
The EPIC test cohort contains \EpicTestParticipants{} participants
(\EpicTestParticipantList) and \EpicTestVideos{} videos. Each G1 crossing uses
\GpuInvGOneItems{} test items, whereas each G2/G3 contrast uses \GpuInvGTwoGThreeItems{}.
The crossed intervals therefore offer only a limited sensitivity summary for participant
variation; \tabref{tab:gpu_pending} reports them beside the seed intervals.
A descriptive video-level sensitivity analysis, separate from the prespecified family, gives
$[\GpuBootVidTopGTwoPolicyLo,\GpuBootVidTopGTwoPolicyHi]$,
$[\GpuBootVidTopGTwoRateLo,\GpuBootVidTopGTwoRateHi]$,
$[\GpuBootVidTopGTwoPlacementLo,\GpuBootVidTopGTwoPlacementHi]$ and
$[\GpuBootVidTopGThreePolicyLo,\GpuBootVidTopGThreePolicyHi]$ for the four EPIC contrasts;
all four still contain zero. The five primary contrasts form a separate
Holm family from the $\StatNConfirm$-member family of \secref{sec:app_stats}. Missing cells and degenerate estimates remain explicit.

\paragraph{Exploratory budget probe.}
An exploratory learning-rate analysis uses the 3B grid protocol at learning
rate $\ExplBudgetLr$, the rate at which the training-batch probe memorises a 16-item train
batch within 325 steps, on \ExplBudgetArms{} arms (centre-only, stored support, exact
Aperture, learned interval and PPE-ranked) over the same \ExplBudgetSeeds{} seeds, cells
and frozen test inventory (sha256 \texttt{\ExplBudgetTestInvSha}\ldots). This exploratory comparison forms a separate analysis from the prespecified contrasts.
\tabref{tab:exploratory_budget} gives the per-arm endpoints. The mean final training
loss lies within $\ExplBudgetMaxLossDev$ of $\ln 12=\ExplBudgetLnTwelve$ for every arm,
\ExplBudgetCells{}-cell test means span $\ExplBudgetGridMeanLo$--$\ExplBudgetGridMeanHi$
($\ExplBudgetChanceFactorLo$--$\ExplBudgetChanceFactorHi$ of chance), and all
\ExplBudgetContrasts{} paired seed-level contrasts against centre-only on the two
prespecified cells contain zero (means $\ExplBudgetContrastMeanLo$ to
$\ExplBudgetContrastMeanHi$; intervals within $[\ExplBudgetContrastCILo,
\ExplBudgetContrastCIHi]$).

\begin{table}[tb]
\centering
\caption{\textbf{Exploratory budget probe: per-arm endpoints.} Means over
\ExplBudgetSeeds{} seeds at learning rate $\ExplBudgetLr$; the protocol is otherwise
identical to the frozen 3B grid; this analysis is outside the prespecified family. Final
train loss is the per-seed last logged value, averaged; $\ln 12=\ExplBudgetLnTwelve$ is
the twelve-way-chance loss. The test mean is over all \ExplBudgetCells{} cells of the
frozen test inventory.}
\label{tab:exploratory_budget}
\begin{tabular}{lccc}
\toprule
arm & final train loss & test mean & $\times$ chance \\
\midrule
centre-only (\texttt{centre}) & $\ExplBudgetTrainLossCentre$ & $\ExplBudgetGridMeanCentre$ & $\ExplBudgetFactorCentre$ \\
stored support (\texttt{stored\_support}) & $\ExplBudgetTrainLossStoredSupport$ & $\ExplBudgetGridMeanStoredSupport$ & $\ExplBudgetFactorStoredSupport$ \\
exact Aperture (\texttt{exact\_interval}) & $\ExplBudgetTrainLossExactInterval$ & $\ExplBudgetGridMeanExactInterval$ & $\ExplBudgetFactorExactInterval$ \\
learned interval (\texttt{learned\_interval}) & $\ExplBudgetTrainLossLearnedInterval$ & $\ExplBudgetGridMeanLearnedInterval$ & $\ExplBudgetFactorLearnedInterval$ \\
PPE-ranked (\texttt{ppe\_ranked}) & $\ExplBudgetTrainLossPpeRanked$ & $\ExplBudgetGridMeanPpeRanked$ & $\ExplBudgetFactorPpeRanked$ \\
\bottomrule
\end{tabular}
\end{table}

Classification latency includes media decoding, feature extraction, compression and forward
inference. Generative timing begins after media preprocessing and covers model feature
extraction, compression, prefill and autoregressive decoding. After 20 warmups, at least
100 synchronised examples provide median and p90 latency,
throughput, peak allocated/reserved GPU memory and actual KV bytes. Time to first token (TTFT) applies to generative
interfaces. Within each cohort, hardware, precision, batch, attention backend,
cache settings and output-length caps are matched; realised output lengths vary by arm. Support-state overhead is measured against the
entire model cache. Accuracy, latency and memory points are recorded together, including dominated
points. These measurements include response-length effects and do not isolate rotary-operation overhead.

Charged support-state bytes are the measured implementation's retained positional state per
request. At the classification points, the centre-only arm is charged the bound of the model's
three-row position-ID table ($48$\,KiB at 2,048 tokens) and stored support its membership
indices and complex moments ($4.5$--$17$\,KiB across the measured cells). On the generative
panel both arms are charged the retained member-coordinate table ($0.527$\,MiB), against
$4.237$\,MiB for PPE's expanded $3K$-row membership; the exact moments are not separately
metered there, so the marginal charged state of stored support over centre is zero.

\begin{table}[htbp]
\centering
\small
\caption{\textbf{Classification request cost on EPIC onset.} Descriptive pooled timings
for ten seeds and 1,280 requests per arm in each cell; observed peak allocated memory.
All models are Qwen2.5-Omni on A800 hardware. Charged state is a 48\,KiB position-ID
bound for centre-only and retained membership/moment buffers for stored support.
Medians and p90 are seconds, peak is GiB, and state is KiB.}
\label{tab:classification_cost}
\begin{tabular}{@{}llccc@{}}
\toprule
Model / retention / rate & Readout & Median / p90 & Peak & State \\
\midrule
3B / 0.5 / 16 fps & Centre-only & 6.61 / 6.87 & 24.08 & 48 \\
3B / 0.5 / 16 fps & Stored support & 6.60 / 6.88 & 24.08 & 17 \\
3B / 0.25 / 8 fps & Centre-only & 5.19 / 5.46 & 22.48 & 48 \\
3B / 0.25 / 8 fps & Stored support & 5.20 / 5.44 & 22.48 & 4.5 \\
7B / 0.25 / 8 fps & Centre-only & 5.19 / 5.42 & 34.71 & 48 \\
7B / 0.25 / 8 fps & Stored support & 5.17 / 5.42 & 34.71 & 4.5 \\
\bottomrule
\end{tabular}
\end{table}

\begin{table}[htbp]
\centering
\small
\caption{\textbf{Generative execution cost.} The spatial panel compares Centre, Stored support and native PPE at one-quarter retention. Timing begins after media preprocessing and covers model execution through autoregressive decoding. Accuracy uses the first 256 items of the fixed evaluation order, paired across arms. Latency is median/p90 in milliseconds, throughput is requests/s, GPU is observed peak allocated GiB, and KV/state use MiB. Raw traces also retain time to first token (TTFT), generated-token throughput and reserved memory. Timing and memory summaries are descriptive.}
\label{tab:gpu_systems_pending}
\begin{tabular}{@{}lccccc@{}}
\toprule
Method & Accuracy & Med./p90 & Req./s & GPU & KV/state \\
\midrule
Centre & 0.676 & 930.8/4144.2 & 0.46 & 25.51 & 320.41/0.527 \\
Stored support & 0.692 & 1333.8/8622.2 & 0.20 & 25.51 & 320.41/0.527 \\
PPE & 0.777 & 984.1/3194.9 & 0.30 & 25.51 & 320.41/4.237 \\
\bottomrule
\end{tabular}
\end{table}

\begin{table}[htbp]
\centering
\small
\caption{\textbf{Generated lengths in the MVBench timing subset.} Descriptive summaries
of five seeds on the same 256 items (1,280 requests per arm), with a 128-token cap.
Repeated item/seed requests are not treated as independent inferential replicates.}
\label{tab:output_lengths}
\begin{tabular}{@{}lccc@{}}
\toprule
Readout & Mean tokens & Median tokens & Requests at cap \\
\midrule
Centre-only & 4.36 & 2 & 8 \\
Stored support & 12.86 & 3 & 79 \\
PPE & 8.95 & 3 & 56 \\
\bottomrule
\end{tabular}
\end{table}

\tabref{tab:classification_cost} reports classification request times and memory.
For generation, \tabref{tab:output_lengths} shows why a common cap does not imply equal
work: stored support produces longer responses on average and reaches the cap more often
than centre-only. The latency differences in \tabref{tab:gpu_systems_pending} include this
behaviour and the membership-based readout computation.

\section{Experimental protocols and additional results}
\label{sec:app_experiments}

This appendix specifies the datasets, models, interventions and statistical estimators used in the
experiments. Tables include the full evaluation grids, with sample sizes stated separately where
arms have different seed budgets; they shorten exact Aperture to `Aperture' where its row pairs
with Aperture RMS. All reported model accuracies use held-out test data unless
explicitly identified as a diagnostic on a different split. Numerical results and protocol
constants are generated from the saved run records.

\subsection{Statistical analysis}
\label{sec:app_stats}

\paragraph{Replicates and uncertainty.}
A replicate is a trained model identified by its training condition and random seed.
Within-model width sensitivity compares the same checkpoint and test set with and without
width permutation. Cross-arm comparisons use matched seeds where available; cells obtained
by evaluating one checkpoint under several policies are averaged within checkpoint before
pooling. For transfer across policies, repeated policy evaluations with the same seed form
one seed group. Standard deviations describe seed variability, not uncertainty in the mean.
Unless stated otherwise, confidence intervals are $95\%$ and tests are two-sided.
The summaries for the studies in this appendix do not contain item-level outcomes needed
to estimate uncertainty over the test population, so their intervals are conditional on the
fixed test sets. The native evaluation separately retains item-level predictions and uses
the crossed bootstrap of \secref{sec:app_gpu}.

\paragraph{Four-policy contrasts.}
For the four conditions $C,D,M,E$ defined in \secref{sec:factorial}, we estimate the
two descriptive contrasts $L_1=(M+E-C-D)/2$ and $L_2=(D+E-C-M)/2$.
Each cell contributes equally. Their difference is $D-M$. Since overlap and gaps are
different geometry changes, $L_1,L_2$ are not causal main effects of a fully crossed design. The main analysis uses cell-wise seed variances with
Welch--Satterthwaite degrees of freedom. The $D-M$ comparison additionally uses
$\GeomNBoot$ percentile-bootstrap resamples of the trained runs within each of its two
conditions. Its interval is $[\GeomEffInversionBootLo,\GeomEffInversionBootHi]$ for
Aperture RMS and $[\GeomPlainInversionBootLo,\GeomPlainInversionBootHi]$ for
exact Aperture. Hedges' $g$ gives a standardised two-sample effect size.
These analyses estimate variability across model fits; neither test-set size nor repeated
evaluation of a checkpoint increases the number of training replicates.
\tabref{tab:factorial_intervals} collects the four-policy contrasts. As a sensitivity analysis, pairing the shared seed identities gives $D-M$ intervals of
$[\ManuscriptPlainPairedLo,\ManuscriptPlainPairedHi]$ for exact Aperture and
$[\ManuscriptRmsPairedLo,\ManuscriptRmsPairedHi]$ for Aperture RMS. Both retain the ordering.

\paragraph{Multiplicity and interpretation.}
The analysis retains a declared family of $\StatNConfirm$ directional effect comparisons,
including comparisons that are not emphasised in the main text. Holm correction is applied
at $\alpha=\StatAlpha$ to this full family. The family includes every declared effect contrast. The family is an analysis declaration, not a study-wide preregistration.
Three transfer hypotheses occupy noninferential slots with conservative $p=1$ in this
family. Transfer uncertainty uses five shared seed blocks and is descriptive
(\secref{sec:app_trans}); checkpoint-level tests do not enter the inferential analysis.
Unadjusted $p$-values in detailed tables are descriptive unless their family-wise status is
stated. In particular, the geometry/content ordering survives correction
($p\StatHolmInversion$), whereas the 3B placement swing does not
(adjusted $p=\StatHolmDidAdj$).
A nonsignificant difference is reported as unresolved; it is not evidence of equality.
Where an equivalence threshold is computed from the observed interval, it is a descriptive
bound, not a test against an independently chosen practical margin.

\begin{table}[htbp]
\centering
\setlength{\tabcolsep}{4pt}
\caption{\textbf{Uncertainty in four-policy contrasts.} The $L_1$ and $L_2$ intervals use cell-wise Welch--Satterthwaite variances; $D-M$ intervals use 20,000 independent seed resamples within each condition. All contrasts are in accuracy units, with 15 runs per cell.}
\label{tab:factorial_intervals}
\small
\begin{tabular}{@{}llcc@{}}
\toprule
Code & Contrast & Estimate & 95\% interval \\
\midrule
Aperture & $L_1$ & $-0.0677$ & $[-0.0808,-0.0546]$ \\
Aperture & $L_2$ & $-0.0077$ & $[-0.0208,+0.0054]$ \\
Aperture & $D-M$ & $+0.0600$ & $[+0.0391,+0.0798]$ \\
Aperture RMS & $L_1$ & $-0.0939$ & $[-0.1069,-0.0808]$ \\
Aperture RMS & $L_2$ & $-0.0072$ & $[-0.0203,+0.0058]$ \\
Aperture RMS & $D-M$ & $+0.0867$ & $[+0.0663,+0.1070]$ \\
\bottomrule
\end{tabular}
\end{table}

\begin{figure}[htbp]
\centering
\includegraphics[width=\linewidth]{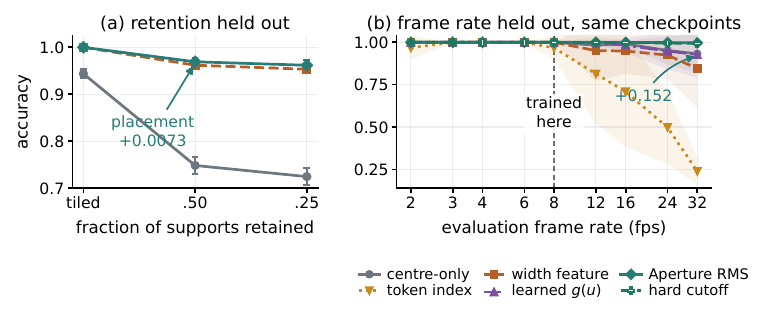}
\caption{\textbf{Sampling-rate and retention transfer.} Test accuracy of the same checkpoints
under token eviction and changes in video frame rate. Points show seed means and bars one
seed standard deviation; every arm has $n=15$ seeds except token index ($n=5$). Arm colours
match \figref{fig:epic}; token index denotes the per-modality index arm. Exact values and all arms are in \tabref{tab:v7} and \tabref{tab:v7full}.}
\label{fig:rate}
\end{figure}

\subsection{Synthetic sampling-rate transfer}
\label{sec:app_rate}

\paragraph{Continuous signals and tokenisation.}
Each $\VseDuration$\,s stream contains one video event and one audio event, with type vectors
in $\mathbb{R}^{\VseDsig}$ drawn from a fixed set of $\VseNtypes$ types.
Events have Gaussian width $\sigma=\VseSigma$\,s and signed lag in
$[-\VseLagMax,\VseLagMax]$\,s. The target is one of $\VseNbins$ lag bins of width
$\VseBinMs$\,ms. Each token is the exact mean of the continuous signal over its support,
computed from the Gaussian CDF. Audio is sampled at $\VseFpsA$\,fps; video is sampled at
$\VseTrainRate$\,fps in training and $\VseEvalFpsLo$--$\VseEvalFpsHi$\,fps in evaluation.
There are no distractor events or additive noise in this configuration.

\paragraph{Signal-level diagnostics.}
An exhaustive decoder estimates the audio event time and compares all candidate lag bins with
the retained video features. On $\CeilNcells$ cells with $\CeilNval$ streams each, its
accuracy is $\CeilRateLo$ at every evaluated rate and at least $\CeilEvictLo$ across the
retention cells. The forward model reproduces the generated features at the true event time
with relative residual $\CeilSelfCheck$. These checks show that rate transfer is not primarily
limited by loss of label identifiability in this noiseless generator. They do not establish
equal numerical difficulty: the smallest wrong-bin residual changes by
$\CeilRateMarginRatio\times$ across the rate ladder. Peak token norm varies by only
$\VseFlatness\times$, limiting input-magnitude change as an explanation of the observed
accuracy differences. The decoder is a diagnostic on generated signals, not a trained baseline.

\paragraph{Reader and training.}
The causal, full-prefix reader has $\VseNlayer$ layers, $\VseDmodel$ channels,
$\VseNhead$ attention heads and $\VseNparams$ parameters. Rotary periods range from
$\VsePmin$ to $\VsePmax$\,s. Training uses $\VseSteps$ steps, batch size $\VseBatch$,
AdamW with learning rate $\VseLr$, weight decay $\VseWd$ and $\VseWarmup$ warmup steps;
evaluation occurs every $\VseEvalEvery$ steps on $\VseNval$ streams.
All runs include random contiguous merges and therefore heterogeneous widths.
One regime trains on tilings and the other with retention fraction $\VseKeep$.
Each checkpoint is evaluated on a tiling, retention fractions $\VseKeep$ and
$\VseKeepLo$, and the rate ladder within its training regime.

\paragraph{Arm definitions and sample sizes.}
The centre-only clock uses physical time; the index arm uses per-modality token counts.
The feature arm appends scalar width to the content pathway. The hard-cutoff arm retains
channels satisfying $\theta_mw\le\pi$, and the learned-gain arm replaces the analytic
response to $u=\theta_mw/2$ with an MLP. These alternatives differ in both their functional
constraints and, for the learned arms, their parameterisation.
The grid contains $\VseNruns$ runs. Seed counts vary by arm and training regime and are
reported with every table block; paired contrasts use the intersection of available seeds.

\paragraph{Additional results.}
The tiling-trained, tiling-evaluated cell is near ceiling
($\VseSatLo$--$\VseSatHi$), limiting its ability to distinguish codes. The same evaluation
cell reached from eviction training spans $\VseEvTiledLo$--$\VseEvTiledHi$ and measures
transfer between geometries. At $\VseFarRate$\,fps under tiling training, the Aperture RMS,
feature-width and hard-cutoff arms obtain $\VseTiApertureRmsFar$,
$\VseTiWallclockWfeatFar$ and $\VseTiNyquistFar$, respectively; the paired
kernel-minus-feature contrast is $\VseMechTi$. Under eviction training the index arm
reaches $\VseEvIndexFar$, compared with $\VseIndexEvHalf$ under half retention at the
training rate. These results distinguish the need for a physical clock from the additional
choice of width representation.

\begin{table}[htbp]
\centering
\setlength{\tabcolsep}{4pt}
\caption{\textbf{Synthetic sampling-rate transfer.} Mean test accuracy (seed s.d.) from the same eviction-trained checkpoints. The training rate is 8 fps; 32 fps is held out. The last column reports the lowest seed accuracy. Chance is 0.125.}
\label{tab:v7}
\small
\begin{tabular}{@{}lrccc@{}}
\toprule
Code & $n$ & 8 fps & 32 fps & Minimum at 32 fps \\
\midrule
Token index & 5 & $0.967\;(0.068)$ & $0.238\;(0.077)$ & 0.174 \\
Interleaved index & 5 & $1.000\;(0.000)$ & $0.292\;(0.079)$ & 0.195 \\
Centre only & 15 & $1.000\;(0.002)$ & $0.932\;(0.117)$ & 0.677 \\
Width feature & 15 & $1.000\;(0.001)$ & $0.845\;(0.228)$ & 0.333 \\
Aperture & 15 & $1.000\;(0.002)$ & $0.951\;(0.092)$ & 0.662 \\
Aperture RMS & 15 & $1.000\;(0.000)$ & $0.997\;(0.006)$ & 0.981 \\
Hard cutoff & 15 & $1.000\;(0.001)$ & $0.994\;(0.009)$ & 0.969 \\
Hard cutoff RMS & 5 & $1.000\;(0.000)$ & $0.954\;(0.101)$ & 0.774 \\
Learned gain & 15 & $0.999\;(0.004)$ & $0.928\;(0.137)$ & 0.479 \\
Learned gain RMS & 5 & $1.000\;(0.000)$ & $0.989\;(0.014)$ & 0.966 \\
Absolute gain RMS & 5 & $1.000\;(0.000)$ & $0.999\;(0.003)$ & 0.994 \\
First lobe RMS & 5 & $1.000\;(0.000)$ & $1.000\;(0.000)$ & 1.000 \\
\bottomrule
\end{tabular}
\end{table}

\begin{table}[htbp]
\centering
\setlength{\tabcolsep}{4pt}
\caption{\textbf{Synthetic retention evaluation.} Mean test accuracy (seed s.d.) at 8 fps, grouped by training regime. Each row reports its own seed count. Training on tilings produces near-ceiling accuracy on the matching evaluation, which limits the sensitivity of that comparison.}
\label{tab:v7full}
\small
\begin{tabular}{@{}lrccc@{}}
\toprule
\multicolumn{5}{@{}l}{\textbf{Trained on tilings}} \\
Code & $n$ & Retain 1.00 & Retain 0.50 & Retain 0.25 \\
\midrule
Token index & 5 & $0.999\;(0.002)$ & $0.370\;(0.048)$ & $0.354\;(0.048)$ \\
Interleaved index & 5 & $0.999\;(0.002)$ & $0.347\;(0.065)$ & $0.339\;(0.068)$ \\
Centre only & 5 & $1.000\;(0.000)$ & $0.442\;(0.067)$ & $0.423\;(0.061)$ \\
Width feature & 5 & $1.000\;(0.000)$ & $0.442\;(0.062)$ & $0.422\;(0.064)$ \\
Aperture & 5 & $1.000\;(0.000)$ & $0.431\;(0.069)$ & $0.411\;(0.066)$ \\
Aperture RMS & 5 & $1.000\;(0.000)$ & $0.453\;(0.121)$ & $0.422\;(0.116)$ \\
Hard cutoff & 5 & $1.000\;(0.000)$ & $0.458\;(0.038)$ & $0.441\;(0.033)$ \\
Hard cutoff RMS & 5 & $1.000\;(0.000)$ & $0.477\;(0.090)$ & $0.466\;(0.088)$ \\
Learned gain & 5 & $1.000\;(0.000)$ & $0.540\;(0.079)$ & $0.517\;(0.074)$ \\
Learned gain RMS & 5 & $1.000\;(0.000)$ & $0.432\;(0.070)$ & $0.411\;(0.060)$ \\
Absolute gain RMS & 5 & $1.000\;(0.000)$ & $0.394\;(0.088)$ & $0.367\;(0.077)$ \\
First lobe RMS & 5 & $1.000\;(0.000)$ & $0.424\;(0.033)$ & $0.403\;(0.036)$ \\
\midrule
\multicolumn{5}{@{}l}{\textbf{Trained with eviction}} \\
Code & $n$ & Retain 1.00 & Retain 0.50 & Retain 0.25 \\
\midrule
Token index & 5 & $0.945\;(0.084)$ & $0.755\;(0.025)$ & $0.736\;(0.020)$ \\
Interleaved index & 5 & $0.935\;(0.017)$ & $0.756\;(0.019)$ & $0.742\;(0.017)$ \\
Centre only & 15 & $0.944\;(0.010)$ & $0.748\;(0.018)$ & $0.724\;(0.018)$ \\
Width feature & 15 & $1.000\;(0.000)$ & $0.962\;(0.008)$ & $0.953\;(0.009)$ \\
Aperture & 15 & $0.999\;(0.002)$ & $0.970\;(0.018)$ & $0.961\;(0.022)$ \\
Aperture RMS & 15 & $1.000\;(0.000)$ & $0.969\;(0.009)$ & $0.962\;(0.011)$ \\
Hard cutoff & 15 & $0.999\;(0.001)$ & $0.952\;(0.012)$ & $0.943\;(0.012)$ \\
Hard cutoff RMS & 5 & $0.999\;(0.001)$ & $0.951\;(0.017)$ & $0.942\;(0.023)$ \\
Learned gain & 15 & $1.000\;(0.000)$ & $0.968\;(0.009)$ & $0.960\;(0.010)$ \\
Learned gain RMS & 5 & $1.000\;(0.000)$ & $0.978\;(0.002)$ & $0.969\;(0.006)$ \\
Absolute gain RMS & 5 & $1.000\;(0.001)$ & $0.967\;(0.004)$ & $0.958\;(0.009)$ \\
First lobe RMS & 5 & $1.000\;(0.001)$ & $0.967\;(0.012)$ & $0.955\;(0.014)$ \\
\bottomrule
\end{tabular}
\end{table}

\subsection{Training coverage of the gain argument}
\label{sec:app_coverage}

This experiment compares single-rate training with augmentation over
\CovTrainRates\,fps under half retention. The augmented condition covers approximately
$86\%$ of the evaluated range in $\log u$, where $u=\theta w/2$.
All runs complete the training schedule with loss below $0.1$.
At $\CovFarRate$\,fps and $\CovNSeeds$ seeds per cell, the smallest paired
Wilcoxon signed-rank comparison across the six arm pairs gives $p\CovWorstP$.
The closed-form versus learned-gain comparison gives $p\CovLearnSingleP$ with single-rate
training and $p\CovLearnCoveredP$ with augmentation.
Seed distributions include both near-ceiling and low-accuracy outcomes: in the augmented
condition, exact Aperture spans $[\CovApertureFarLo,\CovApertureFarHi]$ and the learned gain
$[\CovLearnFarLo,\CovLearnFarHi]$. Medians and ranges at representative rates appear in
\tabref{tab:coverage}. The grid does not establish an ordering among the tested gain functions.

\begin{table}[htbp]
\centering
\setlength{\tabcolsep}{4pt}
\caption{\textbf{Training coverage and rate transfer.} Median test accuracy [minimum, maximum] across seeds. Single-rate training uses 8 fps; covered training samples 2, 4, 8, 16 and 32 fps. Both 64 and 128 fps are outside these training sets. The substantial ranges motivate the paired rank comparisons in the text.}
\label{tab:coverage}
\small
\begin{tabular}{@{}llrccc@{}}
\toprule
Training & Code & $n$ & 32 fps & 64 fps & 128 fps \\
\midrule
Single & Aperture & 15 & $0.979\;[0.147,1.000]$ & $0.885\;[0.152,1.000]$ & $0.750\;[0.131,1.000]$ \\
Single & Learned gain & 15 & $0.971\;[0.290,1.000]$ & $0.762\;[0.108,1.000]$ & $0.574\;[0.201,1.000]$ \\
Single & Width feature & 15 & $0.905\;[0.213,1.000]$ & $0.824\;[0.134,0.991]$ & $0.544\;[0.089,0.981]$ \\
\midrule
Covered & Aperture & 15 & $1.000\;[0.999,1.000]$ & $0.994\;[0.465,1.000]$ & $0.938\;[0.403,1.000]$ \\
Covered & Learned gain & 15 & $1.000\;[0.981,1.000]$ & $0.999\;[0.553,1.000]$ & $0.988\;[0.184,1.000]$ \\
Covered & Width feature & 15 & $1.000\;[0.997,1.000]$ & $1.000\;[0.967,1.000]$ & $0.909\;[0.641,1.000]$ \\
\bottomrule
\end{tabular}
\end{table}

\subsection{Large-model temporal localisation}
\label{sec:app_scale}

\paragraph{Models, data and readout.}
Qwen2.5-Omni-3B and -7B use their respective frozen audiovisual towers with LoRA adaptation
of the thinker. The trainable counts are $\ScaleTrainable$ and $\SevScaleTrainable$ million
parameters. Both scales use the participant-disjoint EPIC onset protocol, with
$\ScaleNBins$ bins and chance accuracy $\ScaleChance$.
The primary comparisons use $\ScaleNPaired$ matched seeds,
$\ScaleSeedList$, at each scale. The width-feature arm adds $\ScaleWfeatExtra$ parameters
($\ScaleWfeatExtraPct$ of the trainable count) through a zero-initialised projection.
All arms receive the same content-similarity term in the readout logits. It helps identify
the relevant content without supplying token width. These experiments therefore evaluate
positional encoding within an assisted localisation reader.
All models train for 3,000 steps with AdamW, batch size 8, learning rate $10^{-4}$,
weight decay 0.01, 100 warmup steps, cosine decay, gradient clipping at norm 1,
and LoRA rank 32. Validation uses 384 items every 500 steps; the checkpoint with
the highest mean bin accuracy across trained budgets is evaluated on 2,048 test items.
The readout receives the query, audio tokens and video tokens in that order under a causal
mask, so this evaluation does not impose chronological audiovisual streaming.

\paragraph{Budgets and evaluation metrics.}
A budget $b$ merges $b$ consecutive video tokens. Training samples
$b\in\{\ScaleTrained\}$, while $b=\ScaleFarBudget$ is the primary held-out budget.
At four video tokens per second and one-second label bins, its video-token width equals
one bin. Bin accuracy and token-localisation accuracy are consequently nearly identical
there: their measured ratio is $\ScaleBinTokRatioMFour$.
At $b=\ScaleSatBudget$, a video token spans multiple bins, so correctly identifying a token
alone does not specify the target bin. We report this wider-budget evaluation descriptively
and use token accuracy for the supplementary budget comparison. Token width by itself does
not prove an information-theoretic ceiling: aggregated content can also encode within-token
timing. The temporal-scramble check, measured at the coarsest trained budget, incurs
$\ScrambleCost$ accuracy loss; it is not a separate validation of every wider budget.

\paragraph{Baseline conventions.}
All arms use the same rounded temporal position IDs and rotary frequency bank.
MN-sinc's duration normaliser is computed over the temporal rotary pairs to which its gain is
applied; non-temporal pairs retain unit gain. The width-capped phase comparison uses the same rounded
centres as the other arms and a width-dependent temporal frequency cap.
The 7B model shares the 3B model's temporal rotary allocation, head dimension and position-ID
scale, so these positional transformations require no scale-specific retuning.

\paragraph{Primary comparisons.}
At $b=\ScaleFarBudget$, exact Aperture exceeds MN-sinc by $\ScaleVsRote$ at 3B and
$\SevScaleVsRote$ at 7B. Against the centre-only clock the differences are
$\ScaleVsBlind$ ($p\ScaleVsBlindP$) and $\SevScaleVsBlind$ ($p\SevScaleVsBlindP$).
These latter comparisons do not establish superiority. At 3B the change in the
Aperture-minus-feature contrast from the mean trained-budget anchor to the held-out budget is
$\ScaleDidWfeatSel$; its adjusted $p=\StatHolmDidAdj$ does not support a family-wise
significance claim. At 7B, the trained and held-out feature contrasts have the same sign
($\SevScaleShapeWfeatIn$ and $\SevScaleShapeWfeatFar$), so the 3B sign pattern does not
reproduce. The width-capped phase comparisons are also unresolved: $\ScaleSyncInDist$ in distribution
and $\ScaleSyncFar$ at the held-out budget, with $\ScaleSyncN$ matched seeds.

\paragraph{Token and bin metrics.}
The model predicts a video-token index. Bin accuracy is obtained by mapping the centre of
that predicted token into the fixed label grid; the reader has no separate within-token
regression head. At the primary held-out budget, Aperture's token contrast against MN-sinc
is $\ScaleTokRote$, closely matching its bin contrast.
At finer budgets, several token centres can map to one bin, whereas at wider budgets one
token spans multiple bins. Ratios of the aggregate bin and token accuracies are consequently
not conditional within-token accuracy and are not used to infer interpolation ability.

\begin{table}[htbp]
\centering
\setlength{\tabcolsep}{4pt}
\caption{\textbf{3B budget evaluation.} Bin accuracy, mean (seed s.d.); training budgets are $b=1,2$, the primary held-out budget is $b=4$, and $b=8$ is supplementary. MN-sinc and width-capped phase are controls defined in Appendix~\ref{sec:app_related}. Contrasts use matched seeds; row means use the stated $n$.}
\label{tab:scale}
\small
\begin{tabular}{@{}lrcccc@{}}
\toprule
Code & $n$ & $b=1$ & $b=2$ & $b=4$ & $b=8$ \\
\midrule
Centre only & 24 & $0.669\;(0.077)$ & $0.739\;(0.059)$ & $0.708\;(0.045)$ & $0.310\;(0.025)$ \\
Width feature & 24 & $0.695\;(0.062)$ & $0.762\;(0.050)$ & $0.722\;(0.038)$ & $0.317\;(0.020)$ \\
Aperture & 48 & $0.687\;(0.067)$ & $0.753\;(0.056)$ & $0.731\;(0.044)$ & $0.346\;(0.013)$ \\
MN-sinc & 24 & $0.677\;(0.080)$ & $0.746\;(0.070)$ & $0.674\;(0.055)$ & $0.317\;(0.029)$ \\
Width-capped phase & 24 & $0.696\;(0.083)$ & $0.757\;(0.069)$ & $0.733\;(0.053)$ & $0.302\;(0.023)$ \\
Aperture RMS & 24 & $0.668\;(0.082)$ & $0.738\;(0.068)$ & $0.714\;(0.052)$ & $0.324\;(0.020)$ \\
Hard cutoff & 24 & $0.678\;(0.074)$ & $0.746\;(0.062)$ & $0.727\;(0.052)$ & $0.338\;(0.020)$ \\
Aperture, matched RMS & 24 & $0.682\;(0.075)$ & $0.747\;(0.064)$ & $0.695\;(0.042)$ & $0.321\;(0.019)$ \\
Absolute gain & 48 & $0.676\;(0.079)$ & $0.743\;(0.069)$ & $0.722\;(0.054)$ & $0.344\;(0.015)$ \\
First lobe & 24 & $0.693\;(0.063)$ & $0.758\;(0.054)$ & $0.735\;(0.040)$ & $0.348\;(0.021)$ \\
\bottomrule
\end{tabular}
\end{table}

\begin{table}[htbp]
\centering
\setlength{\tabcolsep}{4pt}
\caption{\textbf{7B budget evaluation.} Bin accuracy, mean (seed s.d.); training budgets are $b=1,2$, the primary held-out budget is $b=4$, and $b=8$ is supplementary. MN-sinc and width-capped phase are controls defined in Appendix~\ref{sec:app_related}. Contrasts use matched seeds; row means use the stated $n$.}
\label{tab:scale_s7}
\small
\begin{tabular}{@{}lrcccc@{}}
\toprule
Code & $n$ & $b=1$ & $b=2$ & $b=4$ & $b=8$ \\
\midrule
Centre only & 24 & $0.732\;(0.047)$ & $0.794\;(0.038)$ & $0.754\;(0.028)$ & $0.315\;(0.024)$ \\
Width feature & 24 & $0.697\;(0.071)$ & $0.761\;(0.060)$ & $0.730\;(0.046)$ & $0.320\;(0.033)$ \\
Aperture & 24 & $0.709\;(0.073)$ & $0.777\;(0.062)$ & $0.756\;(0.048)$ & $0.356\;(0.015)$ \\
MN-sinc & 24 & $0.711\;(0.048)$ & $0.771\;(0.046)$ & $0.701\;(0.042)$ & $0.337\;(0.022)$ \\
Aperture RMS & 24 & $0.720\;(0.065)$ & $0.780\;(0.058)$ & $0.750\;(0.043)$ & $0.321\;(0.026)$ \\
Aperture, matched RMS & 24 & $0.714\;(0.061)$ & $0.775\;(0.053)$ & $0.708\;(0.039)$ & $0.337\;(0.026)$ \\
\bottomrule
\end{tabular}
\end{table}

\subsection{Budget and gain controls at scale}
\label{sec:app_ladder_scale}

\paragraph{Budget ladder.}
The four budgets in \tabref{tab:ladder_scale} vary the ratio of video to audio widths from
$\LadNarrowRatio\times$ to $\LadWideRatio\times$. Each modality has a common width at a
given budget, so this is a two-width comparison, not a test of arbitrary ragged token supports.
At $b=\LadWideBudget$, the token-accuracy contrast against the clock is $\LadTokMEight$
at 3B (adjusted $p=\LadTokMEightAdj$) and $\SevLadTokMEight$ at 7B
(adjusted $p=\SevLadTokMEightAdj$). The endpoint comparison is exploratory.
A per-seed slope against $\log_2b$ gives $\LadSlopeTok$ at 3B and
$\SevLadSlopeTok$ at 7B; neither survives the reported correction.
The changing budgets also alter sequence length, averaging and distance from training,
so this ladder does not isolate a causal effect of width ratio.

\begin{table}[htbp]
\centering
\setlength{\tabcolsep}{4pt}
\caption{\textbf{Supplementary budget comparison.} Mean paired token-accuracy difference from the centre-only clock on 24 matched seeds. Aperture RMS normalises temporal gain energy and is a separate control from matching MN-sinc. Budgets 1 and 2 are used in training, 4 is the primary held-out budget, and 8 is an exploratory endpoint. Adjusted tests are reported in Appendix~\ref{sec:app_ladder_scale}.}
\label{tab:ladder_scale}
\small
\begin{tabular}{@{}lcccc@{}}
\toprule
Budget & $b{=}1$ & $b{=}2$ & $b{=}4$ & $b{=}8$ \\
\midrule
Aperture $-$ clock, $3$B & $+0.0047$ & $+0.0096$ & $+0.0259$ & $+0.0376$ \\
Aperture $-$ clock, $7$B & $-0.0218$ & $-0.0176$ & $+0.0021$ & $+0.0197$ \\
\addlinespace[2pt]
Aperture RMS $-$ clock, $3$B & $-0.0026$ & $-0.0021$ & $+0.0063$ & $+0.0064$ \\
\bottomrule
\end{tabular}
\end{table}

\paragraph{Gain magnitude.}
At the widest budget, the Aperture RMS arm differs from the clock by
$\LadCtlRms$ ($p\LadCtlRmsP$), while the hard cutoff gives
$\LadCtlNyq$ ($p\LadCtlNyqP$). The latter raw comparison does not survive Holm correction.
These controls motivate a direct match to MN-sinc's magnitude: \texttt{aperture\_infl}
(\tabref{tab:scale_summary}'s Matched-RMS; `Aperture, matched RMS' in \tabref{tab:scale} and
\tabref{tab:scale_s7}) rescales the Aperture gain to MN-sinc's RMS over the temporal channels at each budget.
At the primary held-out budget, the matched contrasts against MN-sinc are
$\ScaleInflVsRote$ at 3B ($p\ScaleInflVsRoteP$, $n=\ScaleInflN$) and
$\SevScaleInflVsRote$ at 7B ($p\SevScaleInflVsRoteP$, $n=\SevScaleInflN$).
Both are smaller than the unmatched contrasts and neither is statistically resolved.
Magnitude is thus a substantial contributor to the measured MN-sinc comparison.
Matching total gain energy does not match the complete gain vector or the resulting logits,
so the controls do not identify every component of the attention mechanism.

\paragraph{Signs and sidelobes.}
Two further variants replace $\operatorname{sinc}(u)$ with its absolute value or retain
only its first lobe. Absolute value preserves every channel magnitude exactly; first-lobe
truncation slightly changes total gain energy as well as removing sidelobes.
At the primary held-out budget, their contrasts against exact Aperture are
$\ClampAbsVsAp$ ($n=\ClampAbsN$) and $\ClampMainVsAp$ ($n=\ClampMainN$), with intervals
$[\ClampAbsCiLo,\ClampAbsCiHi]$ and $[\ClampMainCiLo,\ClampMainCiHi]$.
Neither equivalence test resolves equivalence at the declared margin
$\pm\ClampAbsMargin$. Both modified gains violate exact merge consistency, with
residuals $\ClampCovAbs$ and $\ClampCovMain$, respectively. Their empirical performance
therefore illustrates why the representation theorem and an accuracy ranking are distinct
claims; it does not establish that signs or sidelobes are dispensable.

\subsection{Computational cost}
\label{sec:app_cost}

We measure the rotary operation at \CostShape, using medians over $\CostRounds$ interleaved
rounds. A separate gain application changes its runtime from $\CostRopeUs$ to
$\CostRopeGainUs\,\mu$s, an incremental overhead of $\CostGainOverRope$ relative to RoPE.
The incremental cost is $\CostGainOverAttn$ of a causal SDPA call at the same shape.
These are component measurements on the recorded hardware, not end-to-end throughput estimates.

Since the gain is shared within a rotary pair, it can be folded into the rotary tables
$C=\cos(\theta c)$ and $S=\sin(\theta c)$, with all products applied elementwise:
\begin{equation}
 (qC+\operatorname{rot}(q)S)g
 =q(Cg)+\operatorname{rot}(q)(Sg).
 \label{eq:fold}
\end{equation}
Applying the multiplication to the $[B,1,L,D]$ tables instead of the
$[B,H,L,D]$ activations avoids repeating it across heads.
The implementation agrees with the separate application to relative residual
$\CostFoldErr$ in double precision. Its measured overhead is $\CostFoldOverRope$
when rotary tables are rebuilt within each call, $\CostFoldAmortised$ when reused across
layers, and $\CostFoldShared$ relative to tables shared across the batch.
Caching gains for a finite set of widths gives overhead $\CostFoldHoistOverRope$
when combined with folding. Benefits depend on table reuse and fusion in the host implementation.

For arbitrary support mixtures, storing the complex code requires $\CostPhiBytes$ bytes
per token, $\CostPhiOverKv$ of a single-layer KV entry at the measured shape.
These bytes count Fourier coordinates only. The additive implementation also needs a mass
scalar: 4 bytes in float32 or 8 in float64, unless the compressor already stores it.
The positional state is stored once per token, while KV entries are stored per layer.
The synthetic, small-reader interval and assisted-scale studies apply interval gains
separately. The native 3B/7B and spatial support paths instead replace rotary cosine/sine
tables with the real/imaginary moments, as does the small EPIC reader's stored-support arm.
The folded interval operation implements the same algebra as separate gain application,
subject to floating-point reassociation.

\subsection{Geometric reconstruction diagnostics}
\label{sec:app_ladder}

\tabref{tab:ladder} compares anchored recursion with a linear estimator using relative offsets
on generated segmentations. The estimator is fitted on half the streams and evaluated on the
remaining half, with no stream shared across partitions. The recursion reports a median of
per-stream maximum errors; the linear probe reports the median token error on its held-out streams.
Conditions with too few survivors for a probe are omitted, with valid counts in the table.
The anchored recursion is exact on valid tilings, whereas applying it after gaps or overlaps
need not be. Its error is a diagnostic of that recursion, not a lower bound on all readers.
The relative-offset estimator provides a distribution-dependent test of local recovery,
with a different observation model and error metric.

With the left boundary set to zero, unrolling the recursion gives $w=A_Nc$, where
\begin{equation}
 (A_N)_{1\cdot}=2e_1,\qquad
 (A_N)_{(j+1)\cdot}=2e_{j+1}-2e_j-(A_N)_{j\cdot}.
 \label{eq:recover}
\end{equation}
Its row norms satisfy $\|(A_N)_{j\cdot}\|_1=2(2j-1)$, while the row sums alternate
between $2$ and $-2$. The map therefore depends on an absolute anchor.
Measured conditioning grows approximately quadratically with $N$
(exponent $\RedCondExp$). This concerns sensitivity of the exact inverse and does not
by itself bound the error of an estimator under a distribution of widths.
\tabref{tab:ident} separately evaluates the nullspace dimensions and identified-set
widths in \thmref{thm:reach}.

\begin{table}[htbp]
\centering
\setlength{\tabcolsep}{4pt}
\caption{\textbf{Geometry and width reconstruction.} (a) Width errors normalised by the original mean width. From 2,048 random 32-part streams, conditions retain the stated valid-stream counts. Recursion reports the median maximum error per stream using absolute centres. Linear probes report median token error from relative offsets within $L$ hops, using disjoint stream partitions. (b) Jacobian ranks within fixed families of two-dimensional tilings; full rank gives local identification within that family. (c) Conditioning of the anchored recovery operator. These quantities describe different estimators and information sets.}
\label{tab:ladder}
\small
\begin{tabular}{@{}llccc@{}}
\toprule
\multicolumn{5}{@{}l}{\textbf{(a) Width reconstruction error / $\bar{w}$}}\\
Policy & Streams & Recursion & Linear, $L=1$ & Linear, $L=16$ \\
\midrule
Contiguous tiling & 2048 & $1.78\times 10^{-14}$ & $0.33$ & $0.25$ \\
Unanchored window & 2048 & $0.63$ & $0.34$ & $0.27$ \\
Overlap, $25\%$ & 2048 & $0.98$ & $0.42$ & $0.31$ \\
Overlap, $50\%$ & 2048 & $1.95$ & $0.50$ & $0.37$ \\
Retain $75\%$ & 2048 & $7.10$ & $0.42$ & $0.42$ \\
Retain $50\%$ & 2044 & $12.56$ & $0.53$ & $0.53$ \\
Retain $25\%$ & 1424 & $17.32$ & $0.60$ & $0.59$ \\
\midrule
\multicolumn{5}{@{}l}{(b) Local rank within a fixed tiling family}\\
Guillotine tiling & Boxes & DOF & Full rank & $\sigma_{\min}/\sigma_{\max}$ \\
\midrule
Guillotine & 2 & 1 & 256/256 & $1$ \\
Guillotine & 4 & 3 & 256/256 & $0.215$ \\
Guillotine & 8 & 7 & 256/256 & $0.0315$ \\
Guillotine & 16 & 15 & 256/256 & $0.00556$ \\
Guillotine & 24 & 23 & 256/256 & $0.00159$ \\
Pinwheel & 5 & 4 & 256/256 & $0.707$ \\
\midrule
\multicolumn{5}{@{}l}{(c) conditioning of the recovery operator $w=A_N c$}\\
$N$ & & $\kappa(A_N)$ & $\|A_N\|_2$ & $\max_j\|(A_N)_{j\cdot}\|_1$ \\
\midrule
2 & & $5.828$ & $4.828$ & $6$ \\
8 & & $103.1$ & $20.31$ & $30$ \\
32 & & $1659$ & $81.47$ & $126$ \\
128 & & $2.66\times 10^{4}$ & $325.9$ & $510$ \\
\bottomrule
\end{tabular}
\end{table}

\begin{table}[htbp]
\centering
\setlength{\tabcolsep}{4pt}
\caption{\textbf{Identified sets from relative centre offsets.} $D$ maps the widths of an $N$-part tiling to adjacent offsets; $E$ maps a $2N$-part tiling to offsets among its retained tokens. The reported nullity of $E$ averages 194, 254, 256, 256 and 256 valid half-retention draws for increasing $N$. Identifiable is the fraction of surviving widths fixed by those offsets. Floors are mean identified-set half-diameters, normalised by $\bar w$. Dashes denote spans covered by the full-span column.}
\label{tab:ident}
\small
\begin{tabular}{@{}lccccccc@{}}
\toprule
 & \multicolumn{3}{c}{kernel of the width$\,\to\,$offset map} & \multicolumn{4}{c}{minimax floor / $\bar{w}$} \\
\cmidrule(lr){2-4}\cmidrule(lr){5-8}
$N$ & $\dim\ker D$ & $\dim\ker E$ & Identifiable & $L{=}1$ & $L{=}4$ & $L{=}16$ & Full span \\
\midrule
4 & 1 & 3.9 & 0.000 & 0.880 & --- & --- & 0.480 \\
8 & 1 & 8.4 & 0.000 & 0.813 & 0.318 & --- & 0.246 \\
16 & 1 & 16.5 & 0.000 & 0.781 & 0.269 & --- & 0.121 \\
32 & 1 & 32.6 & 0.000 & 0.767 & 0.249 & 0.085 & 0.062 \\
64 & 1 & 64.4 & 0.000 & 0.759 & 0.237 & 0.073 & 0.030 \\
\bottomrule
\end{tabular}
\end{table}

\subsection{Linear recovery over a finite observation span}
\label{sec:app_span}

\tabref{tab:span} gives the full $(N,L)$ grid for width regression from relative offsets.
At $N=\SpanN$, held-out $R^2$ increases by $\SpanTilingGain$ to $\SpanTilingMax$ on
tilings as span grows, and changes by $\SpanEvictGain$ under eviction.
The linear reader's residual variance on tilings is $\SpanTilingMiss$.
A separate reconstruction using the alternating-mode structure leaves residual variance
$\SpanTilingMissIdent$ on the same ensemble. This shows that much of the linear reader's
error comes from its estimator class; it cannot be equated with information unavailable to
all relative readers. Both quantities are distribution-dependent averages, whereas
\thmref{thm:reach} gives a worst-case identified-set bound.

At fixed span, recovery improves with $N$ in these generated tilings. For example,
$R^2$ at $L=1$ is $\SpanTilingLOne$ at $N=\SpanN$ and $\SpanTilingMinOne$ at
$N=\SpanNmin$. This behaviour is compatible with shrinking parity-wise minima under the
sampled width distribution and does not follow from the condition number of the anchored inverse.

\begin{table}[htbp]
\centering
\setlength{\tabcolsep}{4pt}
\caption{\textbf{Linear width prediction from nearby centre offsets.} Held-out $R^2$ on the segmentation ensemble, with offset span $L$ in token hops. Eviction starts from $2N$ parts and retains half on average. $L_{\rm sat}$ is the smallest tested span within 0.01 of the maximum $R^2$. Dashes denote spans at or beyond the stream length. These are fitted prediction scores, distinct from the worst-case floors in Table~\ref{tab:ident}.}
\label{tab:span}
\small
\begin{tabular}{@{}llccccccc@{}}
\toprule
Regime & $N$ & $L{=}1$ & $L{=}2$ & $L{=}4$ & $L{=}8$ & $L{=}16$ & $R^2_{\max}$ & $L_{\mathrm{sat}}$ \\
\midrule
Tiling & 4 & 0.518 & 0.554 & --- & --- & --- & 0.554 & 2 \\
 & 8 & 0.580 & 0.660 & 0.676 & --- & --- & 0.676 & 4 \\
 & 16 & 0.618 & 0.726 & 0.758 & 0.759 & --- & 0.759 & 4 \\
 & 32 & 0.648 & 0.764 & 0.812 & 0.815 & 0.813 & 0.815 & 4 \\
 & 64 & 0.656 & 0.783 & 0.849 & 0.860 & 0.860 & 0.860 & 8 \\
\midrule
Eviction & 4 & 0.120 & 0.131 & --- & --- & --- & 0.131 & 2 \\
 & 8 & 0.100 & 0.098 & 0.103 & --- & --- & 0.103 & 1 \\
 & 16 & 0.116 & 0.116 & 0.115 & 0.114 & --- & 0.116 & 1 \\
 & 32 & 0.122 & 0.123 & 0.122 & 0.122 & 0.121 & 0.123 & 1 \\
 & 64 & 0.130 & 0.130 & 0.130 & 0.129 & 0.128 & 0.130 & 1 \\
\bottomrule
\end{tabular}
\end{table}

\subsection{EPIC-Kitchens onset localisation}
\label{sec:app_onset}

\paragraph{Four retention-policy interventions.}
\label{sec:factorial}
Write $C,D,M,E$ for width sensitivity under \textsc{Coarsen}, \textsc{Dilate},
\textsc{Mask} and \textsc{Evict}. Coarsening yields a tiling, dilation produces overlap,
masking zeros selected payloads while preserving supports, and eviction removes selected slots.
Masking and eviction remove the same $\GeomLostMask$ tokens per stream on average, with
onset-bearing evidence protected. The distinction between overlap and gaps prevents a causal
factorial interpretation. The direct $D-M$ comparison is $\GeomEffInversion$ for Aperture RMS
and $\GeomPlainInversion$ for exact Aperture; both retain the ordering with matched-seed analysis.
The primary permutation includes padding slots. Live-only permutations test this perturbation
choice separately. Width sensitivity is a within-model quantity and does not rank model accuracy.

\begin{table}[t]
\centering
\setlength{\tabcolsep}{4pt}
\caption{\textbf{Width sensitivity in the geometry/content design.} Test accuracy minus width-permuted accuracy, mean (seed s.d.); $n=15$ per condition and code. The primary audit includes padding slots. The live-slot control gives the same ordering (Appendix~\ref{sec:app_wshuf}).}
\label{tab:geom}
\small
\begin{tabular}{@{}lllcc@{}}
\toprule
Policy & Support & Content & Aperture & Aperture RMS \\
\midrule
Coarsen & Tiling & Retained & $+0.1696\;(0.0220)$ & $+0.2049\;(0.0233)$ \\
Dilate & Overlap & Retained & $+0.1833\;(0.0260)$ & $+0.2132\;(0.0188)$ \\
Mask & Tiling & Masked & $+0.1232\;(0.0322)$ & $+0.1265\;(0.0378)$ \\
Evict & Gaps & Removed & $+0.0942\;(0.0188)$ & $+0.1038\;(0.0123)$ \\
\bottomrule
\end{tabular}
\end{table}

\paragraph{Features and task.}
Each $\GeomWindow$\,s window contains $\GeomJv$ SigLIP\,2 video features of dimension
$\GeomDvid$, extracted at $\GeomVidHz$\,fps, and $\GeomJa$ pooled Whisper audio features
of dimension $\GeomDaud$ at $\GeomAudHz$\,Hz. Video slots are pooled into contiguous tokens
of at most $\GeomMmax$ slots before applying a retention intervention. Audio is unchanged.
The query concatenates the means of $\GeomOnsetP$ uncompressed features before and after
a detected change, giving dimension $\GeomDquery$. The label is one of $\GeomNbins$ bins
of width $\GeomBinS$\,s over $\GeomTmin$--$\GeomTmax$\,s, with chance accuracy
$\GeomChance$ and observed majority-class accuracy $\GeomMajority$.
This query deliberately provides content identity so that the reader must localise that
content in time. It is a controlled feature-derived task, not natural-language event recognition.

\paragraph{Label construction and retention.}
Changes are detected on uncompressed features using a matched step filter with threshold
$\GeomOnsetThr$ and isolation $\GeomOnsetIsolate$ frames. A target bin is sampled first and
the window is positioned around the change, controlling the label prior.
Times are relative to the window start. The average coarsened token is
$\GeomMeanWCoarsen$\,s wide, exceeding the bin width.
The onset-bearing token is retained in the deletion conditions; deleting neighbours changes
context while preserving the principal onset evidence. Masking preserves token locations and
widths but zeros selected features; eviction removes the selected slots. Dilation creates
overlapping supports while retaining content. These interventions do not cover every behaviour
of a learned compressor.

\paragraph{Geometric availability diagnostic.}
The availability column in \tabref{tab:pub} compares three fixed retrieval decoders.
After selecting a token by query/content similarity, they use its centre alone, interpolate
within its true width, or interpolate using a width reconstructed from the anchored centres.
Availability is true-width decoder accuracy minus the larger accuracy of the other two.
It measures this specified decoder contrast, not an information bound over all centre-only readers.

\paragraph{Splits and training.}
The participant-disjoint partitions contain $\GeomNTrainItems$ items from
$\GeomNTrainVids$ videos for training, $\GeomNValItems$ from $\GeomNValVids$ for validation,
and $\GeomNTestItems$ from $\GeomNTestVids$ for testing; the test partition spans
\EpicTestParticipants{} participants (\EpicTestParticipantList).
The reader has $\GeomNlayer$ layers, $\GeomDmodel$ channels, $\GeomNhead$ heads and
$\GeomNparams$ parameters, with rotary periods from $\GeomPmin$ to $\GeomPmax$\,s.
Training uses $\GeomSteps$ steps, batch size $\GeomBatch$, AdamW at learning rate
$\GeomLr$, weight decay $\GeomWd$ and $\GeomWarmup$ warmup steps.
Checkpoints are selected by validation performance; $\GeomNval$ items are used per
validation evaluation. The feature arm adds $\GeomNparamsWfeat$ parameters
($\GeomNparamsWfeatPct\%$).

\paragraph{Degeneracy and optimisation checks.}
Each run is evaluated with the query removed, the stream removed and the time channel
shuffled. The largest query-only, stream-only and time-shuffled accuracies are
$\GeomAuditQonly$, $\GeomAuditSonly$ and $\GeomAuditTshuf$, compared with
$\GeomAccCoarsen$ for the full coarsened model. The stream-only maximum occurs on a
$\GeomAuditSonlyCell$ run whose full accuracy is $\GeomAuditSonlyAcc$.
A sensitivity analysis retains runs whose best validation checkpoint occurs after
$\GeomConvFrac$ of training. Its stream-only maximum is $\GeomAuditSonlyConv$.
This is a heuristic check on optimisation, not a criterion for excluding runs from the
primary result. All completed seeds remain in the main analysis.

\paragraph{Difficulty and width heterogeneity.}
Centre-only accuracies vary across conditions: $\GeomBaseCoarsen$, $\GeomBaseDilate$,
$\GeomBaseMask$ and $\GeomBaseEvict$. Their $D-M$ contrast is
$\GeomHeadInversion$ ($p\GeomHeadInversionP$), so difficulty cannot be assumed constant.
An accuracy-adjusted regression gives coefficient $\GeomAncContent$; this remains descriptive
because accuracy is itself an intervention outcome. On $\GeomCtlStreams$ streams, width coefficients of variation range
from $\GeomCvLo$ to $\GeomCvHi$. The onset token is protected from deletion, but retrieval
hit rates and optimisation variability still differ across cells.
These diagnostics support the observed ordering without identifying a unique mechanism.

\paragraph{Exact Aperture and Aperture RMS.}
\label{sec:app_plain}
Both codes are evaluated on the same complete four-condition design.
Exact Aperture's $L_1$ and $L_2$ contrasts are $\GeomPlainContent$ and
$\GeomPlainTiling$, with ordering contrast $\GeomPlainInversion$ and Hedges'
$g=\GeomPlainInversionG$. The Aperture RMS ordering is $\GeomEffInversion$
with $g=\GeomEffInversionG$. The ordering therefore holds for both representations;
normalisation changes its size. \tabref{tab:geom} reports the primary estimates and
\tabref{tab:codes} supplies the full arm comparison.

\paragraph{Content dose.}
Masking doses preserve the tiling and widths while removing increasing amounts of content.
The keep-$1$ anchor is the coarsened condition. The first deletion, at keep
$\GeomDoseFirstKeep$, changes width sensitivity by $\GeomDoseFirst$
(raw $p\GeomDoseFirstP$, adjusted $p=\GeomDoseFirstAdj$).
The masked levels lie in $[\GeomDosePlateauLo,\GeomDosePlateauHi]$, with within-mask
Spearman correlation $\GeomDoseRhoMasked$. Thus the observed profile is better described
as a first-step change followed by a plateau than as a resolved monotone dose response.
All levels and their seed spreads are reported in \tabref{tab:dose}.

\begin{table}[htbp]
\centering
\setlength{\tabcolsep}{4pt}
\caption{\textbf{Positional-code comparison within each retention policy.} Test accuracy and within-model width sensitivity $\Delta_w$, mean (seed s.d.). Each row reports its own completed-run count. The RMS aperture and width-feature accuracy contrast is averaged over the four policies in the text.}
\label{tab:codes}
\small
\begin{tabular}{@{}lrcc@{}}
\toprule
Code & $n$ & Accuracy & $\Delta_w$ \\
\midrule
\multicolumn{4}{@{}l}{Coarsen} \\
Token index & 15 & 0.4434\,(0.0210) & +0.0000\,(0.0000) \\
Centre only & 15 & 0.4643\,(0.0166) & +0.0000\,(0.0000) \\
Width feature & 15 & 0.5214\,(0.0242) & +0.2626\,(0.0298) \\
Hard cutoff RMS & 5 & 0.5213\,(0.0309) & +0.1678\,(0.0305) \\
Aperture & 15 & 0.5174\,(0.0210) & +0.1696\,(0.0220) \\
Aperture RMS & 15 & 0.5238\,(0.0123) & +0.2049\,(0.0233) \\
\midrule
\multicolumn{4}{@{}l}{Dilate} \\
Token index & 10 & 0.4308\,(0.0143) & +0.0000\,(0.0000) \\
Centre only & 15 & 0.4508\,(0.0291) & +0.0000\,(0.0000) \\
Width feature & 15 & 0.5020\,(0.0195) & +0.2044\,(0.0313) \\
Aperture & 15 & 0.5099\,(0.0234) & +0.1833\,(0.0260) \\
Aperture RMS & 15 & 0.5118\,(0.0146) & +0.2132\,(0.0188) \\
\midrule
\multicolumn{4}{@{}l}{Mask} \\
Token index & 15 & 0.4048\,(0.0348) & +0.0000\,(0.0000) \\
Centre only & 15 & 0.4627\,(0.0311) & +0.0000\,(0.0000) \\
Width feature & 15 & 0.4999\,(0.0485) & +0.1881\,(0.0396) \\
Hard cutoff RMS & 5 & 0.4566\,(0.0477) & +0.1059\,(0.0304) \\
Aperture & 15 & 0.4951\,(0.0366) & +0.1232\,(0.0322) \\
Aperture RMS & 15 & 0.4932\,(0.0647) & +0.1265\,(0.0378) \\
\midrule
\multicolumn{4}{@{}l}{Evict} \\
Token index & 10 & 0.4729\,(0.0128) & +0.0000\,(0.0000) \\
Centre only & 15 & 0.4081\,(0.0251) & +0.0000\,(0.0000) \\
Width feature & 15 & 0.4367\,(0.0277) & +0.1430\,(0.0212) \\
Aperture & 15 & 0.4393\,(0.0253) & +0.0942\,(0.0188) \\
Aperture RMS & 15 & 0.4404\,(0.0250) & +0.1038\,(0.0123) \\
\bottomrule
\end{tabular}
\end{table}

\begin{table}[htbp]
\centering
\setlength{\tabcolsep}{4pt}
\caption{\textbf{Sensitivity across content-retention levels.} Mean test accuracy for the centre-only and Aperture RMS readers; width sensitivity is mean (seed s.d.). Retained fractions below one use masking, with unchanged support geometry. There are 15 seeds at fractions 1 and 0.5, and five at each other level.}
\label{tab:dose}
\small
\begin{tabular}{@{}lcccc@{}}
\toprule
Retained & Masked tokens & Centre only & Aperture RMS & $\Delta_w$ \\
\midrule
1 & 0.00 & 0.4643 & 0.5238 & +0.2049\,(0.0233) \\
0.875 & 1.34 & 0.4168 & 0.4748 & +0.1316\,(0.0394) \\
0.75 & 3.33 & 0.4137 & 0.4250 & +0.0990\,(0.0245) \\
0.5 & 7.14 & 0.4627 & 0.4932 & +0.1265\,(0.0378) \\
0.25 & 10.58 & 0.3750 & 0.3969 & +0.1033\,(0.0282) \\
\bottomrule
\end{tabular}
\end{table}

\subsection{Transfer between retention policies}
\label{sec:app_trans}

We reload $\GeomTransNruns$ checkpoints and evaluate each on $\GeomTransNcells$ cells
with $\GeomTransNitems$ test items, requiring its original evaluation to reproduce.
Tables~\ref{tab:transfer}, \ref{tab:transferfull}, \ref{tab:transfer_mask} and
\ref{tab:transfer_evict} report transfer from coarsen, dilate, mask and evict training,
respectively, with checkpoint weights held fixed.
Width sensitivity remains positive in $\GeomTransNseedcellsPos$ of
$\GeomTransNseedcells$ checkpoint/cell pairs, but this does not rank their accuracy.
Across the four width-aware arms, the rank correlation between retained sensitivity and
accuracy relative to the clock is $\GeomTransRankRho$; restricting to training policies
shared by all four arms gives $\GeomTransRankRhoMatched$.
With four arms, both are descriptive summaries.

\paragraph{Seed groups and transfer comparisons.}
Evaluation cells are first averaged within checkpoint. The four training policies use the
same five seeds, so their contrasts are then averaged within seed before inference.
The resulting independent units are five seed groups, each containing all four policies.
The pooled mean is unchanged by this grouping. For exact Aperture minus the width-feature
arm it is $\GeomTransPoolPlAperture$, with 95\% $t$ interval
$[\RebuildTransferLo,\RebuildTransferHi]$ and unadjusted $p=0.00604$.
An exhaustive two-sided sign-flip test over the five seed contrasts gives $p=0.0625$.
This contrast is reported descriptively; five-seed distributional inference is sensitive
to its assumptions and does not establish a corrected placement advantage.

The width-feature-minus-clock interval is $[-0.0484,-0.0205]$, exact Aperture minus clock
is $[-0.0152,+0.0112]$, and Aperture RMS minus clock is $[-0.0405,+0.0047]$.
These seed-level intervals condition on the fixed test population. The hard-cutoff arm has
fewer training policies; its individual cells remain available without treating missing
policy/seed combinations as observations. The complete numerical archive preserves
checkpoint counts and seed membership as separate fields.

\paragraph{Content-deleting versus overlapping supports.}
On held-out dilation, both the feature and exact Aperture arms have positive mean contrasts
against the clock ($\GeomTransPoolVsWfeatDil$ and $\GeomTransPoolVsApertureDil$).
The feature arm's difference between the deletion and dilation cell sets is
$\GeomTransPoolVsWfeatInt$, reported as a descriptive contrast. Splitting checkpoints by whether their training policy
preserves or removes content is exploratory: it suggests a larger feature-arm deficit after
content-preserving training, without showing that width fabricates content or that
attenuation prevents this behaviour. Absolute transfer accuracies are lower
than in-distribution accuracies, and the reported advantages are relative comparisons among
these transferred models.

\begin{table}[htbp]
\centering
\setlength{\tabcolsep}{3.5pt}
\caption{\textbf{Transfer from coarsen training.} Each row uses five saved checkpoints evaluated on 1,024 test items. $\dagger$ marks the training policy. Width sensitivity is a within-model mean; accuracy differences are paired by seed. Raw denotes unmerged tokens, for which the width-permutation contrast is not reported.}
\label{tab:transfer}
\small
\begin{tabular}{@{}lccccc@{}}
\toprule
Code & Coarsen & Dilate & Mask & Evict & Raw \\
\midrule
\multicolumn{6}{@{}l}{\textbf{(a) Width sensitivity $\Delta_w$}} \\
Token index & $+0.000$$^\dagger$ & $+0.000$ & $+0.000$ & $+0.000$ & --- \\
Centre only & $+0.000$$^\dagger$ & $+0.000$ & $+0.000$ & $+0.000$ & --- \\
Width feature & $+0.266$$^\dagger$ & $+0.009$ & $+0.054$ & $+0.070$ & --- \\
Hard cutoff RMS & $+0.168$$^\dagger$ & $+0.011$ & $+0.060$ & $+0.040$ & --- \\
Aperture & $+0.162$$^\dagger$ & $-0.008$ & $+0.024$ & $+0.042$ & --- \\
Aperture RMS & $+0.203$$^\dagger$ & $+0.033$ & $+0.053$ & $+0.031$ & --- \\
\midrule
\multicolumn{6}{@{}l}{\textbf{(b) Accuracy minus centre only, mean (seed s.d.)}} \\
Token index & $-0.000\;(0.032)$$^\dagger$ & $+0.228\;(0.045)$ & $-0.024\;(0.023)$ & $+0.020\;(0.016)$ & $-0.225\;(0.047)$ \\
Width feature & $+0.073\;(0.009)$$^\dagger$ & $+0.018\;(0.023)$ & $-0.061\;(0.036)$ & $-0.040\;(0.025)$ & $+0.026\;(0.018)$ \\
Hard cutoff RMS & $+0.067\;(0.016)$$^\dagger$ & $+0.035\;(0.015)$ & $+0.025\;(0.016)$ & $-0.014\;(0.028)$ & $+0.017\;(0.026)$ \\
Aperture & $+0.064\;(0.028)$$^\dagger$ & $+0.034\;(0.073)$ & $+0.001\;(0.026)$ & $+0.015\;(0.029)$ & $+0.017\;(0.006)$ \\
Aperture RMS & $+0.066\;(0.013)$$^\dagger$ & $+0.047\;(0.062)$ & $-0.010\;(0.015)$ & $-0.044\;(0.024)$ & $+0.036\;(0.028)$ \\
\bottomrule
\end{tabular}
\end{table}

\begin{table}[htbp]
\centering
\setlength{\tabcolsep}{3.5pt}
\caption{\textbf{Transfer from dilate training.} Each row uses five saved checkpoints evaluated on 1,024 test items. $\dagger$ marks the training policy. Width sensitivity is a within-model mean; accuracy differences are paired by seed. Raw denotes unmerged tokens, for which the width-permutation contrast is not reported.}
\label{tab:transferfull}
\small
\begin{tabular}{@{}lccccc@{}}
\toprule
Code & Coarsen & Dilate & Mask & Evict & Raw \\
\midrule
\multicolumn{6}{@{}l}{\textbf{(a) Width sensitivity $\Delta_w$}} \\
Centre only & $+0.000$ & $+0.000$$^\dagger$ & $+0.000$ & $+0.000$ & --- \\
Width feature & $+0.273$ & $+0.206$$^\dagger$ & $+0.058$ & $+0.082$ & --- \\
Aperture & $+0.148$ & $+0.168$$^\dagger$ & $+0.041$ & $+0.034$ & --- \\
Aperture RMS & $+0.201$ & $+0.217$$^\dagger$ & $+0.049$ & $+0.038$ & --- \\
\midrule
\multicolumn{6}{@{}l}{\textbf{(b) Accuracy minus centre only, mean (seed s.d.)}} \\
Width feature & $+0.054\;(0.029)$ & $+0.040\;(0.036)$$^\dagger$ & $-0.073\;(0.021)$ & $-0.038\;(0.055)$ & $+0.078\;(0.052)$ \\
Aperture & $+0.028\;(0.061)$ & $+0.043\;(0.051)$$^\dagger$ & $-0.012\;(0.026)$ & $-0.026\;(0.050)$ & $+0.047\;(0.047)$ \\
Aperture RMS & $+0.069\;(0.035)$ & $+0.054\;(0.049)$$^\dagger$ & $-0.005\;(0.025)$ & $-0.046\;(0.040)$ & $+0.060\;(0.055)$ \\
\bottomrule
\end{tabular}
\end{table}

\begin{table}[htbp]
\centering
\setlength{\tabcolsep}{3.5pt}
\caption{\textbf{Transfer from mask training.} Each row uses five saved checkpoints evaluated on 1,024 test items. $\dagger$ marks the training policy. Width sensitivity is a within-model mean; accuracy differences are paired by seed. Raw denotes unmerged tokens, for which the width-permutation contrast is not reported.}
\label{tab:transfer_mask}
\small
\begin{tabular}{@{}lccccc@{}}
\toprule
Code & Coarsen & Dilate & Mask & Evict & Raw \\
\midrule
\multicolumn{6}{@{}l}{\textbf{(a) Width sensitivity $\Delta_w$}} \\
Token index & $+0.000$ & $+0.000$ & $+0.000$$^\dagger$ & $+0.000$ & --- \\
Centre only & $+0.000$ & $+0.000$ & $+0.000$$^\dagger$ & $+0.000$ & --- \\
Width feature & $+0.166$ & $+0.010$ & $+0.184$$^\dagger$ & $+0.101$ & --- \\
Hard cutoff RMS & $+0.085$ & $-0.004$ & $+0.106$$^\dagger$ & $+0.046$ & --- \\
Aperture & $+0.118$ & $-0.005$ & $+0.121$$^\dagger$ & $+0.052$ & --- \\
Aperture RMS & $+0.111$ & $+0.001$ & $+0.123$$^\dagger$ & $+0.060$ & --- \\
\midrule
\multicolumn{6}{@{}l}{\textbf{(b) Accuracy minus centre only, mean (seed s.d.)}} \\
Token index & $-0.040\;(0.032)$ & $+0.279\;(0.052)$ & $-0.037\;(0.020)$$^\dagger$ & $-0.068\;(0.040)$ & $-0.254\;(0.027)$ \\
Width feature & $+0.044\;(0.045)$ & $+0.037\;(0.102)$ & $+0.039\;(0.056)$$^\dagger$ & $-0.007\;(0.035)$ & $-0.000\;(0.055)$ \\
Hard cutoff RMS & $+0.001\;(0.052)$ & $+0.036\;(0.057)$ & $+0.013\;(0.047)$$^\dagger$ & $-0.045\;(0.054)$ & $-0.028\;(0.014)$ \\
Aperture & $+0.059\;(0.034)$ & $+0.008\;(0.066)$ & $+0.060\;(0.049)$$^\dagger$ & $+0.003\;(0.040)$ & $+0.009\;(0.050)$ \\
Aperture RMS & $+0.018\;(0.079)$ & $+0.046\;(0.076)$ & $+0.033\;(0.084)$$^\dagger$ & $-0.015\;(0.057)$ & $-0.020\;(0.065)$ \\
\bottomrule
\end{tabular}
\end{table}

\begin{table}[htbp]
\centering
\setlength{\tabcolsep}{3.5pt}
\caption{\textbf{Transfer from evict training.} Each row uses five saved checkpoints evaluated on 1,024 test items. $\dagger$ marks the training policy. Width sensitivity is a within-model mean; accuracy differences are paired by seed. Raw denotes unmerged tokens, for which the width-permutation contrast is not reported.}
\label{tab:transfer_evict}
\small
\begin{tabular}{@{}lccccc@{}}
\toprule
Code & Coarsen & Dilate & Mask & Evict & Raw \\
\midrule
\multicolumn{6}{@{}l}{\textbf{(a) Width sensitivity $\Delta_w$}} \\
Centre only & $+0.000$ & $+0.000$ & $+0.000$ & $+0.000$$^\dagger$ & --- \\
Width feature & $+0.164$ & $+0.003$ & $+0.067$ & $+0.148$$^\dagger$ & --- \\
Aperture & $+0.113$ & $-0.010$ & $+0.047$ & $+0.088$$^\dagger$ & --- \\
Aperture RMS & $+0.125$ & $-0.006$ & $+0.069$ & $+0.109$$^\dagger$ & --- \\
\midrule
\multicolumn{6}{@{}l}{\textbf{(b) Accuracy minus centre only, mean (seed s.d.)}} \\
Width feature & $+0.057\;(0.036)$ & $+0.015\;(0.032)$ & $-0.025\;(0.023)$ & $+0.029\;(0.028)$$^\dagger$ & $+0.050\;(0.046)$ \\
Aperture & $+0.056\;(0.056)$ & $-0.002\;(0.033)$ & $+0.001\;(0.055)$ & $+0.023\;(0.046)$$^\dagger$ & $+0.048\;(0.074)$ \\
Aperture RMS & $+0.047\;(0.037)$ & $+0.003\;(0.018)$ & $-0.004\;(0.046)$ & $+0.017\;(0.020)$$^\dagger$ & $+0.045\;(0.044)$ \\
\bottomrule
\end{tabular}
\end{table}

\subsection{Width-permutation controls}
\label{sec:app_wshuf}

The primary width audit permutes the padded width tensor, including padding slots.
A live token can therefore receive zero width: the measured fractions are
$\GeomWshufPadFracCoarsen\%$ in coarsening, dilation and masking, and
$\GeomWshufPadFracEvict\%$ under eviction. A second audit permutes live slots only.
Across the same checkpoints, its ordering contrast is $\GeomWshufValidEffInversion$,
compared with $\GeomWshufFullEffInversion$ for the primary audit; the ordering persists.
This definition is relevant because width permutation is an intervention on the model input,
not direct removal of width information from a separately trained model.

We evaluate $\GeomWshufNruns$ width-aware checkpoints under $\GeomWshufNmodes$
perturbations, including one transposition, partial permutations and permutations within
width quartiles. The $\GeomWshufNblind$ centre-only controls have exactly zero width
contrast in every mode. A transposition gives mean loss $\GeomWshufSwapGap$, compared with
$\GeomWshufFullGap$ for a full permutation. Quartile-restricted permutations move more
slots than the matched partial permutation but incur less loss; the paired difference is
$\GeomWshufPairedQuant\pm\GeomWshufPairedQuantSd$.
This pattern in \tabref{tab:wshuf} demonstrates sensitivity to the magnitude of width errors. It does not, by
itself, distinguish useful width dependence from graded sensitivity to inputs outside the
training distribution.

\begin{table}[htbp]
\centering
\setlength{\tabcolsep}{4pt}
\caption{\textbf{Width-permutation controls on fixed checkpoints.} Aperture RMS, five checkpoints per policy; perturbation magnitudes average all 20 checkpoints. Moved is the fraction of live slots with a changed width; RMS width change is in seconds, and relative gain change is dimensionless. Policy columns report Aperture RMS width sensitivity on five seeds per policy. All-slot permutation includes padding; other modes operate on live slots. Extent-blind controls have zero sensitivity in every mode.}
\label{tab:wshuf}
\small
\begin{tabular}{@{}lccccccc@{}}
\toprule
Permutation & Moved & $\lVert\Delta w\rVert_{\mathrm{rms}}$ & $\lVert\Delta g\rVert/\lVert g\rVert$ & Coarsen & Dilate & Mask & Evict \\
\midrule
\multicolumn{8}{@{}l}{(a) Padding and live-slot controls} \\
All slots & 0.855 & 1.036 & 0.325 & +0.2033 & +0.2168 & +0.1234 & +0.1092 \\
Live only & 0.809 & 0.890 & 0.266 & +0.1746 & +0.1975 & +0.0891 & +0.0670 \\
\midrule
\multicolumn{8}{@{}l}{(b) Fraction of live slots permuted} \\
One swap & 0.146 & 0.368 & 0.047 & +0.0330 & +0.0451 & +0.0105 & +0.0156 \\
$25\%$ & 0.149 & 0.381 & 0.049 & +0.0424 & +0.0480 & +0.0244 & +0.0178 \\
$50\%$ & 0.368 & 0.600 & 0.121 & +0.0898 & +0.0881 & +0.0543 & +0.0279 \\
$75\%$ & 0.583 & 0.754 & 0.191 & +0.1348 & +0.1576 & +0.0785 & +0.0395 \\
$100\%$ & 0.806 & 0.886 & 0.264 & +0.1709 & +0.1918 & +0.0865 & +0.0711 \\
\midrule
\multicolumn{8}{@{}l}{(c) Restricted width changes} \\
Within quartiles & 0.400 & 0.269 & 0.071 & +0.0350 & +0.0436 & +0.0143 & +0.0070 \\
\bottomrule
\end{tabular}
\end{table}

\subsection{Training with uninformative widths}
\label{sec:app_sham}

The sham-width control permutes live widths independently at every training step and
evaluation batch. It preserves the observed width distribution while breaking its assignment
to tokens. The $\GeomShamNruns$ runs match the corresponding real-width runs in architecture,
training schedule, policy and seed.
Their mean audit gap is $\GeomShamFloor\pm\GeomShamFloorSd$, with maximum absolute gap
$\GeomShamFloorMax$. This measures the audit's behaviour for models trained with uninformative
width assignments; it is not a universal bound on distribution-shift sensitivity in other models.

Real-width minus sham-width accuracy is $\GeomShamDelta\pm\GeomShamDeltaSd$.
Applying the four-policy contrasts to these paired differences gives $L_1=$
$\GeomShamEffContent$ and $L_2=\GeomShamEffTiling$, with the same ordering as
the width audit but less precise estimates. The sham arm also performs below the centre-only
clock by $\GeomShamBelowBlind$. Accordingly, the real-minus-sham advantage should not be
read as the gain from adding width to a model without it; the corresponding real-minus-clock
contrast is $\GeomShamDeltaBlind\pm\GeomShamDeltaBlindSd$.
The comparisons in \tabref{tab:sham} have complementary control structures and do not independently identify
a causal mechanism.

\begin{table}[htbp]
\centering
\setlength{\tabcolsep}{4pt}
\caption{\textbf{Training with informative and permuted widths.} Top: Aperture RMS models with real widths or freshly permuted live widths, five matched seeds per condition. The last column counts seeds with a positive accuracy difference. Bottom: four-policy contrasts and raw $p$-values; the real-width sensitivity uses the full 15-seed grid, while sham sensitivity and accuracy differences use five seeds per cell. These comparisons assess width dependence but do not universally bound sensitivity to out-of-distribution inputs.}
\label{tab:sham}
\small
\begin{tabular}{@{}lcccccc@{}}
\toprule
\multicolumn{7}{@{}l}{(a) Per-condition means} \\
Cell & Sham gap & Real gap & Sham acc.\ & Real acc.\ & $\Delta$ acc.\ & Positive \\
\midrule
Coarsen & $-0.0012$ & $+0.2033$ & 0.4398 & 0.5199 & $+0.0801$ & 5/5 \\
Dilate & $-0.0061$ & $+0.2168$ & 0.4197 & 0.5080 & $+0.0883$ & 5/5 \\
Mask & $-0.0031$ & $+0.1234$ & 0.4318 & 0.4766 & $+0.0447$ & 3/5 \\
Evict & $+0.0031$ & $+0.1092$ & 0.4004 & 0.4369 & $+0.0365$ & 5/5 \\
\midrule
\multicolumn{7}{@{}l}{(b) Four-policy contrasts and raw $p$-values} \\
Contrast & \multicolumn{2}{c}{sham gap} & \multicolumn{2}{c}{real gap (\tabref{tab:geom})} & \multicolumn{2}{c}{$\Delta$ acc.} \\
$L_2$ & $+0.0007$ & $0.837$ & $-0.0072$ & $0.268$ & $+0.0000$ & $1$ \\
$L_1$ & $+0.0036$ & $0.286$ & $-0.0939$ & $<10^{-4}$ & $-0.0436$ & $0.056$ \\
$D-M$ & $-0.0029$ & $0.58$ & $+0.0867$ & $<10^{-4}$ & $+0.0436$ & $0.265$ \\
\bottomrule
\end{tabular}
\end{table}

\subsection{Similarity merging and attention-based retention}
\label{sec:app_pub}

We compare adjacency-based merging with similarity-based merging inspired by ToMe
\citep{tome}, and random eviction with an attention-based retention rule inspired by H$_2$O
\citep{h2o}, at matched budgets. This evaluates retention rules within the onset reader,
not the complete published systems. The grid contains $\PubNruns$ runs trained for
$\PubSteps$ steps, evaluated on $\PubNitems$ test items across $\PubNcells$ cells.
Similarity merging yields width sensitivity $\PubGapTome$, compared with
$\PubGapEvict$ for random eviction. The similarity-minus-adjacency contrast is
$\PubTomeCoarsen$, even though similarity merging has lower accuracy
($\PubAccTome$ versus $\PubAccCoarsen$).
The attention-based policy has the degeneracy reported in \tabref{tab:pub}; its results
are descriptive and excluded from the clean policy ordering.
These controls extend the observed ordering to alternative retention rules, without
implying that all such rules preserve tiling or exact interval support.

\begin{table}[htbp]
\centering
\setlength{\tabcolsep}{4pt}
\caption{\textbf{Retention-policy controls.} (a) Aperture RMS, five seeds per policy; the centre-only accuracy uses one seed. Availability is the geometric width advantage defined in Appendix~\ref{sec:app_onset}; reconstruction error is the median anchored-recursion error at the answer token. No-query accuracy diagnoses content-selection shortcuts. (b) Mean Aperture RMS test accuracy under policy transfer; underlining marks the training condition. Similarity merging and salience eviction are ToMe- and H2O-inspired controls, respectively. Salience-based selection permits shortcuts, limiting comparisons of absolute accuracy.}
\label{tab:pub}
\small
\begin{tabular}{lrrrrrrr}
\toprule
\multicolumn{8}{l}{\bfseries (a) Matched-policy evaluation}\\
Policy & Avail. &  $\Delta_w$ & S.d. & Accuracy & Clock & Error & No query\\
\midrule
Coarsen & 0.0000 & +0.2096 & 0.0093 & 0.5379 & 0.4463 & 0.00\,s & 0.036\\
Similarity merge & 0.0000 & +0.3113 & 0.0134 & 0.4900 & 0.4141 & 0.00\,s & 0.046\\
Evict & 0.2363 & +0.1084 & 0.0053 & 0.4646 & 0.4209 & 1.75\,s & 0.054\\
Salience evict & 0.2227 & +0.2473 & 0.0077 & 0.6480 & 0.5703 & 0.00\,s & 0.214\,\\
\midrule
\multicolumn{8}{l}{\bfseries (b) Test accuracy under policy transfer}\\
Policy & Coarsen & Similarity & Evict & Salience & Raw & \multicolumn{2}{l}{}\\
\midrule
Coarsen & \underline{0.5379} & 0.0605 & 0.2375 & 0.2455 & 0.5820 & \multicolumn{2}{l}{}\\
Similarity merge & 0.4197 & \underline{0.4900} & 0.2330 & 0.2609 & 0.5059 & \multicolumn{2}{l}{}\\
Evict & 0.4834 & 0.0480 & \underline{0.4646} & 0.4723 & 0.3609 & \multicolumn{2}{l}{}\\
Salience evict & 0.2242 & 0.0324 & 0.1627 & \underline{0.6480} & 0.2508 & \multicolumn{2}{l}{}\\
\bottomrule
\end{tabular}
\end{table}

\subsection{Depth and attention-window controls}
\label{sec:app_depth}

The depth grid contains $\GeomNrunsDepth$ runs at
$d\in\{\GeomDepthNs\}$ layers, under coarsening and half eviction, with at least
$\GeomDepthNseeds$ seeds per cell. Full-prefix causal attention is used at every layer,
so depth varies composition rather than the set of input positions potentially observable.
Across the endpoints in \tabref{tab:depth}, centre-only accuracy changes by $\GeomDepthBlindGainCoarsen$ on
coarsening and $\GeomDepthBlindGainEvict$ on eviction; width sensitivity changes by
$\GeomDepthGapGainCoarsen$ and $\GeomDepthGapGainEvict$, respectively.
These are trained-model results, distinct from the geometric width-recovery theorem.

\begin{table}[htbp]
\centering
\setlength{\tabcolsep}{4pt}
\caption{\textbf{Depth with full-prefix attention.} Mean test accuracy or within-model width sensitivity; at least $\GeomDepthNseeds$ seeds per cell. Width sensitivity is measured on Aperture RMS.}
\label{tab:depth}
\small
\begin{tabular}{@{}llcccc@{}}
\toprule
Policy & Quantity & 1 layer & 2 layers & 4 layers & 8 layers \\
\midrule
Coarsen & Clock accuracy & $\GeomDepthBlindCoarsenOne$ & $\GeomDepthBlindCoarsenTwo$ & $\GeomDepthBlindCoarsenFour$ & $\GeomDepthBlindCoarsenEight$ \\
Coarsen & Width sensitivity & $\GeomDepthCoarsenOne$ & $\GeomDepthCoarsenTwo$ & $\GeomDepthCoarsenFour$ & $\GeomDepthCoarsenEight$ \\
Evict & Clock accuracy & $\GeomDepthBlindEvictOne$ & $\GeomDepthBlindEvictTwo$ & $\GeomDepthBlindEvictFour$ & $\GeomDepthBlindEvictEight$ \\
Evict & Width sensitivity & $\GeomDepthEvictOne$ & $\GeomDepthEvictTwo$ & $\GeomDepthEvictFour$ & $\GeomDepthEvictEight$ \\
\bottomrule
\end{tabular}
\end{table}

\paragraph{Attention-window sweep.}
\label{sec:app_span_window}
A separate four-layer grid restricts attention to windows of $W\in\{4,8,16,32\}$ tokens.
The model's token-index receptive field grows with $W$; pooled and interleaved tokens have
unequal temporal extents, so this is not an exact conversion to a fixed number of seconds.
In \tabref{tab:span_window}, the coarsening width contrasts span $\SpanGapCoarsenRange$ and the eviction contrasts
$\SpanGapEvictRange$. Neither sequence provides a clear monotone reduction of width
sensitivity as the attention window grows. This complements the depth experiment but is
not a direct test of a theorem about estimation from relative geometry alone.

\begin{table}[htbp]
\centering
\setlength{\tabcolsep}{4pt}
\caption{\textbf{Attention-window sweep.} Mean test accuracy (seed s.d.) [number of runs]. Four layers each attend within a window of $W$ sequence positions. Variable token supports and interleaved modalities prevent a fixed conversion from this window to elapsed time. Width sensitivities in the text use only matching test audits.}
\label{tab:span_window}
\small
\begin{tabular}{@{}llcccc@{}}
\toprule
Policy & Code & $W=4$ & $W=8$ & $W=16$ & $W=32$ \\
\midrule
Coarsen & Aperture RMS & $0.505\;(0.016)$ [5] & $0.501\;(0.040)$ [5] & $0.512\;(0.012)$ [5] & $0.524\;(0.025)$ [5] \\
Coarsen & Centre only & $0.453\;(0.014)$ [5] & $0.471\;(0.045)$ [5] & $0.440\;(0.009)$ [5] & $0.460\;(0.023)$ [5] \\
Evict & Aperture RMS & $0.427\;(0.020)$ [5] & $0.405\;(0.023)$ [5] & $0.435\;(0.020)$ [5] & $0.460\;(0.031)$ [5] \\
Evict & Centre only & $0.375\;(0.031)$ [5] & $0.392\;(0.036)$ [5] & $0.430\;(0.017)$ [5] & $0.415\;(0.018)$ [5] \\
\bottomrule
\end{tabular}
\end{table}

\section{Temporal-position diagnostics in a pretrained model}
\label{sec:app_tmrope}

We examine Qwen2.5-Omni-7B's temporal rotary configuration and evaluate Aperture gains
without fine-tuning. These diagnostics concern the pretrained model and the tested items;
they are separate from the adapted onset readers in \secref{sec:app_scale}.

\paragraph{Sampling and rotary frequencies.}
The model assigns integer temporal position IDs at $\TmrAudioStepMs$\,ms resolution,
with $\TmrTemporal$ of $\TmrPairs$ rotary pairs allocated to time.
Their periods range from $\TmrPeriodFast$ to $\TmrPeriodSlow$\,s.
For sampling interval $T$, frequencies satisfying $\theta T>\pi$ exceed the scalar
Nyquist threshold. The recorded configuration has $\TmrAliasedAudio$ such temporal
channels for audio, $\TmrAliasedFpsTwo$ for video at $2$\,fps, and
$\TmrAliasedFpsOne$ at $1$\,fps.
This channel-wise check characterises sampled phases; it does not prove that the complete
multichannel positional representation is ambiguous or that a task must lose accuracy.

\paragraph{Projection-energy diagnostic.}
Across $\TmrLayers$ thinker layers, the above-threshold temporal pairs have
$\TmrEnergyQ$ and $\TmrEnergyK$ times the mean projection energy of the remaining
pairs in $W_q$ and $W_k$. Matched index splits of the spatial-height section give
$\TmrMatchedHQ$/$\TmrMatchedHK$, and spatial width gives
$\TmrMatchedWQ$/$\TmrMatchedWK$.
The temporal ratios lie within these spatial reference ranges.
Thus the weight norms do not show distinctive suppression of these temporal pairs.
Weight energy alone does not measure their activation-level contribution to predictions.

\paragraph{Gains applied without adaptation.}
We apply the interval gain to the temporal rotary pairs using each token's support width.
On $\TmrRevN$ forward/reversed video pairs, unmodified accuracy is $\TmrRevBase$,
Aperture accuracy is $\TmrRevSinc$, and hard-cutoff accuracy is $\TmrRevNyq$.
On $\TmrSensN$ multiple-choice items, the corresponding unmodified and Aperture
accuracies are $\TmrSensBase$ and $\TmrSensSinc$.
The patch does not improve either measured aggregate.
Constant temporal IDs change accuracy by $\TmrSensConstD$; shuffled, reversed and
doubled IDs give $\TmrSensShufD$, $\TmrSensRevD$ and $\TmrSensStretchD$.
These different interventions probe different aspects of temporal dependence.
A small aggregate accuracy difference can coexist with changed predictions on individual items.

\paragraph{Sensitivity-selected items.}
Of $\GateNTotal$ items, $\GateNSens$ change prediction under at least one of four
perturbations: constant, shuffled or reversed temporal IDs, or zeroed temporal channels.
For each perturbation, a leave-one-perturbation-out analysis selects items using the other
three, avoiding direct selection on the tested perturbation.
The constant-ID accuracy contrast on this subset is $\GateLooConst$ with interval
$[\GateLooConstLo,\GateLooConstHi]$ ($n=\GateLooConstN$).
Its raw $p\GateLooConstP$ does not survive the declared Holm correction.
Selection using related perturbations can still enrich sensitivity, so this is a conditional
diagnostic rather than a population-level estimate.
On the sensitivity-selected set, the Aperture contrast is $\GateSensSinc$, with interval
$[\GateSensSincLo,\GateSensSincHi]$; it does not resolve a benefit.
Fixed frame rate does not imply the common-width, unrestricted-projection conditions of \propref{prop:gauge},
so the reparameterisation result does not explain this empirical outcome.

The forward/reverse test additionally finds identical answers on $\TmrRevBlind$ of pairs.
These diagnostics motivate evaluating temporal representations with both prediction-level
perturbations and accuracy comparisons under the tokenisation changes of interest.
They support no general claim that the tested benchmarks are insensitive to time.

\section{Relation to existing positional representations}
\label{sec:app_related}

\paragraph{Integrated Fourier features.}
Averaging a Fourier feature over a region is an established construction.
Integrated positional encoding in mip-NeRF uses a Gaussian approximation to a spatial
footprint \citep{mipnerf}; related approaches apply integration to super-resolution
\citep{ipesr} and evaluate volumetric encodings more precisely \citep{exactnerf}.
Kernel mean embeddings provide a general framework for representing distributions through
feature expectations \citep{kernelmean,hilbertembed}.
The support code used here is a finite Fourier feature expectation. Unlike a full
characteristic function or a characteristic-kernel embedding, a finite rotary bank is
not generally injective over all support measures.
Our analysis concerns its compatibility with token aggregation and the observations
available to an attention reader.

\paragraph{Interval-conditioned attention.}
TIE derives interval-integrated rotary features for video-generation conditioning
\citep{tie}. With temporal scale $\gamma$ (default 4), the encoded interval radius is
$r=\gamma(b-a)/2$ and centre $\gamma(a+b)/2$. Its uniform-kernel RoTE gain is
$\operatorname{sinc}(\theta_mr)/C_r$, where $C_r$ is the mean sinc gain over the
encoded pairs. Thus for physical width $w$, its argument is $\theta_m\gamma w/2$.
The native method rotates point queries and interval keys in video-generation cross-attention.
A symmetric interval-query application is a separately defined adaptation.
Our equal-merge characterisation instead asks which calibrated code remains affine under
aggregation. A nonconstant duration normaliser generally breaks that property.
Normalised readout remains compatible with exact merging when the unnormalised moments,
or enough information to recover them, are retained.
These objectives address different operations: normalising interval-conditioned attention
and preserving a representation through merges.

\paragraph{The mean-normalised control evaluated here.}
The saved arm labelled \texttt{rote} uses
\begin{equation}
 G_m^{\mathrm{MN}}(w)=
 \frac{\operatorname{sinc}(\theta_mw)}
 {\operatorname{sgn}_{+}(C_w)\max(|C_w|,10^{-3})},\qquad
 C_w=\frac{1}{n_t}\sum_{\ell=1}^{n_t}\operatorname{sinc}(\theta_\ell w),
 \label{eq:mncontrol}
\end{equation}
where $\operatorname{sgn}_{+}(0)=1$ and $n_t$ counts the encoded temporal pairs.
We call this arm \emph{MN-sinc}. After matching the effective physical-frequency bank
$\gamma\theta$, its argument uses full width in place of RoTE's half width. It is a
mean-normalised sinc control, not a
reproduction of RoTE on matched intervals. The magnitude-matched Aperture arm uses
MN-sinc's gain RMS as its target. The empirical comparisons apply to these explicitly
defined controls and do not establish superiority over published RoTE.

\paragraph{Frequency adaptation.}
EchoingPixels introduces Sync-RoPE, reallocating low-frequency channels to the temporal
axis and retaining a high-frequency temporal section \citep{echoingpixels}.
The control evaluated in this repository instead sets
$\theta_{m,i}^{\mathrm{cap}}=\min(\theta_m,\pi/w_i)$ independently for each token.
We call it \emph{width-capped phase}; its saved arm identifier is \texttt{syncrope}.
It tests a related frequency-control idea, not the published channel-allocation mechanism.
When two tokens receive different capped frequencies, a common shift adds different
phases, so the usual common-frequency translation cancellation need not hold.
This property of the control should not be attributed to published Sync-RoPE.

\paragraph{Feature placement and compression.}
PPE partitions each coordinate axis's rotary section into groups and assigns different
original member position IDs to those groups \citep{ppe}. It requires source membership
and retained coordinates. Scalar or sinusoidal width features do not reproduce this operator.
Our width-feature arm is a generic content-pathway baseline, independent of PPE.
Merging and retention methods such as ToMe, PruMerge, H$_2$O and SnapKV change the set
and supports of tokens available to attention \citep{tome,prumerge,h2o,snapkv}.
We use these operations to study positional state; no new compression policy is proposed.
Physical-time representations \citep{funtime,time2vec}, context rescaling
\citep{pi,yarn}, and multimodal axis allocation \citep{qwen2vl,videorope} address
complementary choices of clock, frequency and coordinate system.

\paragraph{Reference-operator validation.}
We implement the published RoTE interval multiplier, PPE's assignment for supplied selected IDs,
Sync-RoPE's $[T,H,W,T]$ allocation, and the Gaussian Fourier expectation underlying integrated
positional encoding. CPU comparisons with fixed official source revisions give maximum RoTE
state discrepancy below $5\times10^{-16}$; PPE and Sync-RoPE query/key rotations agree
coordinatewise in their tested configurations. These checks validate positional primitives,
not full downstream reproductions. Native token selection, input construction, adaptation and
evaluation follow the executed GPU protocol (\secref{sec:app_gpu}). The released PPE selector and its score-ranked
interpretation are retained as separately named protocols where they differ.

\end{document}